\documentclass[10pt,dvipsnames]{article} % For LaTeX2e

\usepackage[left=1in,right=1in,top=1in,bottom=1in]{geometry}
\usepackage{xcolor}
\usepackage[round, sort&compress]{natbib}
\usepackage{amssymb,amsthm, amsfonts}            % Defines common symbols like \mathbb R
\usepackage{mathtools, thmtools}          % Extends amsmath, providing common math tools
\usepackage{mathrsfs}           % Enables \mathscr, which can work in cases that \mathcal does not
\mathtoolsset{showonlyrefs}     % Only number equations that are referenced (optional)
\usepackage{graphicx}           % For including images
\usepackage{subcaption}         % Allows for the use of subfigures and subcaptions
\usepackage[space]{grffile}     % For spaces in image names
\usepackage{url}                % For displaying URLs
\usepackage{lipsum}             % For placeholder text
\usepackage{algorithm}
\usepackage[noend]{algpseudocode}
\usepackage{booktabs}
\usepackage{enumitem}
\usepackage{hyperref}

\hypersetup{
    colorlinks=true,
    linkcolor=Cerulean,
    citecolor=Cerulean,
    urlcolor=Magenta,
    pdftitle={RQRE},
    pdfauthor={Jake Gonzales, Max Horwitz, Eric Mazumdar, Lillian J. Ratliff}
}
\theoremstyle{definition}
\newtheorem{theorem}{Theorem}
\newtheorem{lemma}{Lemma}
\newtheorem{proposition}{Proposition}
\newtheorem{assumption}{Assumption}
\newtheorem{definition}{Definition}
\newtheorem{corollary}{Corollary}
\newtheorem{remark}{Remark}
\newtheorem{example}{Example}

\newcommand{\mc}{\mathcal}
\newcommand{\mb}{\mathbb}

\DeclareMathOperator*{\E}{\mathbb{E}}
\DeclareMathOperator*{\Reg}{\mathrm{regret}}
\DeclareMathOperator*{\argmax}{arg\,max}
\DeclareMathOperator*{\KL}{\mathrm{KL}}

\newcommand{\br}{\mathrm{br}}
\newcommand{\reg}{\mathrm{reg}}

\let\cite\citep
\newcommand{\X}{\mc{X}}
\newcommand{\A}{\mc{A}}
\newcommand{\Ps}{\mc{P}}
\newcommand{\la}{\langle}
\newcommand{\ra}{\rangle}
\newcommand{\entropy}{\Phi}

\newcommand{\bddrat}{\epsilon}
\newcommand{\risk}{\tau}
\newcommand{\bigmid}{\,\Big|\,}

\DeclareMathOperator*{\argmin}{arg\,min}
\newcommand{\penaltyfn}{\varphi}
\newcommand{\penaltyfntype}[1]{\varphi_{#1}}

\newcommand{\riskloss}{l}
\newcommand{\baseV}{\widetilde{V}}
\newcommand{\psf}{{\sf p}}
\newcommand{\esf}{{\sf e}}
\newcommand{\mP}{\mc{P}}

\newcommand{\sBR}{\mathrm{sBR}}

\newcommand{\R}{\mb{R}}

\newcommand{\dist}{\mathrm{dist}}

\newcommand{\norm}[1]{\left\lVert #1 \right\rVert}
\newcommand{\ip}[2]{\left\langle #1,#2 \right\rangle}
\newcommand{\Vhat}{\widehat V}
\newcommand{\rhohat}{\widehat\rho}

\newcommand{\pibr}{\pi^{\rm br}}
\setlist[itemize]{itemsep=-2pt, topsep=-2pt}

\title{Strategically Robust Multi-Agent Reinforcement Learning with Linear Function Approximation}

\author{Jake Gonzales\textsuperscript{1}\and Max Horwitz\textsuperscript{1}\and Eric Mazumdar\textsuperscript{2}\and Lillian J.~Ratliff\textsuperscript{1} }

\begin{document}

\maketitle

\begin{center}
\small
$^{1}$Department of Electrical \& Computer Engineering, University of Washington \\
\small $^{2}$Computing \& Mathematical Sciences, California Institute of Technology \\
\small $^{*}$Corresponding author: \texttt{jakegonz@uw.edu}
\end{center}

\begin{abstract}
Provably efficient and robust equilibrium computation in general-sum Markov games remains a core challenge in multi-agent reinforcement learning. Nash equilibrium is computationally intractable in general and brittle due to equilibrium multiplicity and sensitivity to approximation error. We study Risk-Sensitive Quantal Response Equilibrium (RQRE), which yields a unique, smooth solution under bounded rationality and risk sensitivity. We propose \texttt{RQRE-OVI}, an optimistic value iteration algorithm for computing RQRE with linear function approximation in large or continuous state spaces. Through finite-sample regret analysis, we establish convergence and explicitly characterize how sample complexity scales with rationality and risk-sensitivity parameters. The regret bounds reveal a quantitative tradeoff: increasing rationality tightens regret, while risk sensitivity induces regularization that enhances stability and robustness. This exposes a Pareto frontier between expected performance and robustness, with Nash recovered in the limit of perfect rationality and risk neutrality. We further show that the RQRE policy map is Lipschitz continuous in estimated payoffs, unlike Nash, and RQRE admits a distributionally robust optimization interpretation. Empirically, we demonstrate that \texttt{RQRE-OVI} achieves competitive performance under self-play while producing substantially more robust behavior under cross-play compared to Nash-based approaches. These results suggest \texttt{RQRE-OVI} offers a principled, scalable, and tunable path for equilibrium learning with improved robustness and generalization.
\end{abstract}

%%%%%%%%%%%%%%%%%%%%%%%%%%%%%%%%%%%%%%%%%%%%%%%%%%%%%%%%%%%%%%%%
%% Introduction
%%%%%%%%%%%%%%%%%%%%%%%%%%%%%%%%%%%%%%%%%%%%%%%%%%%%%%%%%%%%%%%%
\section{Introduction}
\label{sec:introduction}

Modeling multiple strategic decision-makers in dynamic, real-world environments remains a central challenge in artificial intelligence. Multi-agent reinforcement learning  provides a principled framework for such settings, where agents learn policies by interacting with the environment and adapting to the behavior of others \cite{Zhang2021MARL}. This has enabled advances in domains ranging from autonomous driving \cite{cusumano-towner2025robust}, high-frequency trading \cite{mohl2025jaxmarlfinance}, multi-robot control \cite{Gu2023MultiRobotControl}, and alignment of large language models \cite{pmlr-v235-munos24a}. Despite these empirical successes, developing principled algorithms with rigorous theoretical guarantees for multi-agent reinforcement learning remains an open challenge, particularly in settings with continuous or large state spaces, where equilibrium selection and brittleness lead to poor generalization and robustness.

Towards addressing scalability,  \citet{cisneros2023finite} proposed Nash Q-learning with Optimistic Value Iteration (NQOVI), extending the frameworks of \citet{hu2003nash} and \citet{liu2021sharp} to the linear function approximation regime and providing finite-sample regret bounds that scale with feature dimension rather than the size of the state space. While this represents significant progress, a fundamental bottleneck persists: the algorithm requires solving for a Nash equilibrium at every stage game and every episode. This not only inherits the computational intractability of Nash equilibria in general-sum games but also exposes the learning process to the inherent instability of the Nash correspondence (as we demonstrate in Example~\ref{ex:nash_unstable}), where arbitrarily small perturbations in estimated payoffs can produce discontinuous jumps in the selected equilibrium strategy. In the function approximation setting, where Q-values are necessarily estimated with error, this brittleness is especially problematic: it means that the equilibrium computation at the core of the algorithm is ill-conditioned with respect to the approximation errors that the algorithm itself introduces.

These limitations motivate equilibrium concepts that are simultaneously computationally tractable, stable under payoff perturbations, and amenable to scalable reinforcement learning with formal guarantees. This leads to our central research question:
%%%%%% Research Question %%%%%%
\begin{quote}
\emph{Can we achieve provably sample-efficient learning of computationally tractable and robust equilibria in general-sum Markov games with linear function approximation?}
\end{quote}
To address this question, we study the Risk-Sensitive Quantal Response Equilibrium (RQRE), a solution concept introduced in \citet{mazumdar2025tractable} which models agents as both boundedly rational and risk-averse. The RQRE combines two behavioral modeling choices, each of which independently contributes desirable properties for multi-agent learning while also capturing behaviors inherent to natural learning agents (as opposed to algorithmic ones). Indeed, \emph{bounded rationality} replaces exact best responses with smooth, regularized response mappings, yielding well-posed and often unique equilibria while mitigating discontinuities and equilibrium multiplicity \cite{MCKELVEY19956, Erev1998PredictingHowPeoplePLayGames, Goeree1999PlayingGames}. In multi-agent reinforcement learning, this framework improves stability under approximation error and non-stationarity, providing a principled tradeoff between optimality and robustness that aligns naturally with regularized learning objectives. At the same time, incorporating \emph{risk sensitivity} addresses key limitations of purely expected-value objectives by discouraging policies that achieve high average performance at the expense of rare but catastrophic outcomes. In multi-agent settings,  risk-sensitive formulations act as a principled form of robustification---reducing sensitivity to modeling error, noise, and opponent misspecification \cite{mazumdar2025tractable, lanzetti2025strategicallyrobustgametheory}---thereby promoting more stable, predictable, and generalizable equilibria.

\textbf{Contributions.} The main contributions of this work are summarized as follows:
\begin{itemize}[itemsep=-2pt, topsep=0pt]
\item \textbf{Finite-sample guarantees.} We perform regret analysis for a novel algorithm, \texttt{RQRE-OVI}, and establish that regret satisfies $\reg(K) \leq \tilde{\mc O}\bigl(L_{\sf env}\,B\,\sqrt{K\,d^3\,H^3}\bigr) + KH\bigl(\varepsilon_{\sf env} + L_{\sf env}(\varepsilon_{\sf pol} + \varepsilon_{\sf eq})\bigr)$, where $\varepsilon_{\sf env}, \varepsilon_{\sf pol}, \varepsilon_{\sf eq}$ capture environment-risk, policy-risk, and stage-equilibrium approximation errors. The bound explicitly characterizes how sample complexity scales with the rationality and risk-sensitivity parameters: the value range $B = H(1 + \log|A|/\epsilon)$ depends on the rationality parameter $\epsilon$, while the solver error scales as $\varepsilon_{\sf eq} = O(B\sqrt{\Delta_{\sf eq}/\tau_{\min}})$, so that greater risk aversion relaxes solver accuracy requirements. Our analysis provides the first regret guarantees for optimistic MARL with linear function approximation with risk-sensitivity and approximate equilibrium computation, rather than relying on a Nash equilibrium oracle.
\item \textbf{Distributional robustness.} We show that the RQRE are distributionally robust (Proposition~\ref{prop:rqe_strict_generalization_with_existence_comment}), and generalize other robust notions of equilibrium.   
This connects the bounded rationality parameter to a formal notion of robustness against payoff misspecification.
\item \textbf{Stability.} We prove that the RQRE policy map is Lipschitz continuous in estimated payoffs (Corollary~\ref{cor:rqe_lipschitz}), a stability property that Nash equilibria provably lack (Example~\ref{ex:nash_unstable}). This theoretically justifies the use of RQRE over Nash. Moreover, it allows us to conclude policy convergence under additional regularity (Proposition~\ref{prop:rqe_ovi_robust}). 
\item \textbf{Empirical evaluation.} We evaluate \texttt{RQRE-OVI} on popular multi-agent coordination benchmarks and demonstrate that it achieves competitive self-play performance while producing substantially more robust behavior under cross-play compared to Nash-based methods.
\end{itemize}

An extended related work discussion is given in Appendix~\ref{sec:related-work-extended}.

%%%%%%%%%%%%%%%%%%%%%%%%%%%%%%%%%%%%%%%%%%%%%%%%%%%%%%%%%%%%%%%%
%% Preliminaries
%%%%%%%%%%%%%%%%%%%%%%%%%%%%%%%%%%%%%%%%%%%%%%%%%%%%%%%%%%%%%%%%
\section{Preliminaries}
\label{sec:preliminaries}
In this section, we present the game-theoretic and multi-agent reinforcement learning preliminaries. We consider the framework of Markov games (also known as stochastic games), which generalizes the standard Markov decision process (MDP) to the multi-player setting \cite{Shapley1953, Littman1994, StochasticGamesNeyman2003}.

We define an episodic Markov game of the form $\mc{MG}=(\X, \A, H, \Ps, r)$ with state space $\X$, joint action space $\mc{A}=\mc{A}_1\times \cdots \times \mc{A}_n$, episode length $H$, and transition kernels $\mc{P}=\{\mc{P}_h\}_{h\in [H]}$, where $\mc{P}_h(\cdot \mid x,a)$ denotes the transition kernel at step $h$. Agent $i$'s reward is $r_i=\{r_{i,h}\}_{h=1}^H$ with $r_{i,h}:\X\times \A\to [0,1]$. At step $h$, agent $i$ selects action $a_{i,h}\in \A_i$, and we write $a_h = (a_{1,h}, \ldots, a_{n,h}) \in \A$ for the joint action. The agents' action spaces are assumed finite; however, the state space can be arbitrarily large or continuous.

We denote agent $i$'s policy by $\pi_i=\{\pi_{i,h}\}_{h\in [H]}$ with $\pi_{i,h}:\X\to \Delta(\A_i)$, and by an abuse of notation write $\pi_h:\X\to \Delta(\A)$ for the joint policy. For agent $i$, define the value function $\baseV_h^{i}(\cdot; \pi):\X\to \mb{R}$ at step $h$ as 
$
\baseV_h^{i}(x;\pi):=\mb{E}_\pi[\sum_{h'=h}^H r_{i,h'}(x_{h'},a_{h'}) \bigmid x_h=x]$.
The most common solution concept for games is the Nash equilibrium \cite{Nash1950}. 
\begin{definition}[Markov Nash Equilibrium]
\label{def:nash-equilibrium}
A joint policy $\pi^* = (\pi_1^*, \ldots, \pi_n^*)$ is a \emph{Nash equilibrium} if no player can improve their value by unilaterally deviating. That is, for all $i \in [n]$, all alternative policies $\pi_i'$, and all initial states $s_0 \in \mathcal{X}$, the inequality 
$
\baseV_1^{i}(s_0;\pi^*) \geq \baseV_1^{i}(s_0;\pi')$ holds,
where $\pi' = (\pi_i', \pi_{-i}^*)$ denotes the joint policy in which player $i$ deviates to $\pi_i'$ while all others play $\pi_{-i}^*$. If $\baseV_1^{i}(s_0;\pi^\ast)\leq \varepsilon+\baseV_1^{i}(s_0;(\pi_i,\pi_{-i}^\ast))$ for all $i$ and $s_0$, then $\pi^\ast$ is an $\varepsilon$-approximate Nash equilibrium. 
\end{definition}

\subsection{Bounded Rationality and Quantal Response Equilibrium}
\label{sec:bdd_rationality}
We now formalize the notion of bounded rationality central to risk quantal response equilibrium. 
Bounded rationality provides a principled mechanism for smoothing the equilibrium correspondence, ensuring uniqueness, and regularizing the learning problem, properties that are essential for stable multi-agent learning under function approximation. 
The quantal response framework of \citet{MCKELVEY19956} captures this by modeling agents as stochastic optimizers who assign higher probability to higher-payoff actions while maintaining exploration over the full action space.

With this in mind, we introduce the formalism for incorporating bounded rationality into the admissible policy class. Consider a \emph{finite normal-form game} defined by the tuple $\mathcal{G} = ([n],\{\A_i\}_{i\in [n]},\{u_i\}_{i\in [n]})$, where for each player $i\in [n]$, $u_i : \A \to \mathbb{R}$ is the utility function. We write a joint action as $a=(a_i,a_{-i})$, where $a_{-i}$ represents the actions of all players except $i$.  For a fixed profile $\pi_{-i}$, let $U_i(\pi_{-i}) \in \mathbb{R}^{|\mathcal{A}_i|}$ be the vector whose entry corresponding to action $a_i$ is $u_i(a_i, \pi_{-i}) := \mathbb{E}_{a_{-i} \sim \pi_{-i}}[u_i(a_i, a_{-i})]$.

\begin{definition}[Quantal Response Equilibrium \cite{MCKELVEY19956}]
A mixed-strategy profile $\pi^* = (\pi_1^*, \dots, \pi_n^*)$ is a \emph{Quantal Response Equilibrium} (QRE) if, for every player $i \in [n]$, the response is given by
$\pi_i^* = \sigma_i\left( U_i(\pi_{-i}^*) \right)$, where $\sigma_i: \mathbb{R}^{|\mathcal{A}_i|} \to \Delta(\mathcal{A}_i)$ is a continuous quantal response function  that is  \emph{monotonic} (i.e., for any  $x \in \mathbb{R}^{|\mathcal{A}_i|}$, $x_{a} > x_{a'}$ implies $\sigma_i(a \mid x) > \sigma_i(a' \mid x)$) and \emph{full support} (i.e., $\sigma_i(a \mid x) > 0$ for all $a \in \mathcal{A}_i$).
\end{definition}
It is well known that quantal responses  may be introduced into games via regularization; indeed this is shown via constraining the player's responses to quantal responses (see, e.g.,~\citet{follmer2002convex,hofbauer2002global}). This is straightforward via Fenchel duality.

As an example consider the quantal response equilibrium induced by the logit map.
\begin{example}[Logit Quantal Response Equilibrium]\label{ex:logit-qre}
 For game $\mc{G}$, a mixed-strategy profile $\pi=(\pi_1,\dots,\pi_n)$ is a \emph{logit-QRE} with precision parameter $\bddrat>0$ if for every player $i$ and action $a_i\in\mathcal A_i$,
 \[
 \pi_i(a_i)
 =
e\exp\left(
   \bddrat\, \E_{a_{-i}\sim\pi_{-i}}[u_i(a_i,a_{-i})]
 \right)
 /\left(
 \sum_{a'_i\in\mathcal A_i}
 \exp\left(
   \bddrat\, \E_{a_{-i}\sim\pi_{-i}}[u_i(a'_i,a_{-i})]
 \right)
 \right).
 \]
 Equivalently, with the Shannon entropy $\entropy(\mu)=-\frac{1}{\bddrat}\sum_{a_i}\mu(a_i)\log\mu(a_i)$, the policy $\pi_i$ solves the entropy-regularized problem
 $\pi_i
 \in
 \argmax_{\mu\in\Delta(\mathcal A_i)}
 \left\{
 \E_{ a_i\sim\mu, a_{-i}\sim\pi_{-i}}[u_i(a_i,a_{-i})]
 + \frac{1}{\bddrat} \entropy(\mu)
 \right\}.$ 
Observe that as $\bddrat\to\infty$, the regularization term $\frac{1}{\bddrat}\entropy(\mu)\to 0$, thereby recovering the standard Nash equilibrium. Conversely, as $\bddrat\to 0$, the entropy term dominates and the policy approaches the uniform distribution.
\[
\pi_i
\in
\argmax_{\mu\in\Delta(\mathcal A_i)}
\left\{
\E_{\substack{ a_i\sim\mu\\ a_{-i}\sim\pi_{-i}}}[u_i(a_i,a_{-i})]
+ \frac{1}{\bddrat} \entropy(\mu)
\right\}.
\]
\end{example}

\subsection{Risk-Sensitive Markov Games}
\label{subsec:risk-sensitive-markov-game}
Having established the role of bounded rationality, we now incorporate risk aversion into the multi-agent framework. As discussed in Section~\ref{sec:introduction}, risk-sensitive objectives provide a principled mechanism for promoting safety, robustness to modeling error and strategic misspecification, and stability in settings with multiple equilibria or volatile dynamics. To formalize this, we equip each agent with a convex risk measure drawn from the class studied extensively in mathematical finance and operations research \cite{follmer2002convex}.

\begin{definition}[Convex Risk Measure]\label{def:convex-risk-measure}
A risk functional $\rho: L^\infty(\Omega,\mathcal{F},\mathbb{P}) \to \mb{R}$ is a convex risk measure if it satisfies the following properties:
\begin{enumerate}[ label=\alph*.]
  \item \textbf{Convexity:} For all $x,y\in L^\infty(\Omega,\mathcal{F},\mathbb{P})$ and $\lambda\in[0,1]$,
      $\rho(\lambda x + (1-\lambda)y) \leq \lambda \rho(x) + (1-\lambda) \rho(y)$.
  \item \textbf{Monotonicity:} For all $x,y\in L^\infty(\Omega,\mathcal{F},\mathbb{P})$ such that $x\leq y$,
     $\rho(x) \geq \rho(y)$.
  \item \textbf{Translation Invariance:} For all $x\in L^\infty(\Omega,\mathcal{F},\mathbb{P})$ and $c\in\mb{R}$,
  $\rho(x+c) = \rho(x) + c$.
\end{enumerate}
\end{definition}

A fundamental property of convex risk measures is that they admit a dual representation as worst-case expectations over adversarial perturbations \cite{follmer2002convex}, providing a principled framework for modeling uncertainty in both environment dynamics and opponent strategies.

\begin{theorem}[Dual Representation of Convex Risk Measures]
\label{thm:dual-rep}
A mapping $\rho : \mc{Z} \to \mb{R}$ over a finite outcome space $\Omega$ is a convex risk measure if and only if there exists a convex, lower-semicontinuous penalty function $\penaltyfn: \Delta(\Omega) \to (-\infty, \infty]$ such that $\rho(Z) = \sup_{p \in \Delta(\Omega)} \left\{ \E_{p}[-Z] - \penaltyfn(p) \right\}$.
\end{theorem}

\begin{assumption}[Lipschitz Penalty Functions]
\label{ass:lipschitz-penalty}
For each player $i\in[n]$, the penalty function $\penaltyfn_i : \Delta(\Omega) \to (-\infty, \infty]$ in the dual representation (Theorem~\ref{thm:dual-rep}) is $L_{\penaltyfn}$-Lipschitz continuous with respect to the $\ell_1$ norm. That is, for all $p, q \in \Delta(\Omega)$, the bound holds: $|\penaltyfn_i(p) - \penaltyfn_i(q)| \leq L_{\penaltyfn} \|p - q\|_1$.
\end{assumption}
This assumption ensures that the risk measure $\rho$ is Lipschitz continuous with respect to the underlying distribution of returns and hence with respect to the $Q$-functions we estimate.

\paragraph{Risk-adjusted loss in matrix games.} Consider an $n$-player matrix game with action sets $\{\A_i\}_{i=1}^n$ and payoff functions $u_i:\A \to \mb{R}$. Each player $i$ minimizes a risk-adjusted loss based on the negation of their utility. For a joint mixed-strategy profile $\pi$, the risk-adjusted loss is $\ell_i(\pi_i,\pi_{-i}) := \rho_i(-u_i(a))$, where the randomness is over the joint action $a \sim (\pi_i,\pi_{-i})$. By  Theorem~\ref{thm:dual-rep} (dual representation), this loss is equivalently expressed as
\begin{equation}\label{eq:risk_induced_loss}
\ell_i(\pi_i, \pi_{-i}) = \sup_{p_i \in \Delta(\mathcal{A}_{-i})} \left\{ \E_{a_{-i} \sim p_i} [ -u_i(\pi_i, a_{-i}) ] - \penaltyfn_i(p_i, \pi_{-i}) \right\}.
\end{equation}
Here $p_i$ represents an adversarial distribution over the opponents' joint actions, while $\penaltyfn_i(\cdot, \pi_{-i})$ is a convex penalty that constrains the adversary by penalizing deviations from the reference policy $\pi_{-i}$. A canonical choice is the entropic risk measure.

\begin{example}[Entropic Risk]
\label{ex:entropic_risk}
The entropic risk measure is obtained by setting the penalty function $\penaltyfn_i$ to be the Kullback--Leibler divergence scaled by a temperature parameter. The risk-adjusted loss then admits the closed form
\[\ell_i(\pi_i, \pi_{-i}) = \frac{1}{\tau_i} \log \left( \E_{a_{-i} \sim \pi_{-i}} \left[ \exp( -\tau_i \, u_i(\pi_i, a_{-i}) ) \right] \right).\]
The parameter $\tau_i > 0$ controls the degree of risk aversion: as $\tau_i \to 0$, the objective recovers the standard risk-neutral expected loss, whereas $\tau_i \to \infty$ yields the worst-case (minimax) objective.
\end{example}

\paragraph{Risk-Averse Quantal Response Equilibrium (RQE).}
To model agents that are both risk-averse and boundedly rational, we  augment the risk-adjusted loss with a regularization term. For each player $i$, let $\nu_i: \Delta(\mathcal{A}_i) \to \mathbb{R}$ be a strictly convex regularizer and $\epsilon_i > 0$ a precision parameter controlling the degree of bounded rationality. The \emph{regularized risk-adjusted loss} is
\[
  \riskloss_i(\pi_i, \pi_{-i})
  :=
  \ell_i(\pi_i, \pi_{-i}) + \frac{1}{\epsilon_i} \nu_i(\pi_i),
\]
where $\ell_i$ is the risk-adjusted loss in \eqref{eq:risk_induced_loss}. A common example of a convex regularizer is negative entropy $\nu_i(p)=\sum_kp_k\log(p_k)$, for instance, if players are are constrained to logit response functions. Each player seeks to minimize this combined objective, balancing risk minimization against the entropy constraint imposed by bounded rationality.

\begin{definition}[\citet{mazumdar2025tractable}]
\label{def:rqe}
Consider an $n$-player normal-form game. A joint strategy profile $\pi^\star = (\pi_1^\star, \dots, \pi_n^\star)$ is a \emph{Risk-Averse Quantal Response Equilibrium (RQRE)} if, for every player $i \in [n]$, the policy in equilibrium is given by
\[
  \pi_i^* = \argmin_{\pi_i \in \Delta(\mathcal{A}_i)}
 \{
    \ell_i(\pi_i, \pi_{-i}^*) + \frac{1}{\epsilon_i} \nu_i(\pi_i)
\}.\]
\end{definition}
Under sufficient risk aversion and bounded rationality, \citet{mazumdar2025tractable} show that RQRE is computationally tractable via no-regret learning by establishing an equivalence to the coarse correlated equilibrium of a \emph{lifted $2n$-player game}.

\paragraph{Extension to Risk-Sensitive Markov games.} 
Now we show the extension of the static RQRE concept to the finite-horizon Markov game setting,  considering
both the uncertainty of agents’ strategies and the environments (rewards and transitions). The key is  applying the equilibrium condition stage-wise to the action-value functions.

Let us formalize the connection between the policy-side and 
environment-side risk operators used in the risk-sensitive Markov game and a 
single abstract convex risk functional $\rho$.  
As demonstrated via Theorem~\ref{thm:dual-rep}, a large class of convex risk functionals admit the dual representation
\[
\rho(Z)
=
\sup_{p \in \Delta(\Omega)}
\left\{
\mathbb E_p[-Z]
-
\varphi(p)
\right\}.\]
In risk-sensitive Markov games, risk operators arise from this dual form, but applied to different sources of randomness:
\begin{itemize}[label={(\roman*)}, topsep=5pt]
    \item \textbf{Environment risk:} randomness in next state $x_{h+1}$; 
    \item \textbf{Policy risk:} randomness in opponents' actions $a_{-i}$.
\end{itemize}
The penalty functions $\penaltyfn_{\esf,i}$ and $\penaltyfn_{\psf,i}$ determine the 
particular choice of convex risk functional $\rho$ in each case. We remark that these two sources of risk are fundamentally different and thus must be separated. Policy risk represents a form of strategic risk-aversion and changes as the opponents' strategies change over time. Environment risks are taken with respect to the randomness stemming from the random---but stationary---environment the agent is acting in. 

Now we express the environment risk as a conditional risk measure. 
Fixing $(i,h,x,a)\in [n]\times [H]\times \mc{X}\times \mc{A}$,  the randomness in the environment arises from the next state
$
x_{h+1} \sim \mP_h(\cdot \mid x,a)$.
Given a continuation value function $V_{i,h+1}(\cdot;\pi)$, define the random variable
$
Z_{i,h}^{\mathsf{e}}(x')
:=
V_{i,h+1}(x';\pi)$.
We define the environment risk operator
\begin{equation}
\label{eq:R-env-def}
\mc R^{\esf}_{i,h}
\big(r_{i,h}, \mP_h, V_{i,h+1}; x,a\big)
:=
r_{i,h}(x,a)
+
\rho^{\mathsf{e}}_{i,h}
\big(
Z_{i,h}^{\mathsf{e}}
\,\big|\,
\mP_h(\cdot \mid x,a)
\big),
\end{equation}
where 
the conditional risk functional is given in the dual form by
\[\rho^{\mathsf{e}}_{i,h}(Z| \mP)
=
\inf_{\widetilde{\mP} \in \mc P(\mc X)}
\left\{
\int Z(x')\,\widetilde{\mP}(dx')
+
\penaltyfn_{\esf,i}(\widetilde{\mP}\,\|\,\mP)
\right\}.\] 
Thus, in this notation, we have action-value function 
\[
Q_h^i(x,a;\pi)
=
\mc R^{\mathsf{e}}_{i,h}
\big(r_{i,h},\mP_h,V_{i,h+1};x,a\big).\]
The penalty function $\penaltyfn_{\esf,i}$ determines the specific form of 
$\rho^{\mathsf{e}}_{i,h}$ (e.g.\ entropic risk, distributionally robust risk).

Now, let $\pi_h(\cdot\mid x)$ denote the joint mixed action.
For each opponent action profile $a_{-i}$, define the induced cost
\[
Z_{i,h}^{\mathsf{p}}(a_{-i})
:=
\sum_{a_i\in\mc A_i}
\pi_{i,h}(a_i\mid x)\,
Q_h^i(x,a_i,a_{-i};\pi).\]
The policy-side risk operator is
\begin{equation}
\label{eq:R-pol-def}
\mc R^{\sf p}_{i,h}
\big(Q_h^i,\pi_h;x\big)
:=
\rho^{\mathsf{p}}_{i,h}
\big(
Z_{i,h}^{\mathsf{p}}
\,\big|\,
\pi_{-i,h}(\cdot\mid x)
\big),
\end{equation}
where the dual representation is
    \[\rho^{\mathsf{p}}_{i,h}(Z\mid \pi_{-i})
    =
    \sup_{p_i\in \mc P_i^{\psf}}
    \{
    -\langle Z,p_i\rangle
    -
    \penaltyfn_{\psf,i}(p_i,\pi_{-i})
    \}.\]
Hence, in this notation, the value function is
\[
V_h^i(x;\pi)
=
\mc R^{\mathsf{p}}_{i,h}
\big(Q_h^i,\pi_h;x\big).\]

With an additional strictly convex regularizer 
$\nu_i:\Delta(\mc A_i)\to\mathbb R$, the regularized value is
\begin{equation}
\label{eq:R-pol-regularized}
V_{i,h}^{\epsilon_i}(\pi;x)
=
\mc R^{\mathsf{p}}_{i,h}
\big(Q_h^i,\pi_h;x\big)
+
\frac{1}{\epsilon_i}\nu_i(\pi_{i,h}(\cdot\mid x)).
\end{equation}
Thus the risk-regularized Bellman recursion is defined by 
\[
Q_h^i
=
\mc R^{\mathsf{e}}_{i,h}
\big(r_{i,h},\mP_h,V_{i,h+1}\big),\quad\text{and}\quad
V_h^i
=
\mc R^{\mathsf{p}}_{i,h}(Q_h^i,\pi_h).\]

\begin{definition}[Risk Quantal Response Equilibrium for Markov Games]
\label{def:rqre-general}
Consider a finite-horizon risk-sensitive Markov game $\mc{MG}$ equipped with
environment risk operators $\mc R^{\mathsf e}_{i,h}$ and
policy risk operators $\mc R^{\mathsf p}_{i,h}$, as defined in
\eqref{eq:R-env-def}–\eqref{eq:R-pol-def}, for each player $i\in[n]$.
A Markov policy profile
$\pi^\star=\{\pi_h^\star\}_{h=1}^H$
is called a \emph{Risk Quantal Response Equilibrium (RQRE)}
if for every stage $h\in[H]$, every state $x\in\mc X$, and every player
$i\in[n]$, the mixed action
$\pi_{i,h}^\star(\cdot\mid x)$ satisfies the fixed-point condition
\begin{equation}
\label{eq:rqre-general}
\pi_{i,h}^\star(\cdot\mid x)
\in
\arg\min_{\mu_i\in\Delta(\mc A_i)}
\Big\{
\mc R^{\mathsf p}_{i,h}
\big(
Q_h^i(\cdot;\pi^\star),
(\mu_i,\pi_{-i,h}^\star);x
\big)
+
\tfrac{1}{\epsilon_i}\nu_i(\mu_i)
\Big\},
\end{equation}
where the continuation action-value functions satisfy the
risk-sensitive Bellman recursion
$
Q_h^i(x,a;\pi^\star)
=
\mc R^{\mathsf e}_{i,h}
\big(
r_{i,h},\mP_h,V_{i,h+1}(\cdot;\pi^\star);x,a
\big)$, and
$V_{i,h}^{\epsilon_i}(\pi^\star;x)
=
\mc R^{\mathsf p}_{i,h}
\big(
Q_h^i,\pi_h^\star;x
\big)
+
\tfrac{1}{\epsilon_i}\nu_i(\pi_{i,h}^\star(\cdot\mid x))$,
with terminal condition
$V_{i,H+1}^{\epsilon_i}(\cdot;\pi^\star)\equiv 0$.
\end{definition}

\section{Risk-Sensitive QRE Optimistic Value Iteration}

\begin{algorithm}[t!]
  \caption{Risk QRE-Optimistic Value Iteration (RQRE-OVI)}
  \label{alg:pr-er-rqre-ovi}
  \begin{algorithmic}[1]
  \Require Risk-sensitivity $\tau_i>0$, policy regularization (bounded rationality) $\epsilon_i>0$.
  \Require finite sets $\widehat{\mathcal P}_i^{\psf}$ (policy risk) and $\widehat{\mathcal P}_{i}^{\esf}$ (environment risk).
  Set $Q_{h}^{i,1}\equiv H$ for all $i,h$.

  \For{$k=1,2,\dots,K$} \Comment{Episode $k$}
    \For{$h=H,H-1,\dots,1$} \Comment{Backward pass}
      \State $\Lambda_h^k\gets \lambda I_d+\sum_{t=1}^{k-1}\phi(x_h^t,a_h^t,h)\phi(x_h^t,a_h^t,h)^\top$
      \State $\pi_{h+1}^{k}(x)\gets \texttt{RQRE}_{\varepsilon}(Q_{h+1}^{1,k}(x,\cdot),\dots,Q_{h+1}^{n,k}(x,\cdot))$

      \For{$i=1,\dots,n$}

        \State
          $\widehat V_{i,h+1}^{\epsilon_i}(\pi;x)
          :=
          \max_{p_i\in \widehat{\mathcal P}_i^{\psf}}
          \{
          - \pi_i(\cdot|x)^\top Q_{h+1}^{i,k}(x,:)\,p_i
          - \penaltyfntype{\psf,i}(p_i,\pi_{-i})
          \}
          + \frac{1}{\epsilon_i}\nu_i(\pi_i).$

        \State 
        $
          \widehat y_{i,h}^t
          :=
          r_{i,h}(x_h^{t},a_h^{t})
          +
          \widehat\rho_{i,h}^{\sf e,k} \left(
            \widehat V_{i,h+1}^{\epsilon_i}(\pi_{h+1}^{k};x_{h+1})
            \ \middle|\ x_h^{t},a_h^{t}
          \right),
          \qquad t=1,\dots,k-1.
        $

        \State $w_{h}^{i,k}\gets(\Lambda_h^k)^{-1}\sum_{t=1}^{k-1}\phi(x_h^{t},a_h^{t},h)\,\widehat y_{i,h}^t$

        \State $Q_{h}^{i,k}(x,a)\gets
        \min\Big\{
          (w_{h}^{i,k})^\top\phi(x,a,h)
          + \beta\sqrt{\phi(x,a,h)^\top(\Lambda_h^k)^{-1}\phi(x,a,h)},
          \; B
        \Big\}$
      \EndFor
    \EndFor

    \State Execute episode $k$ with $\pi_h^k(x)\leftarrow \texttt{RQRE}_{\varepsilon}(Q_{h}^{1,k}(x,\cdot),\dots,Q_{h}^{n,k}(x,\cdot))$.
  \EndFor
  \end{algorithmic}
\end{algorithm}

Consider an $n$-player episodic Markov game of the form $\mc{MG}=(\X, \A, H, \Ps, r)$ as defined in Section~\ref{sec:preliminaries}.
To address large or continuous state spaces, we consider the setting of linear Markov games, in which the transition kernels and reward functions are assumed to be linear. This assumption implies that the action-value functions are linear, as we will show. As discussed in \citet{jin2020ProvablyEfficient}, this is not an assumption on the policy class but rather a statistical modeling assumption on how the data are generated.

\begin{assumption}\label{ass:linear-markov-game}
Consider a Markov game $\mc{MG}=(\X, \A, H, \Ps, r)$, and let $\pi^\star$ be an RQRE.
\begin{enumerate}[label={\alph*.}]
    \item \textbf{The transition kernels and reward functions are linear.} That is, there exists a feature map $\phi:\X\times \A\times [H]\to \mb{R}^d$ such that for every $h\in [H]$, there exist $d$ unknown signed measures $(\mu_h^1, \ldots, \mu_h^d)$ over $\mc{X}$ and an unknown vector $\theta_{i,h} \in \mb{R}^d$ satisfying
$
\mc{P}_h(x'\mid x,a)=\la \phi(x,a,h),\mu_h(x')\ra$, and
$r_{i,h}(x,a)=\la \phi(x,a,h),\theta_{i,h}\ra$
for all $x,x'\in \mc{X}$, $a\in \mc{A}$, and $h \in [H]$. Without loss of generality, we assume that $\|\phi(x,a,h)\|\leq 1$ for all $(x,a,h) \in \mc{X}\times \mc{A} \times [H]$ and $\max\{\max_{x\in \mc{X}}\|\mu_h(x)\|, \|\theta_{i,h}\|\}\leq \sqrt{d}$ for all $h\in [H]$.
\item \textbf{Realizability.} For each player $i\in[n]$ and stage $h\in[H]$, there exists a vector
$w_{i,h}^\star\in\mb{R}^d$ such that for all $(x,a)\in\mc{X}\times\mc{A}$, the action-value map
$Q_{h}^{i}(x,a; \pi^\star)=\la \phi(x,a,h),\, w_{h}^{i,\star}\ra$.
\end{enumerate}
\end{assumption}

Critically, this assumption implies that for any policy, the action-value functions are linear in the feature map $\phi$. Therefore, when designing multi-agent reinforcement learning algorithms, it suffices to focus on linear action-value functions. We refer the reader to \citet{jin2020ProvablyEfficient} for common examples of games encompassed by Assumption~\ref{ass:linear-markov-game}.

\paragraph{Regret Notion.} Let $\pi_{i,h}^{\rm br}(\cdot\mid x)
\in
\arg\max_{\mu_i\in\Delta(\mc A_i)}
V_{i,h}^{\epsilon_i}
\big(
(\mu_i,\pi_{-i,h}(\cdot\mid x));
x
\big)$,
is the smoothed best response such that future continuation uses $\pi_i^{\rm br}$ recursively in 
\begin{align*}
    Q_{i,h}^{(\pi_i',\pi_{-i})}(x,a)&:=
r_{i,h}(x,a)+\rho_{i,h}^{\sf e}
(
V_{i,h+1}^{\epsilon_i}((\pi_i',\pi_{-i});X_{h+1})
\ \big|\ x,a
)\\
V_{i,h}^{\epsilon_i}((\pi_i',\pi_{-i});x)
&:=
\mc R^{\sf pol}_{i,h}
\big(
Q_{i,h}^{(\pi_i',\pi_{-i})},
(\pi_i',\pi_{-i})_h;
x
\big)+\frac{1}{\epsilon_i}\nu_i(\pi_i').
\end{align*}
 with terminal condition $V_{i,H+1}^{\epsilon_i}\equiv 0$.
Define the so-called \emph{exploitability regret} by
\[
\reg(K)
:=\textstyle
\sum_{k=1}^K
\max_{i\in[n]}
\Big(
V_{i,1}^{\epsilon_i}
\big(\pibr_i(\pi_{-i}^k),\pi_{-i}^k;\,s_0\big)
-
V_{i,1}^{\epsilon_i}
\big(\pi^k;\,s_0\big)
\Big),
\]
where the episode-wise regret term is zero precisely when the policy $\pi_k$ is an RQRE, which itself is commonly unique under sufficient regularity~\cite{mazumdar2025tractable}. 
This regret is also analogous to the regret notion considered in prior works \cite{cisneros2023finite}. Bounding the exploitability $\reg(K)$ implies that each $\pi^k$ is \emph{approximately unilaterally stable}: no player can gain much by deviating in the risk-regularized sense. Indeed, it is algorithmically natural here as agents are computing
$\pi^k$ via an approximate RQRE mapping from an optimistic $Q^k$.

\paragraph{RQRE-OVI (Algorithm~\ref{alg:pr-er-rqre-ovi}).} 
The algorithm performs optimistic value iteration over episodes, solving an approximate RQRE at each stage game via the subroutine $\texttt{RQRE}_\varepsilon$. This replaces the Nash oracle used in prior work \cite{cisneros2023finite} with a unique, Lipschitz stable, and computationally tractable equilibrium computation, and risk measures for environmental, policy and opponent uncertainty. The environment-risk operator $\rho_{i,h}^{\sf e}$ is estimated from samples, for instance via finite dual discretization or closed-form evaluation as in the entropic case (Example~\ref{ex:entropic_risk}); concrete instantiations and error analysis are provided in Appendix~\ref{sec:risk-estimation}.

\begin{theorem}[Regret bound]
\label{thm:per_rqre_ovi_full}
Consider Algorithm~\ref{alg:pr-er-rqre-ovi}, and let 
 Assumption~\ref{ass:linear-markov-game} holds. Suppose that there exists constant $B>0$ such that 
$0\le V_{i,h}^{\epsilon_i}(\pi;x)\le B$ for all $(i,h,x,\pi)\in [n]\times [H]\times \mc{X}\in \Pi$, and suppose the following approximations are given:
\begin{itemize}[itemsep=0pt, topsep=5pt]
    \item \textbf{Environment \& Policy Risk.} There exists $\varepsilon_{\sf env}\ge 0$ and $\varepsilon_{\sf pol}$ such that
for all bounded $X$,
$
0\leq
\rho_{i,h}^{\sf env}(X)
-
\widehat\rho_{i,h}^{\,k}(X)
\leq
\varepsilon_{\sf env}$, and 
for all $(i,h,x,\pi)\in [n]\times [H]\times \mc{X}\in \Pi$,
\[
0 \leq
V_{i,h}^{\epsilon_i}(\pi;x)
-
\widehat V_{i,h}^{\epsilon_i}(\pi;x)
\le
\varepsilon_{\sf pol}.\] 
\item \textbf{Stage RQRE Approximation.} For each $(h,x)\in [H]\times \mc{X}$ and $Q$ the computed stage equilibrium
$\tilde\pi_h(\cdot|x;Q)$ satisfies
$|
V_{i,h}^{\epsilon_i}\big(\tilde\pi_h;x,Q\big)
-
V_{i,h}^{\epsilon_i}\big(\pi_h^\star;x,Q\big)|
\leq
\varepsilon_{\sf eq}$ for each $i$, where $\pi_h^\star$ is the exact stage equilibrium.
\end{itemize}
Then with probability at least $1-\delta$, the estimate holds:
\begin{equation}
\label{eq:main_regret_bound_general}
\Reg(K)
\;\le\;
\tilde{\mc O} \left(
L_{\sf env}\,B\,
\sqrt{K\, d^{3}\, H^{3}}
\right)
\;+\;
K H
\Big(
\varepsilon_{\sf env}
+
L_{\sf env}\big(
\varepsilon_{\sf pol}
+
\varepsilon_{\sf eq}
\big)
\Big),
\end{equation}
where $\tilde{\mc O}(\cdot)$ hides logarithmic factors in
$(n,d,H,K,1/\delta)$.
\end{theorem}
This theorem builds on the analysis of optimistic value iteration with linear function approximation   \cite{cisneros2023finite,HeMinimaxOptimalPMLR2021,jin2020ProvablyEfficient}, sharing the same elliptical potential and statistical machinery, but requires several new ingredients to handle the approximation errors introduced by the risk estimation and approximate equilibrium computation. 

First, the covering number for the risk induced value class (Lemma~\ref{lem:covering_induced_value_class}) must be established for the RQRE smoothed best-response structure rather than for an exact Nash oracle.  
Unlike single-agent linear MDPs, the continuation values depend
on the data-dependent mapping $Q \mapsto \tilde{\pi}(Q)$
and are subsequently transformed by nonlinear risk operators.
As a result, the regression targets do not form a simple linear class,
preventing a direct linear-bandit-style concentration argument.
This requires controlling the complexity of the risk-induced value-function class
via a covering argument.
Second, the optimism guarantee (Lemma~\ref{lem:optimism_full}) must account for the combined effect of the confidence bonus at each episode and the  approximate RQRE---rather than an exact Nash equilibrium---across all stages, translating the environment-risk, policy-risk, and equilibrium approximation errors into an additive optimism deficit. To the end, Lemma~\ref{lem:use_eps_eq_in_regret} relates the deviation gap at each stage to an optimistic $Q$-advantage plus the exploration bonus.
Finally, a stage-wise performance difference recursion propagates these per-step errors through the risk-sensitive Bellman backup, where the Lipschitz constant $L_{\sf env}$ of the environment risk operator governs how approximation errors compound across stages. We provide more detail on novelty in addition to the full proof in Appendix~\ref{app:rqre-ovi-proofs}.

%------------------------------------------------------------
\subsection{Regret Bounds in terms of Rationality and Risk Preferences}
\label{sec:regret_risk_rationality}

In Algorithm~\ref{alg:pr-er-rqre-ovi}, we compute stage-wise $\varepsilon_{\sf eq}$--approximate RQRE. Here we give some sense in which this approximation depends on problem parameters. 
For example, a natural class of algorithms are \emph{no-regret}; if agents run generic no-regret dynamics, then $\varepsilon_{\sf eq}$  is $\varepsilon_{\mathrm{eq}}(T)=
\tilde{\mathcal O}(
\sum_i \sqrt{\log(|\mathcal A_i|)/T})$  to a coarse correlated equilibrium \cite{cesa2006prediction,blum2007external}.

In the special case where  the
\emph{policy-space} risk penalty is entropic (KL) and the player policy is
entropy-regularized, there are number of methods one can employ including no-regret learning methods such as extra-gradient or mirror-prox, multiplicative weights, or iterative best response. 
In this case, we can obtain more explicit bounds which reveal how the risk sensitivity and bounded rationality parameters influence different performance criteria. 
To this end, fix a stage $h$ and state $x$. For player $i$, let $\nu_i\equiv \Phi$ be an entropy regularizer, and let $\penaltyfntype{\psf,i}\equiv \frac{1}{\tau_i}\text{KL}$.
Then policy-risk value functional admits the closed form
\begin{equation}
\label{eq:Veps-entropic-closed}
V_{i,h}^{\epsilon_i,\tau_i}(\pi;x;Q)
=
-\frac{1}{\tau_i}
\log\big(
\sum_{a_{-i}}
\pi_{-i}(a_{-i}|x)\,
\exp (
-\tau_i\,u_i^\pi(x,a_{-i};Q)
)\big)
+\tfrac{1}{\epsilon_i}\entropy(\pi_i(\cdot|x)),
\end{equation}
where $u_i^\pi(x,a_{-i};Q):=\sum_{a_i}\pi_i(a_i|x)\,Q_i(x,a_i,a_{-i})$.
Here $Q_i(x,:)$ denotes the tensor slice $a\mapsto Q_i(x,a)$,
and $\pi_i^\top Q_i\,p_i$ denotes $\sum_{a_i,a_{-i}}\pi_i(a_i|x)\,Q_i(x,a_i,a_{-i})\,p_i(a_{-i})$.

The entropy regularization also impacts the range of values for value function. 
Indeed, assume stage rewards satisfy $r_{i,h}(x,a)\in[0,1]$ and $V_{i,H+1}^{\epsilon_i,\tau_i}\equiv 0$.
Then for any policy profile $\pi$ and any $(h,x)\in [H]\times \mc{X}$,
the cumulative (undiscounted) reward-to-go is at most $H$.
Moreover, the entropy satisfies
\[
0 \le \entropy(\pi_i(\cdot|x)) \le \log|\mc A_i|.\]
Therefore the additive regularization term is bounded as
\[
0 \le \entropy(\pi_i(\cdot|x))/\epsilon_i \le \log(|\mc A_i|)/\epsilon_i.\]
Since the entropic risk term in \eqref{eq:Veps-entropic-closed} is a (KL-regularized) log-sum-exp
aggregation of bounded $Q$-values, it preserves the same scale as the underlying payoff
(up to the horizon factor). Consequently, when $Q$ is clipped to $[0,B]$ with 
\[B:=\max_{i\in[n]}B_i\quad \text{where} \quad B_i
:=
H\left(1+\log(|\mathcal A_i|)/\epsilon_i\right),\]
we have
$
0 \leq V_{i,h}^{\epsilon_i,\tau_i}(\pi;x;Q)\leq B$
for all $(i,h,x)\in[n]\times [H]\times \mc{X}$ and policies $\pi$.
Here the factor $H$ matches the horizon scaling of cumulative costs, while
$\log|\mc A_i|/\epsilon_i$ captures the maximum magnitude of entropy regularization.

\begin{corollary}[Regret under entropic policy-risk and regularization]
\label{cor:per_rqre_ovi_entropic}
Under the assumptions of Theorem~\ref{thm:per_rqre_ovi_full}, specialize the policy-side
risk/value operator to \eqref{eq:Veps-entropic-closed}
and set the policy regularizer to $\nu_i(\mu_i)=\entropy(\mu_i)$.
Suppose the stage solver returns a policy with duality gap
at most $\Delta_{\sf eq}$ uniformly over $(h,x)\in [H]\times \mc{X}$ and matrices $Q$ so that 
$\varepsilon_{\sf eq}
\leq
c_{\sf eq}\,B\sqrt{\Delta_{\sf eq}/\tau_{\min}}$ where $\tau_{\min}:=\min_{i\in[n]}\tau_i$, and
for an absolute constant $c_{\sf eq}>0$ the dependence on $\tau_{\min}$ is explicit.
With probability at least $1-\delta$,
\begin{equation}
\label{eq:regret_entropic_cor}
\textstyle\reg(K)
\leq
\tilde{\mc O}(
L_{\sf env}\,B\,
\sqrt{K\, d^{3}\, H^{3}})+
K H
\big(
\varepsilon_{\sf env}
+
L_{\sf env}\varepsilon_{\sf pol}
+
L_{\sf env}\,c_{\sf eq}\,B\sqrt{\Delta_{\sf eq}/\tau_{\min}}
\big).
\end{equation}
\end{corollary}
For example, if agents run no-regret as mentioned above, then $\Delta_{\sf eq}\leq \tilde{O}(\sum_i\sqrt{\log(|\mc{A}_i|)/T})$ where the equilibrium solver $\texttt{RQRE}_{\varepsilon}$ is run for $T$ steps. 
If the  variational inequality is strongly monotone with modulus 
$\mu=\mu(\tau)$ (cf.~Corollary~\ref{cor:rqe_lipschitz}) and Lipschitz constant $L = L(\tau)$, 
then an extragradient/Mirror-Prox solver run for $T$ steps yields
$\varepsilon_{\mathrm{eq}}(T)\leq
C \exp(-T\mu/L)$.
Plugging this into Theorem~\ref{thm:per_rqre_ovi_full} replaces the additive equilibrium-solver term by
$K H L_{\mathsf{env}}C \exp(-T\mu/L)$.
 In Appendix~\ref{app:technical_lemmas}, we provide more precise examples of these approximations.

Corollary~\ref{cor:per_rqre_ovi_entropic} reveals the following parameter tradeoffs. Increasing $\epsilon_i$ (weaker entropy regularization) decreases $B(\epsilon)$ and thus both the leading statistical term and the equilibrium-solver error. For fixed $\Delta_{\sf eq}$, increasing $\tau_{\min}$ decreases the solver-error contribution as $O(\sqrt{\Delta_{\sf eq}/\tau_{\min}})$. However, recovering risk neutrality ($\tau\to0$) requires scaling $\Delta_{\sf eq}=\mc{O}(\tau_{\min}\varepsilon_{\sf eq}^2/B^2)$ to maintain a fixed $\varepsilon_{\sf eq}$. In Appendix~\ref{app:explicit_regret} we analyze the risk approximation errors as a function of $\epsilon$ and $\tau$.

\section{Distributional Robustness \& Stability of RQRE}
\label{sec:dro-rqe}
Beyond the ability to capture phenomena of natural learners, a key motivation for RQRE over Nash is that the additional structure introduced by bounded rationality and risk sensitivity which yield  both distributionally robustness and stability  under approximation errors.

\subsection{Distributional Robustness of RQRE}
\label{subsec:dro_connections}
In parallel to RQRE, several distributionally robust equilibrium concepts have been proposed, including strategically robust equilibrium via optimal transport~\cite{lanzetti2025strategicallyrobustgametheory} and distributionally robust Nash equilibria~\cite{alizadeh2025distributionally}.
The dual representation of the risk-adjusted loss (Section~\ref{sec:preliminaries}) reveals a direct connection between RQRE and distributionally robust optimization (DRO). In particular, the risk-adjusted loss \eqref{eq:risk_induced_loss} admits a natural interpretation as a \emph{penalized DRO} objective: player $i$ evaluates a strategy under an adversarial distribution over opponents’ actions while paying a convex penalty for deviating from the reference distribution $\pi_{-i}$.
Consequently, RQRE corresponds to a fixed point of regularized best responses under penalized distributional robustness. Classical ambiguity-set DRO arises as a special case of
\eqref{eq:risk_induced_loss}.

\begin{proposition}[RQRE generalizes ambiguity-set robustness]
\label{prop:rqe_strict_generalization_with_existence_comment}
Consider an $n$-player normal-form game with utilities $u_i:\mc A\to\R$.
\begin{enumerate}[label={\alph*.}, topsep=2pt, itemsep=0pt]
    \item \textbf{Any  hard-DRO equilibrium is an RQRE.}  
The risk-adjusted loss $\ell_i$ defined for each  fixed $\pi_{-i}$ by a proper convex and lower semi-continuous map $p_i\mapsto\penaltyfn_i(p_i,\pi_{-i})$ reduces to the ambiguity-set DRO loss
\[\ell_i(\pi_i,\pi_{-i})
=\sup_{p_i\in \mc P_i(\pi_{-i})}\E_{a_{-i}\sim p_i}\big[-u_i(\pi_i,a_{-i})\big].\]
\item \textbf{Convex-penalty RQRE strictly generalize indicator-penalty RQRE.}
There exist proper convex penalties $\penaltyfn_i(\cdot,\pi_{-i})$ for which the induced loss
$\ell_i(\cdot,\pi_{-i})$ cannot be represented as a hard-DRO loss for any ambiguity set correspondence
$\mc P_i(\cdot)$.
\end{enumerate}
\end{proposition}
\begin{proof}[Proof Sketch]
For any ambiguity set correspondence $\mc P_i(\pi_{-i})\subseteq \Delta(\mc A_{-i})$, define the
indicator penalty
\[
\penaltyfn_i(p_i,\pi_{-i}) := \iota_{\mc P_i(\pi_{-i})}(p_i).
\]
The first claim holds by deriving the stated reduction using $\penaltyfn_i(\cdot,\pi_{-i})=\boldsymbol{1}_{\mc P_i(\pi_{-i})}$ and the convex measure duality (cf.~Theorem~\ref{thm:dual-rep}). 
Indeed, given $\penaltyfn_i(\cdot,\pi_{-i})=\boldsymbol{1}_{\mc P_i(\pi_{-i})}$, for any function $Z$ on $\mc A_{-i}$,
\[
\sup_{p_i\in\Delta(\mc A_{-i})}\{\mb{E}_{p_i}[Z]-\boldsymbol{1}_{\mc P_i(\pi_{-i})}(p_i)\}
=
\sup_{p_i\in\mc P_i(\pi_{-i})}\mb{E}_{p_i}[Z],
\]
which yields the stated reduction (with $Z(a_{-i})=-u_i(\pi_i,a_{-i})$). Taking fixed points of the best-response maps
gives the equilibrium inclusion.

To see that the second claim holds, consider the KL penalty
$
\penaltyfn_i(p_i,\pi_{-i})=\tfrac{1}{\tau_i}\KL(p_i\|\pi_{-i})$ with $\tau_i>0$ which
induces the entropic risk functional of Example~\ref{ex:entropic_risk}---i.e.,
 \[
 \ell_i(\pi_i,\pi_{-i})
 =
 \frac{1}{\tau_i}\log
 \mathbb E_{a_{-i}\sim \pi_{-i}}
 \big[\exp(\tau_i(-u_i(\pi_i,a_{-i})))\big].
 \]
This is \emph{not representable} as \[\sup_{p_i\in\mc P_i(\pi_{-i})}\E_{p_i}[-u_i(\pi_i,a_{-i})]\quad\text{ for any set}\quad
\mc P_i(\pi_{-i}).\]
Any hard-DRO mapping has the form $Z\mapsto \sup_{p\in\mc P}\E_p[Z]$ and is positively homogeneous:
$\sup_{p\in\mc P}\E_p[tZ]=t\sup_{p\in\mc P}\E_p[Z]$ for all $t\ge 0$.
The entropic mapping $Z\mapsto (1/\tau)\log \E_{\pi_{-i}}[\exp(\tau Z)]$ is not positively homogeneous in general: e.g., take $Z\in\{0,1\}$ with probability $1/2$ each under $\pi_{-i}$, so that
$
\tfrac{1}{\tau}\log((1+e^{2\tau})/2)\neq \tfrac{2}{\tau}\log((1+e^{\tau})/2)$.
Hence no ambiguity set $\mc P$ can make these mappings coincide for all $Z$, proving strictness.
\end{proof}
In particular, part a shows that the strategically robust equilibrium concept based on ambiguity sets proposed by
\citet{lanzetti2025strategicallyrobustgametheory}
arises as a special case of RQRE obtained by choosing
$\mathcal{P}_i(\pi_{-i})$ to be an optimal-transport ball around
$\pi_{-i}$.
In fact, the following  inclusions are generally true:
\begin{equation}\label{eq:robustness_hierarchy}
\begin{aligned}
\text{RQRE}
&\;\supset\;
\text{Penalized DRO Equilibrium}
\;\supset\;
\text{Hard ambiguity-set DRO Equilibrium}
\;\supset\;
\text{Nash Equilibrium}
\end{aligned}
\end{equation}
\begin{remark}[Existence versus expressivity]
\label{rem:existence_vs_expressivity}
Indicator penalties $\boldsymbol{1}_{\mc P}$ are convex but not Lipschitz on $\Delta(\mc A_{-i})$.
The Lipschitz assumption on $\penaltyfn_i$ is used to obtain existence (and stability) of RQRE. Proposition~\ref{prop:rqe_strict_generalization_with_existence_comment}
is an \emph{expressivity} statement: it shows that hard ambiguity-set robustness is contained in the RQRE class,
even though such penalties lie outside the Lipschitz-penalty subclass for which our existence theory applies.
\end{remark}
Now for Markov games, we show that RQRE are distributionally robust (see Appendix~\ref{sec:dro-proofs} for details).
\begin{proposition}[RQRE in Markov Games are Distributionally Robust]
\label{prop:dro}
Let $\pi^\star$ be an RQRE of the risk-sensitive Markov game $\mc{MG}$ with penalty functions $\penaltyfntype{\psf,i}$, $\penaltyfntype{\esf,i}$ as defined in the dual representation (Theorem~\ref{thm:dual-rep}) and satisfying Assumption~\ref{ass:lipschitz-penalty}. Then $\pi^\star$ is distributionally robust with respect to the agents policy, the opponents strategies, and the stochastic environment transitions where Lagrange multipliers that uniquely determine the size of the corresponding ambiguity sets.
\end{proposition}

\subsection{RQRE Admit Stability Properties that Yield Performance Guarantees}
Another main advantage of RQRE over Nash equilibrium for value-iteration-based algorithms with function approximation is stability: small errors in the estimated $Q$-functions should produce small changes in the computed equilibrium policy. 
Since RQRE is unique by construction, the map from payoff tables to equilibrium policy is a well-defined function. In contrast, the Nash correspondence is set-valued in general-sum games, and any single-valued selection can be discontinuous \citep{fudenberg1991game, vanDamme1987-VANSAP-14}. This distinction is the source of the stability gap.

\begin{corollary}[Lipschitz Stability of RQRE]
\label{cor:rqe_lipschitz}
Suppose each player $i \in [n]$ employs an $\alpha_\nu$-strongly convex  regularizer $\nu_i$ , and an $\alpha_{\psf}$-strongly convex  policy-risk penalty $\varphi_{\psf,i}(\cdot,q)$ for any $q$.  The regularized risk-adjusted objective is $\mu:= \frac{1}{\epsilon} + \tau  \alpha_{\psf} > 0$ strongly concave in $\pi_i$, and there exists a constant $c > 0$ such that for any $Q, \widetilde{Q}$, the estimate holds:
\[\|\pi^{\tt{RQRE}}(Q) - \pi^{\tt{RQRE}}(\widetilde{Q})\|_1 \leq \frac{c}{\mu} \|Q - \widetilde{Q}\|_\infty.\]
In particular, for entropy regularization and policy risk, $\alpha_\nu = 1/\epsilon$ and $\alpha_{\psf} = \tau$, so $\mu = 1/\epsilon + \tau^2$.
\end{corollary}
The Lipschitz constant $c/\mu$ is controlled entirely by the bounded rationality and the risk-aversion parameters, both of which are design choices of \texttt{RQRE-OVI}. Stronger regularization (larger $\bddrat$) or greater risk aversion (smaller $\risk$) yields a smaller $c/\mu$, improving stability at the cost of regret performance. In contrast, Nash equilibria in general-sum games are not guaranteed to be unique, and any algorithm that must therefore perform stage-wise equilibrium selection suffers.  
The following example illustrates this phenomenon.
\begin{example}[Instability of Nash Equilibrium under Multiplicity]
\label{ex:nash_unstable}
Consider the two-player symmetric coordination game with payoff matrix $Q(\alpha) = \bigl(\begin{smallmatrix} 1 & 0 \\ 0 & \alpha \end{smallmatrix}\bigr)$ for any $\alpha > 0$. This game has three Nash equilibria, including two pure-strategy equilibria $(e_1, e_1)$ and $(e_2, e_2)$. Standard equilibrium selection criteria (e.g., risk-dominance or payoff-dominance) select $(e_1, e_1)$ when $\alpha < 1$ and $(e_2, e_2)$ when $\alpha > 1$.
Hence any Lipschitz constant $L$ would have to satisfy
\[
L \ge \frac{\|\pi^{\tt NE}(Q(1-\varepsilon))-\pi^{\tt NE}(Q(1+\varepsilon))\|_1}
{\|Q(1-\varepsilon)-Q(1+\varepsilon)\|_\infty}
= \frac{2}{2\varepsilon}
= \frac{1}{\varepsilon},
\]
which diverges as $\varepsilon \to 0$. Thus no finite uniform Lipschitz constant exists.
\end{example}
The instability is not due to a particular selection rule, but a consequence of equilibrium multiplicity itself. 
Under function approximation, estimated payoff tables are necessarily perturbed, and when the stage game admits multiple equilibria, arbitrarily small errors may cause a Nash oracle to return different policies across episodes~\citep{cisneros2023finite}. \texttt{RQRE-OVI} \emph{avoids this failure mode}: the regularization inherent in the formulation guarantees a unique equilibrium at every stage, and Corollary~\ref{cor:rqe_lipschitz} ensures it varies smoothly with estimated payoffs. Such stability properties also translate to a form of policy convergence. 
\begin{proposition}[Robustness of \texttt{RQRE-OVI} under Linear Function Approximation]
\label{prop:rqe_ovi_robust}
Consider Algorithm~\ref{alg:pr-er-rqre-ovi} with parameter estimates $\widehat{w}_h$ satisfying $\max_{i \in [n]}\|\widehat{w}_{h}^{i} - w_{h}^{i,\star}\|_2 \leq \delta_h$ for all $h \in [H]$ where $w^{i,\star}_h$ are the true parameters (Assumption~\ref{ass:linear-markov-game}). The induced RQRE policies satisfy $\|\widehat{\pi}_h - \pi_h^\star\|_1 \leq \frac{c}{\mu} \delta_h$ at each stage where $\pi^\star=\{\pi_h^\star\}_{h\in [H]}$ is the RQRE of the Markov game and $c>0$ is an absolute constant, and 
\[\sup_x |V_{1}^{\widehat{\pi}}(x) - V_{1}^{\pi^\star}(x)| = \mc{O}\left(\frac{B\bar{\delta}}{\mu} \sum_{t=0}^{H-1} L_{\sf env}^t\right)\] when $\delta_h \leq \bar{\delta}$ for all $h\in [H]$. Moreover, if the environment risk operator is non-expansive in $\|\cdot\|_\infty$,
i.e., $L_{\sf env}\le 1$,
then 
$\sum_{t=0}^{H-1} L_{\sf env}^t \le H$,
and hence
$\sup_x |V_1^{\widehat{\pi}}(x)-V_1^{\pi^\star}(x)|
=\mc{O}(HB\bar\delta/\mu)$. 
\end{proposition}
Observe that the $1$-Lipschitz continuity property of the environment risk operator holds in a number of common settings including entropic risk and many coherent risk maps.
The complete finite-horizon analysis is in Appendix~\ref{sec:rqe_stability_linear_fa}.

\section{Numerical Experiments}
We evaluate \texttt{RQRE-OVI} on two multi-agent coordination benchmarks\footnote{Code is available at \url{https://jakeagonzales.github.io/linear-rqe-website/}.}: a \emph{dynamic} Stag Hunt environment modified from the Melting Pot suite \cite{agapiou2022melting}, and Overcooked \cite{gessler2025overcookedv}, specifically using the implementation from JaxMARL \cite{flair2024jaxmarl}. For each environment, we compare \texttt{RQRE-OVI} across a range of risk-aversion parameters against \texttt{QRE-OVI} (risk-neutral) and \texttt{NQ-OVI} \cite{cisneros2023finite}.  We asses performance under three scenarios:

\begin{figure}[htbp]
    \centering
    \begin{subfigure}[b]{0.45\textwidth}
        \centering
        \includegraphics[width=\textwidth]{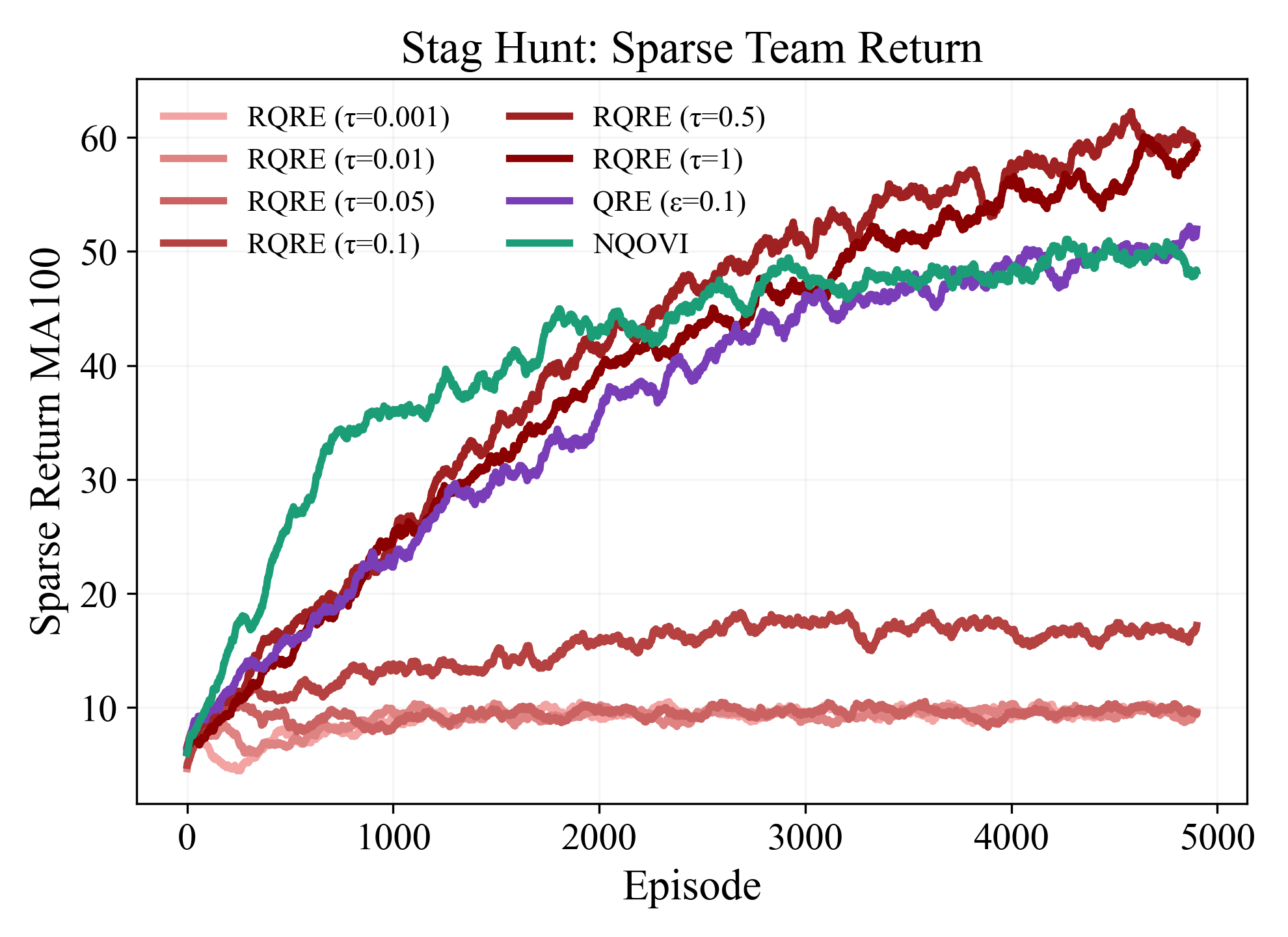}
    \end{subfigure}
    \begin{subfigure}[b]{0.45\textwidth}
        \centering
        \includegraphics[width=\textwidth]{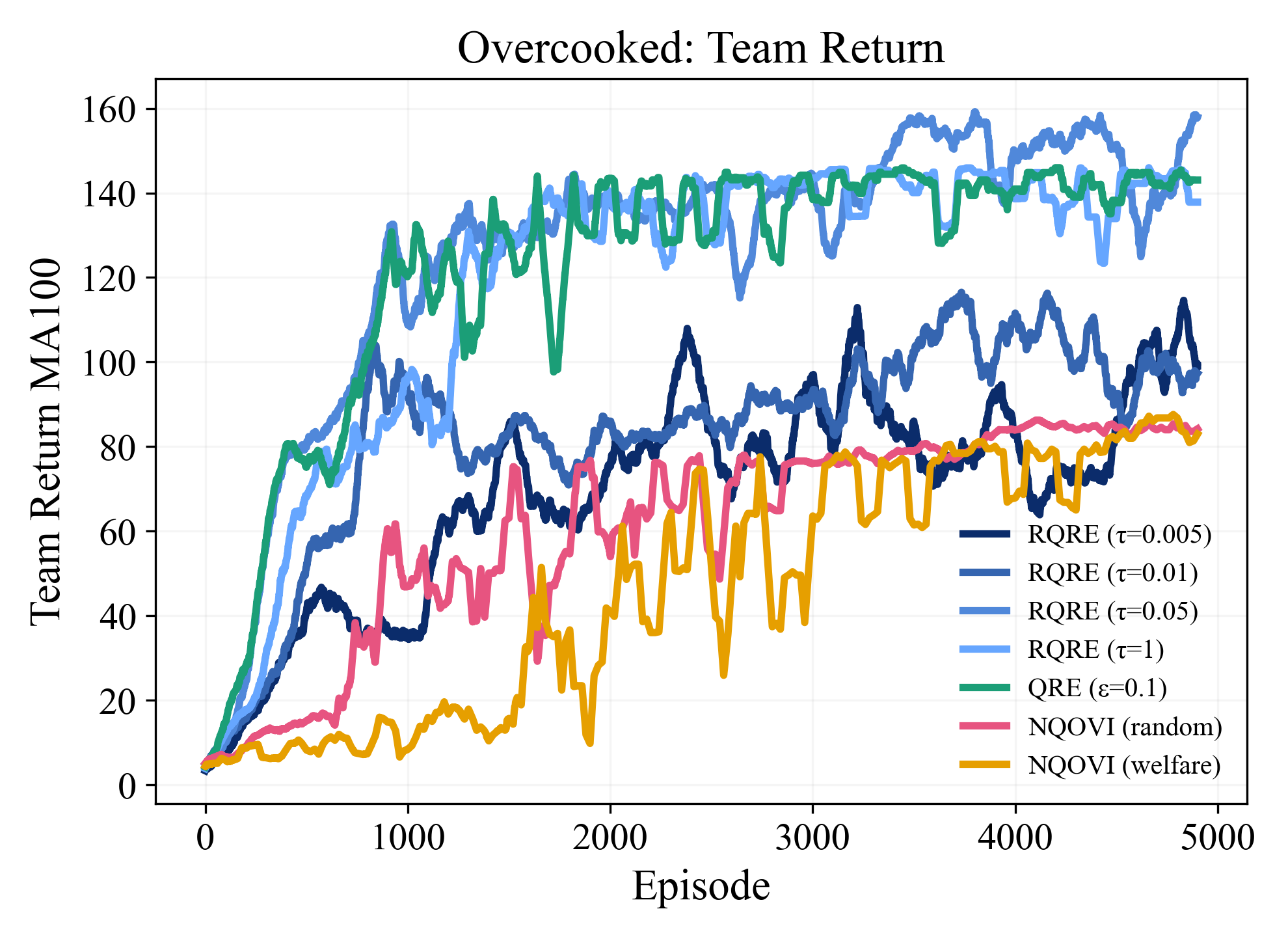}
    \end{subfigure}
    \caption{\textbf{Self-play team return during training.} Moving average of team return over episodes for Stag-Hunt (left) and Overcooked (right). In Stag-Hunt, higher $\risk$ drives agents toward the payoff dominant (\texttt{stag}, {\tt stag}) outcome, while lower $\risk$ yields the more robust risk-dominant ({\tt hare}, {\tt hare}) outcome. In Overcooked, all RQRE variants and QRE converge to comparable team returns, with Nash variants reaching similar or slightly lower levels.}
    
    \label{fig:team_return}
\end{figure}

\begin{enumerate}[itemsep=0pt,topsep=2pt]
    \item Self-play, where agents train and evaluate with their own partner; 
    \item Cross-play with a perturbed partner, where an ego agent's partner deviates to a random or fixed action with probability $\delta$ at test time; 
    \item Cross-play with an unseen partner, where agents trained under different algorithms are paired at test time.
\end{enumerate}
 Across both environments, we find  RQRE achieves competitive self-play performance while producing \emph{substantially more robust} behavior in cross-play, with the risk-aversion parameter $\tau$ governing a tradeoff between peak self-play return and worst-case robustness. A detailed description of the environments, training procedures, and hyperparameters is in Appendix~\ref{sec:additional-experiment-details}. 

\subsection{Environment~1: Dynamic Stag-Hunt Game}
Stag Hunt is a classical coordination game in which two agents must choose between a safe, low-payoff action (hare) and a risky, high-payoff action (stag) that succeeds only under mutual coordination. We implement a dynamic, spatial variant of this game leveraging the Melting Pot framework \cite{agapiou2022melting}. In our environment, agents navigate a $9\times9$ grid to collect resources and interact to resolve payoffs. The payoff matrix follows the standard structure: mutual stag yields $(4,4)$, mutual hare $(2,2)$, and stag-hare mismatch gives $(0,2)$.

\paragraph{\emph{\textbf{Self-play reveals a risk-return tradeoff.}}} Under self-play (Figure~\ref{fig:team_return}, left) the risk aversion parameter $\risk$ directly controls which equilibrium agents coordinate on. High $\risk$ agents ($\risk \geq 0.5$) converge to the payoff dominate outcome with team returns near 60, while low $\risk$ agents ($\risk \leq 0.05$) settle on the risk dominant outcome near 20. This is the tradeoff predicted by the theory---greater risk aversion biases agents toward strategies robust to coordination failure,  at the cost of forgoing the higher-payoff but fragile stag equilibrium. \texttt{QRE-OVI} and \texttt{NQ-OVI} achieve intermediate returns converging to the payoff dominant strategy. 

\paragraph{\emph{\textbf{Risk-averse agents degrade gracefully under partner perturbation.}}} Under cross-play with a perturbed partner (Figure~\ref{fig:perturbed_partner_retention}, left), more risk-averse agents (low $\risk$) maintain high retention across all noise levels, as their convergence to the risk-dominant hare equilibrium renders them inherently insensitive to partner deviations. In contrast, less risk-averse agents coordinate on the fragile stag equilibrium and suffer sharp retention drops as $\delta$ increases. NQOVI and QRE exhibit similar fragility, having also converged to the payoff-dominant strategy.
This is precisely the robustness--performance tradeoff predicted by the theory.

\begin{figure}[htbp]
    \centering
    \begin{subfigure}[b]{0.45\textwidth}
        \centering
        \includegraphics[width=\textwidth]{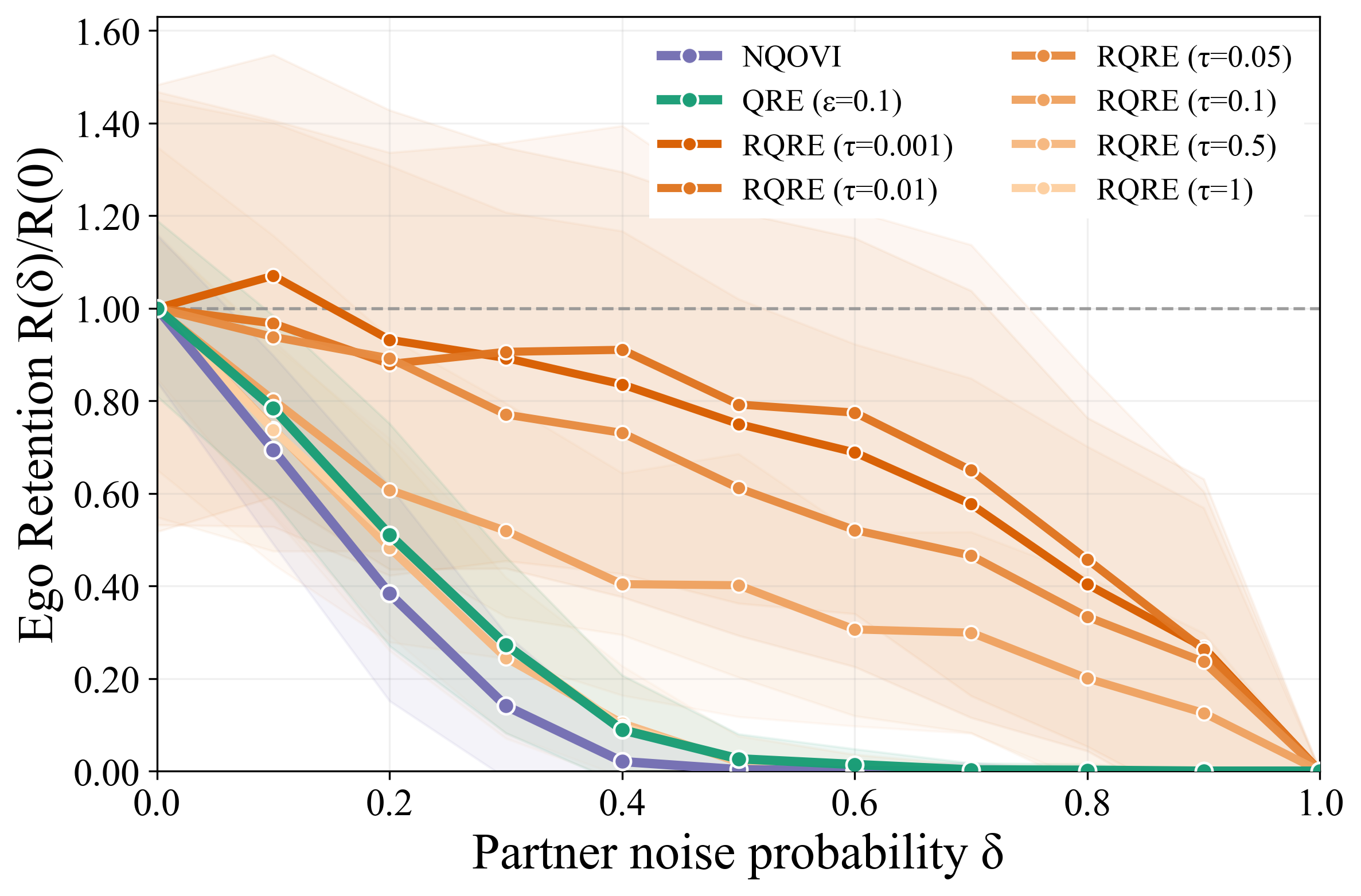}
    \end{subfigure}
    \begin{subfigure}[b]{0.45\textwidth}
        \centering
        \includegraphics[width=\textwidth]{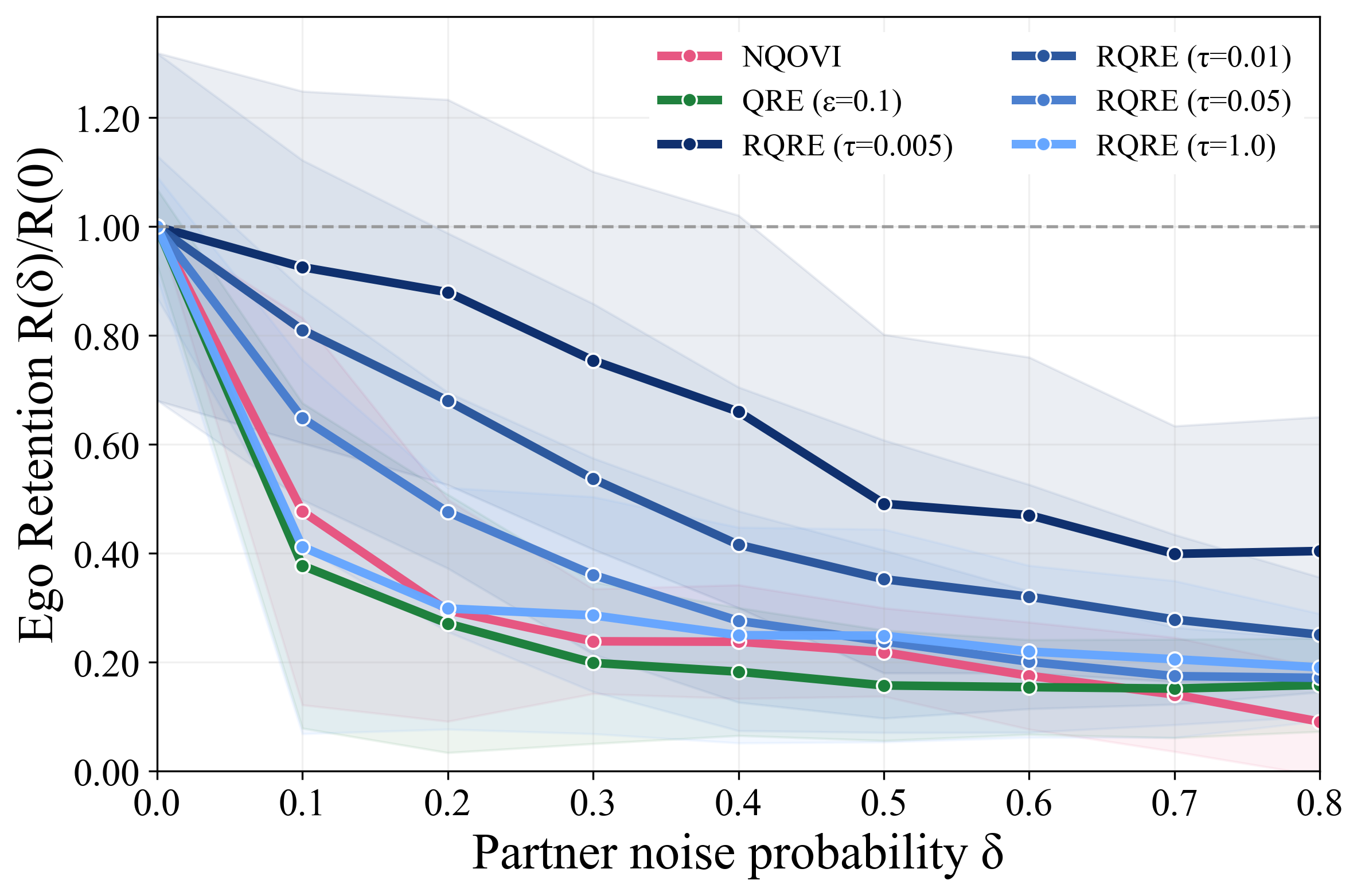}
    \end{subfigure}
    \caption{\textbf{Cross-play 
    retention $(R(\delta)/R(0))$ as a function of perturbed partner noise ($\delta$):} Stag-Hunt (left) and Overcooked (right). At each evaluation step, the partner's action is  a fixed deterministic action (e.g., always move in one direction) with probability $\delta$ and otherwise follows its trained policy. This   produces high-signal deviations to emphasize the robustness phenomena. 
    Curves are normalized by the $\delta=0$ baseline, so higher values indicate strong robustness and lower values indicate performance degradation. Results are averaged over 200 evaluation rollouts per noise level. 
    }
    \label{fig:perturbed_partner_retention}
\end{figure}

\subsection{Environment~2: ``Overcooked" Cooperative Dynamic Game}
The second environment on which we evaluate is Overcooked  \citep{flair2024jaxmarl, gessler2025overcookedv}: two agents cooperate in a small kitchen to prepare and deliver onion soup. The task requires a sequential chain of subtasks---picking onions, loading a pot, retrieving cooked soup with a plate, and delivering it---with a sparse team reward of 20 per delivery and shaped rewards for completing intermediate steps. The tight layout makes coordination essential, as agents must divide labor while avoiding blocking one another. 

\paragraph{\emph{\textbf{RQRE and QRE outperform NQOVI in self-play.}}}
Under self-play (Figure~\ref{fig:team_return}, right), all \texttt{RQRE-OVI} variants and \texttt{QRE-OVI} converge to team returns in medium to high ranges while \texttt{NQ-OVI} converges more slowly and to notably lower returns. We attribute this gap to the equilibrium selection problem inherent to Nash computation. Overcooked requires consistent coordination across many time steps, as agents must implicitly agree on a division of labor and maintain this role assignment throughout the episode. Because Nash equilibria are non-unique in general-sum games, solving for Nash at each stage game introduces the possibility of inconsistent selections across stages and episodes, disrupting the sustained coordination the task demands. \texttt{RQRE-OVI} and \texttt{QRE-OVI} avoid this failure mode entirely because the RQRE ensures a unique equilibrium at every stage, producing consistent and coherent multi-step behavior. Among the RQRE variants, performance is relatively stable across risk levels indicating that moderate risk aversion does not substantially reduce self-play returns. 

\begin{figure}[htbp]
    \centering
    \includegraphics[width=0.9\textwidth]{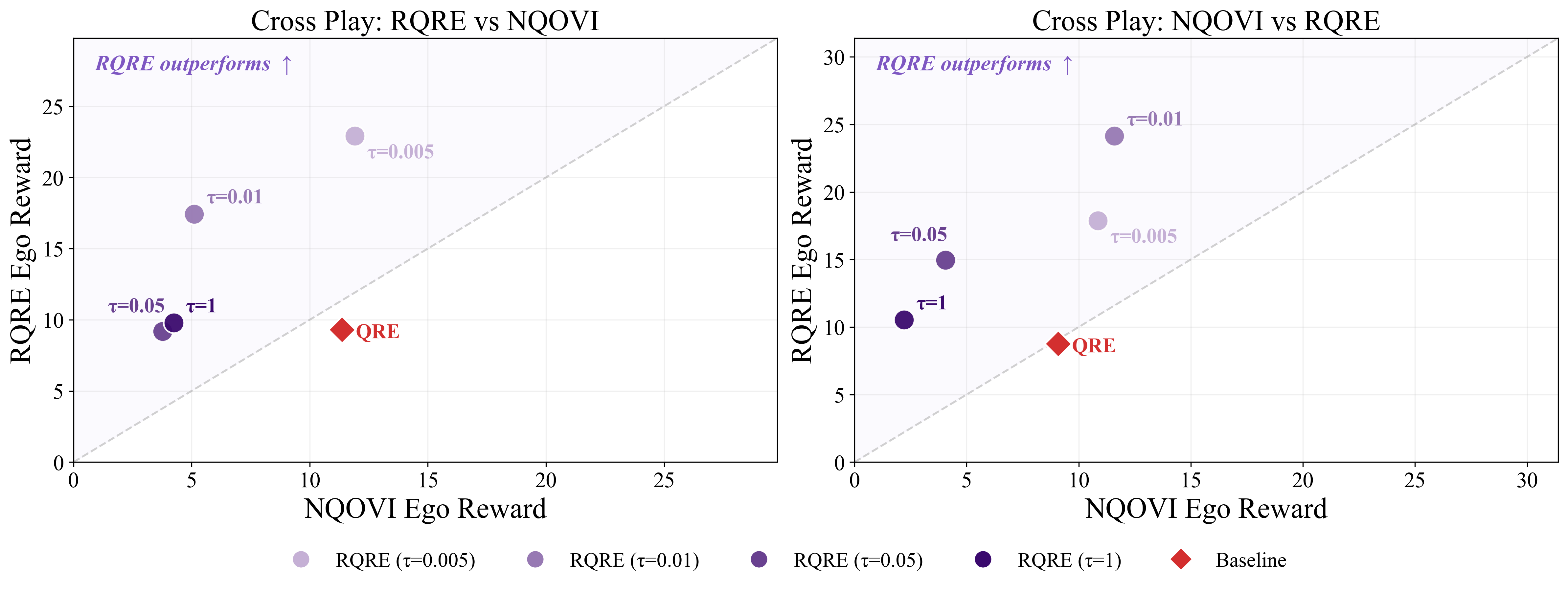}
    \caption{\textbf{Cross-play with unseen partners in Overcooked.} Each point represents the reward of two agents trained under a different algorithm and paired at test time without ever seen each other before. The left panel shows \texttt{RQRE-OVI} agent's reward (vertical axis) versus \texttt{NQ-OVI} agent's reward (horizontal axis) for each pairing; the right panel reverses the roles. Points above the diagonal indicate that the agent on the vertical axis (in this case RQRE) achieves higher return than its partner. Labels denote the $\risk$ value of the corresponding \texttt{RQRE-OVI} agent, and the red diamond marks the \texttt{QRE-OVI} baseline score against \texttt{NQ-OVI}. Across \emph{all}  pairings, \texttt{RQRE-OVI} agents achieve equal or higher ego reward than their cross-play partner, with moderate $\risk$ values (e.g., $\risk=0.01$) yielding the strongest advantage.}
    \label{fig:cross_play}
\end{figure}

\paragraph{\emph{\textbf{RQRE retains performance under partner perturbation.}}}Figure~\ref{fig:perturbed_partner_retention} shows that the moderate risk aversion retains $60-80\%$ of performance at $\delta =0.3$, while NQOVI, QRE, and higher $\risk$ RQRE variants drop below $40\%$. Combined with the cross play with unseen partners results that follow, this suggests that risk aversion produces policies that are fundamentally more adaptive to partner deviation, rather than overfitting to the specific partner encountered during training.

\paragraph{\emph{\textbf{RQRE dominates cross-play with unseen partners.}}} Figure~\ref{fig:cross_play} shows rewards when agents trained under different algorithms are paired at test time without ever seeing each other before. Across nearly all pairings, RQRE agents achieve equal or higher ego reward than their partner, with points consistently above the diagonal. The advantage is most pronounced at moderate risk levels, where the RQRE agent achieves around $2-3\times$ the reward of the \texttt{NQ-OVI} partner. Notably, these results hold across both pairing directions, indicating that the robustness advantage reflects a genuinely more adaptive policy rather than an artifact of which agent is designated as ego. Moreover, the QRE agent performs comparably to \texttt{NQ-OVI} in cross-play, suggesting  bounded rationality alone does not account for  improved robustness and risk aversion is the key driver of  cross-play advantage.

\section{Conclusion}
In this paper, we study the problem of learning robust equilibria in general-sum Markov games with linear function approximation. We provide the first finite-sample regret guarantees for Risk-Sensitive Quantal Response Equilibria that explicitly characterize how the rationality and risk-sensitivity parameters govern sample complexity, establish that RQRE strictly generalizes existing distributionally robust equilibrium concepts, and prove Lipschitz stability of the equilibrium map under payoff perturbation---a property we leverage to obtain policy convergence guarantees under approximation error. The theoretical and empirical results suggest that robust and behaviorally grounded solution concepts offer a promising path for scalable multi-agent reinforcement learning beyond Nash equilibrium. Natural future directions include decentralizing the bonus term, which may be feasible under the monotonicity conditions established for RQRE in the infinite-horizon setting \cite{zhang2025QLearningRiskAverse}, and a more fine-grained understanding of the role of risk sensitivity, including risk-seeking agents and asymmetric risk profiles across players.

\subsubsection*{Acknowledgments}
\label{sec:ack}
LJR is supported in part by NSF award 2312775.
JG is supported by an NSF Graduate Research Fellowship under Grant No. DGE-2140004.

\bibliography{arxiv_refs}
\bibliographystyle{plainnat}
%%%%%%%%%%%%%%%%%%%%%%%%
%% 
%%  APPENDIX 
%%
%%%%%%%%%%%%%%%%%%%%%%%%

\appendix

%%%%%%%%%%%%%%%%%%%%%%%%%%%%%%%%%%%%%%%%%%%%%%%%%%%%%%%%%%
%% Extended Related Work
%%%%%%%%%%%%%%%%%%%%%%%%%%%%%%%%%%%%%%%%%%%%%%%%%%%%%%%%%%

\section{Extended Discussion \& Related Work}
\label{sec:related-work-extended}

In this section we provide a detailed discussion of related literature. Our work studies the Risk Sensitive Quantal Response Equilibrium (RQRE) introduced by \citet{mazumdar2025tractable}, which combines risk aversion with bounded rationality to yield a class of equilibria that is computationally tractable via no-regret learning in \emph{all} finite-action, finite-horizon Markov games, provided agents exhibit a sufficient degree of risk aversion and bounded rationality. \citet{zhang2025QLearningRiskAverse} extended these results to discounted infinite-horizon Markov games, proving contraction of the risk-averse quantal response Bellman operator and convergence of a Q-learning algorithm under monotonicity conditions. 

In this work we build on this foundation by providing the first finite-sample regret guarantees for learning RQRE with linear function approximation, enabling scalability to large or continuous state spaces. Beyond this extension to function approximation, our analysis reveals an explicit dependence of the regret on the rationality and risk-sensitivity parameters, establishes that RQRE strictly generalizes existing distributionally robust equilibrium concepts, and proves Lipschitz stability of the RQRE policy map---a property that Nash equilibria provably lack---which we leverage to obtain policy convergence guarantees under approximation error.

\paragraph{Nash equilibrium methods in Markov games.}
%A natural framework for analyzing multi-agent interactions is game theory \cite{LanctotMARL2017}, where 
The standard solution concept when studying multi-agent interactions from a game-theoretic perspective is the Nash equilibrium \cite{Nash1950}: a joint policy from which no agent can improve their outcome by unilaterally deviating. In a seminal work, \citet{hu2003nash} introduced \emph{Nash Q-learning}, the first extension of single-agent Q-learning to general-sum stochastic games. The core mechanism involves agents maintaining Q-functions over joint actions and performing updates based on the assumption that all participants will adopt Nash equilibrium behavior in the \emph{stage games}. However, this original formulation provides only asymptotic guarantees. Building on this, \citet{liu2021sharp} provided the first finite-sample guarantees by incorporating optimism into the online learning process, an exploration bonus term that allows agents to efficiently learn an accurate model of the environment from a limited number of interactions. Despite these advancements, both works are restricted to the tabular setting, whose computational and memory costs scale prohibitively with the size of the state space. Moreover, computing Nash equilibria in general-sum games is computationally intractable in general \cite{DaskalakisNashComplexity2013}. This intractability has motivated alternative solution concepts such as correlated equilibria \citep{AUMANN197467_CorreleatedEquilibria} and coarse correlated equilibria \citep{Moulin1978_CCE}, but these rely on a shared correlation device and lack individual rationalizability, meaning each agent's learned policy carries no meaningful optimality guarantees unless all agents jointly follow the correlated recommendation. To address the scalability issues of tabular methods, \citet{cisneros2023finite} proposed Nash Q-learning with Optimistic Value Iteration (NQOVI), extending the framework of \citet{hu2003nash} to the linear function approximation regime and providing finite-sample regret bounds that scale with feature dimension rather than the size of the state space. As discussed in Section~\ref{sec:introduction} \& \ref{sec:dro-rqe}, a fundamental bottleneck persists: the algorithm requires solving for a Nash equilibrium at every stage game and every episode, inheriting both computational intractability and the instability of the Nash correspondence under payoff perturbations (Example~\ref{ex:nash_unstable}). In the function approximation setting, where Q-values are necessarily estimated with error, this brittleness is especially problematic.

\paragraph{Bounded Rationality and Risk Aversion.}
Bounded rationality has a long history in behavioral game theory, where it was introduced to better predict the strategic behavior of human decision-makers known to be imperfect optimizers \cite{MCKELVEY19956, Erev1998PredictingHowPeoplePLayGames, Goeree1999PlayingGames}. The standard formalization replaces exact best responses with \emph{quantal responses}---stochastic policies that assign higher probability to higher payoff actions. In the context of multi-agent reinforcement learning, this modeling choice offers several important benefits. First, it provides a principled equilibrium selection since entropy-regularized response functions yield a unique, smooth equilibrium for any finite game, resolving the multiplicity and discontinuity issues inherent to Nash equilibrium \cite{MCKELVEY19956}. Second, bounded rationality reduces sensitivity to the approximation errors and non-stationarity unavoidable in learning settings, producing policies that trade off optimality against stability. Third, this formulation aligns naturally with modern reinforcement learning, where entropy-regularized objectives are widely employed as components of coherent decision-theoretic frameworks \citep{softactorcrtiic-haarnoja18b, sokota2023a, MertikopoulousRLandReg2016}.

Incorporating risk aversion addresses fundamental limitations of purely expected utility maximizing agents. Risk-sensitive objectives provide a principled mechanism for promoting safety and reliability by discouraging policies that achieve high average performance at the cost of rare but catastrophic outcomes, a failure mode that is especially pronounced in strategic, non-stationary environments. In game-theoretic settings, risk aversion acts as a form of robustification: by penalizing outcome variability, agents become less sensitive to modeling error, finite-sample noise, and strategic misspecification of opponents \cite{mazumdar2025tractable, lanzetti2025strategicallyrobustgametheory}. This can further stabilize learning by biasing agents towards equilibria with lower variance and more predictable joint behavior, ultimately improving generalization across opponents and environments.

We refer the reader to \citet{mazumdar2025tractable} who provide a thorough exposition and history of these concepts, and their use in Markov games. 

\paragraph{Risk sensitive multi-agent reinforcement learning.} Our work leverages risk-aversion, where policies prefer actions with more certain outcomes in exchange for potentially lower expected return. Risk sensitivity in decision-making has its origins in Markov decision processes \cite{HowardMathesonRiskMDP1972}, dynamic programming \cite{RuszczynskiDP2010}, and optimal control \cite{HernandezRiskControl1996}, with a subsequent line of work integrating risk sensitivity into reinforcement learning \cite{ShenRisk2014, HEGER1994105, mazumdar2017gradient, ratliff2019inverse, SinghRiskRL2018, LacotteRiskGenerativeImitationLearning2018}. This work focuses on risk in \emph{multi-agent reinforcement learning}. From a theoretical point of view, most existing works differ in the source of risk they address and make varying structural assumptions. \citet{WangNashRisk2024} considers risk-aversion to randomness in the stochastic payoffs via the conditional value at risk  and requires strong monotonicity of the risk-averse game to guarantee convergence to a unique Nash equilibrium. \citet{yekkehkhany2020riskaverseequilibriumgames} studies risk-aversion to the randomness induced by a population of mixed strategies of other agents in stochastic games, proving existence of equilibria in all finite games. \citet{SlumbersGameTheoryRisk2023} considers risk-aversion to opponents' strategies by regularizing each agent's utility with the variance induced by the other agents' actions, and proves existence of their equilibrium concept, but the resulting formulation lacks tractability for general games. \citet{lanzetti2025strategicallyrobustgametheory} takes a complementary approach, proposing a strategically robust equilibrium in which each agent optimizes against the worst-case opponent behavior within an ambiguity set, defined using optimal transport, and shows that such equilibria exist under the same assumptions as Nash equilibria and interpolate between Nash and maximin strategies. While these works focus on theoretical guarantees, recent empirical studies suggest risk-sensitive techniques extend to function approximators such as neural networks \cite{QiuRMIX2021, ShenRiskQ2023}. A related but distinct line of work is robust multi-agent reinforcement learning \cite{ShiRobustMARL2024, ZhangRobustMARL2020, mcmahan2024roping}, which primarily addresses environmental uncertainty and worst-case parameters.

\paragraph{Linear function approximation in RL and Markov games.} Foundational work by \citet{NIPS1996_e0040614} established asymptotic convergence guarantees in reinforcement learning algorithms with linear function approximation. Early results in providing \emph{finite-sample} guarantees with linear function approximation was provided by \citet{munos08a} studying fitted-value iteration under the assumption of access to a generative model of the environment. More recently, \citet{jin2020ProvablyEfficient} provided a provably sample efficient algorithm for linear MDPs in an online setting using optimism. \citet{HeMinimaxOptimalPMLR2021} further refined these results, obtaining nearly minimax-optimal guarantees. Subsequently, linear function approximation has been extensively studied in the single-agent setting \cite{YangW19, wang2021sampleefficient, eaton2025replicablereinforcementlearninglinear, pmlr-v108-zanette20a}. Although less explored, extending these approaches to Markov games has attracted growing attention. The first line of works accomplished this in the two-player zero-sum setting \cite{chen2022optimalalgorithmstwoplayerzerosum, qiu2022rewardfreerlkernelneural}. Most recently, concurrent works \citet{pmlr-v195-wang23b} and \citet{pmlr-v195-cui23a} introduced independent linear function approximation to simultaneously address large state spaces while breaking the curse of multi-agents, achieving sample complexities polynomial in the number of agents rather than exponential. \citet{pmlr-v247-dai24a} further refined these results with improved convergence rates. Most relevant to our work, \citet{cisneros2023finite} extended Nash Q-learning \cite{hu2003nash} to the linear function approximation regime, providing finite-sample regret bounds that scale with the feature dimension rather than the size of the state space. However, their algorithm requires solving for a Nash equilibrium at every stage game, inheriting both the computational intractability of Nash in general-sum games and its brittleness under payoff perturbations (as we show in Example~\ref{ex:nash_unstable}). Moreover, Nash equilibria are not guaranteed to be unique in general-sum games, introducing an equilibrium selection problem that is particularly problematic under function approximation, where Q-values are necessarily estimated with error. In contrast, RQRE yields a unique, smooth equilibrium that is Lipschitz continuous in estimated payoffs (Corollary~\ref{cor:rqe_lipschitz}), and is individually rationalizable in the sense that each agent's learned Q-function carries meaningful optimality guarantees independent of whether other agents follow the equilibrium. Beyond the choice of equilibrium concept, our analysis introduces new technical challenges not present in \citet{cisneros2023finite}: we must bound the equilibrium approximation error incurred by computing an approximate RQRE at each stage in terms of the risk-aversion and bounded rationality parameters, and account for the additional approximation error introduced by the risk operators.

\section{RQRE-OVI Regret Bounds}
\label{app:rqre-ovi-proofs}
In this appendix section, we prove the main regret results (Theorem~\ref{thm:per_rqre_ovi_full}, Corollary~\ref{cor:per_rqre_ovi_entropic}).

\subsection{Proof Roadmap \& Comparison to Prior Art}
\label{sec:proof_roadmap}

This section provides a high-level overview of the proof of
Theorem~\ref{thm:per_rqre_ovi_full} and highlights the key differences
relative to prior analyses of optimistic value iteration in
single-agent reinforcement learning and Nash Q-learning in Markov games
\cite{cisneros2023finite}.

\paragraph{Regret definition.}
Recall that the regret is defined via the episode-wise \emph{exploitability
gap}:
\begin{equation}
\Delta_k
=
\max_{i\in[n]}
\Big(
V_{i,1}(\sBR_i(\pi_{-i}^k),\pi_{-i}^k;s_0)
-
V_{i,1}(\pi^k;s_0)
\Big),
\qquad
\Reg(K)=\sum_{k=1}^K \Delta_k .
\end{equation}

\paragraph{Proof Sketch in Four Steps.}
The proof proceeds through four main steps.

\paragraph{Step 1: Optimistic Value Estimation under Risk Operators.}
The algorithm maintains a linear approximation of the action-value
functions
\begin{equation}
Q_{i,h}^k(x,a)
=
\min\Big\{
\langle w_{i,h}^k,\phi(x,a,h)\rangle
+
\beta
\sqrt{\phi(x,a,h)^\top(\Lambda_h^k)^{-1}\phi(x,a,h)},
\, B
\Big\},
\end{equation}
where the second term is the exploration bonus and $\Lambda_h^k$
is the ridge regression covariance matrix.

The first key technical step establishes an \emph{optimism property}. 
Lemma~\ref{lem:optimism_full} shows that with high probability
\begin{equation}
Q_{i,h}^k(x,a)
\ge
r_{i,h}(x,a)
+
\rho_{i,h}^{\sf env}
\big(
V_{i,h+1}(\pi_{h+1}^k;\cdot)
\big)
-
\Big(
\varepsilon_{\sf env}
+
L_{\sf env}(\varepsilon_{\sf pol}+\varepsilon_{\sf eq})
\Big).
\end{equation}
Compared to the standard optimistic Bellman inequality used in
risk-neutral RL, this bound includes additional approximation terms
arising from
$(i)$ estimation error $\varepsilon_{\sf env}$ in the environment risk operator, 
 $(ii)$
approximation error $\varepsilon_{\sf pol}$ in policy-risk evaluation,
$(iii)$ error in the stage equilibrium solver $\varepsilon_{\sf eq}$.
These errors propagate through the environment risk operator via its
Lipschitz constant $L_{\sf env}$.

\paragraph{Step 2: Uniform Concentration via Covering Numbers.}
The second step controls the statistical complexity of the induced
value-function class.
This is formalized in
Lemma~\ref{lem:covering_induced_value_class}. In single-agent linear MDP analyses, concentration bounds often rely on
self-normalized martingale inequalities that lead to a linear-bandit
style argument. In contrast, the present setting requires controlling a
much richer class of value functions.

Specifically, the regression targets depend on the composition
\begin{equation}
Q
\;\mapsto\;
\tilde\pi_h(\cdot\mid x;Q)
\;\mapsto\;
V_{i,h+1}(\cdot;\tilde\pi)
\;\mapsto\;
\rho_{i,h}^{\esf}(V_{i,h+1}),
\end{equation}
where $\tilde\pi_h$ is the approximate RQRE policy computed from $Q$.

Since the equilibrium policy depends on the current estimate $Q$, the
induced value functions form a data-dependent function class.
Controlling concentration therefore requires a uniform bound over this
class.

To accomplish this, the proof constructs an $\varepsilon$-cover of the
induced value-function class $\mathcal V_{i,h}^k$ and bounds its covering
number $\mc{N}(\cdot)$. 
Indeed, since the equilibrium policy depends on the current estimate $Q$, the
resulting value functions form a data-dependent class.
Lemma~\ref{lem:covering_induced_value_class} bounds the covering number
of this class, yielding
\begin{equation}
\log \mc{N}(\epsilon,\mathcal V_{i,h}^k,\|\cdot\|_\infty)
=
\tilde{\mathcal O}(d^2)
\end{equation}
reflecting the need to simultaneously control $(i)$ the linear predictor $w$,
$(ii)$ the bonus term involving $(\Lambda_h^k)^{-1}$,
and $(iii)$ the equilibrium mapping $Q\mapsto\tilde\pi$. This bound determines the confidence radius $\beta$ used in the
optimistic value estimates.
This larger covering number leads to a larger confidence radius
$\beta$, which ultimately yields a regret bound with higher dependence
on the feature dimension $d$ than in the single-agent linear MDP case.

\paragraph{Step 3: Incorporating Equilibrium Approximation.}
The third ingredient shows how the stage-game equilibrium solver error
enters the regret bound.
Lemma~\ref{lem:use_eps_eq_in_regret} proves that the value accuracy
guarantee of the approximate RQRE solver translates into an additive
error term in the exploitability gap.

Combining this lemma with the optimism property yields a bound on the
per-episode exploitability gap in terms of the cumulative exploration
bonuses along the trajectory plus an additive approximation budget.

\paragraph{Step 4: Performance Difference Decomposition \& Sum of Bonuses.}
Fix episode $k$ and let
\begin{equation}
i^\star
\in
\arg\max_i
\Big(
V_{i,1}(\sBR_i(\pi_{-i}^k),\pi_{-i}^k;s_0)
-
V_{i,1}(\pi^k;s_0)
\Big).
\end{equation}
Define the deviation policy
$
\pi^{k,\br}
=
(\sBR_{i^\star}(\pi_{-i^\star}^k),\pi_{-i^\star}^k)$.
For each stage $h$, define
$
\Delta_{k,h}(x)
=
V_{i^\star,h}(\pi^{k,\br};x)
-
V_{i^\star,h}(\pi^k;x)$.
A stage-wise recursion shows
\begin{equation}
\Delta_{k,h}(x)
\le
\text{(advantage term)}
+
\mathbb E[\Delta_{k,h+1}(x_{h+1})].
\end{equation}

Using the optimism property and the approximate equilibrium guarantee,
the advantage term can be bounded by
\begin{equation}
\beta
\sqrt{\phi(x_h^k,a_h^k,h)^\top
(\Lambda_h^k)^{-1}
\phi(x_h^k,a_h^k,h)}
+
\varepsilon_{\sf eq}.
\end{equation}
This yields a bound on the \emph{exploitability gap} in terms of the cumulative
bonuses along the trajectory.

Next, we bound the sum of bonuses. Indeed, summing the recursion over $h=1,\dots,H$ yields
\begin{equation}
\Delta_k
\lesssim
\sum_{h=1}^H
\beta
\sqrt{\phi(x_h^k,a_h^k,h)^\top
(\Lambda_h^k)^{-1}
\phi(x_h^k,a_h^k,h)}
+
H\big(
\varepsilon_{\sf env}
+
L_{\sf env}(\varepsilon_{\sf pol}+\varepsilon_{\sf eq})
\big).
\end{equation}
Summing over episodes and applying the elliptical potential lemma gives
\begin{equation}
\sum_{k,h}
\sqrt{\phi(x_h^k,a_h^k,h)^\top
(\Lambda_h^k)^{-1}
\phi(x_h^k,a_h^k,h)}
=
\tilde{\mathcal O}(\sqrt{K d H}).
\end{equation}

\paragraph{Final Bound Construction.} 
The remaining step of the proof bounds the cumulative exploration
bonuses along the trajectory.
This is done using the standard elliptical potential argument
commonly used in analyses of optimistic value iteration
\cite{jin2020ProvablyEfficient,cisneros2023finite}.
Indeed,
combining this with the covering-number bound on $\beta$ yields the
final regret bound
\begin{equation}
\Reg(K)
\le
\tilde{\mathcal O}
\big(
L_{\sf env} B
\sqrt{K d^3 H^3}
\big)
+
K H
\big(
\varepsilon_{\sf env}
+
L_{\sf env}(\varepsilon_{\sf pol}+\varepsilon_{\sf eq})
\big).
\end{equation}

\paragraph{Comparison to Nash Q-Learning.}
The closest related analysis is that of
\citet{cisneros2023finite}, which studies optimistic value iteration
for risk-neutral Markov games with linear function approximation.

Three key differences distinguish the present analysis:
\begin{enumerate}
\item
\textbf{Explicit equilibrium approximation.}
Their algorithm assumes access to a Nash equilibrium oracle for each
stage game. In contrast, our algorithm computes an approximate
RQRE which are known to be efficiently computable via no-regret algorithms, introducing an additional
error term $\varepsilon_{\sf eq}$ that must be propagated through the
dynamic programming recursion.

\item
\textbf{Risk-sensitive Bellman operators.}
The Bellman operators  involve nonlinear risk operators acting on the
continuation values, as opposed to risk-neutral expectations. As a result, the analysis must control the
propagation of estimation and equilibrium errors through these
operators, leading to the Lipschitz amplification factor
$L_{\sf env}$ and additional approximation terms in the regret bound.

\item \textbf{Data-Dependent Risk-Induced Value Function Class.} Since the equilibrium policy is computed
from the current value estimates, the continuation values depend on the
data-dependent mapping $Q \mapsto \tilde{\pi}(Q)$, which then are transformed via the risk measure $\rho^{\esf}$. This requires constructing the appropriate vale function class that is induced by the risk operators. Controlling the resulting function class requires a covering-number
argument for the induced value functions.
\end{enumerate}
Together these features lead to a regret analysis that differs
 from the Nash-Q case (and single agent case) while yielding guarantees for an
algorithm that is computationally realizable in practice.

\paragraph{Comparison to Single-Agent Linear MDP Analyses.}
In single-agent linear MDP analyses, the Bellman target depends only on
a linear expectation of the next-stage value function. This structure
enables sharper self-normalized concentration arguments and typically
leads to regret bounds with smaller dimension dependence.

In contrast, the regression targets in the present setting involve the composition
\begin{equation}
Q
\;\mapsto\;
\tilde\pi_h(\cdot\mid x;Q)
\;\mapsto\;
V_{i,h+1}(\cdot;\tilde\pi)
\;\mapsto\;
\rho_{i,h}^{\sf env}(V_{i,h+1}),
\end{equation}
where $\tilde\pi_h$ is the approximate RQRE policy computed from $Q$.
Controlling concentration uniformly over this induced value class
requires a covering-number argument, which leads to a larger confidence
radius and the higher $d$-dependence appearing in the regret bound.
\subsection{Technical Lemmas}
\label{app:technical_lemmas}
In this section, we have all the technical lemmas for the main regret results (Theorem~\ref{thm:per_rqre_ovi_full}). 

% =========================
% Lemma 1: covering number
% =========================
\paragraph{Covering Number Bounds for the Value Functions.} The first technical lemma we prove is the covering number bound that accounts for both the risk measures and the data-dependent risk-induced value functions. 
\begin{lemma}[Covering number of the induced value class]
\label{lem:covering_induced_value_class}
Fix $(i,h)$ and an episode index $k$.
Let $\Lambda_h^k\succeq \lambda I_d$ be the (random but fixed conditional on $\mc F_{k,h-1}$)
design matrix used by Algorithm~\ref{alg:pr-er-rqre-ovi} at stage $h$ in episode $k$.
Define the optimistic $Q$-class at stage $h$ as follows:
\[
    \mc Q_{i,h}^k
:=
\Big\{
(x,a)\mapsto
\min\left\{
\ip{w}{\phi(x,a,h)}
+
\beta\sqrt{\phi(x,a,h)^\top(\Lambda_h^k)^{-1}\phi(x,a,h)}
,\;B\right\}:\;\norm{w}_2\leq B
\Big\}.
\]
Additionally, define the induced policy-evaluated value class as follows:
\[
\mc V_{i,h}^k
:=
\left\{
x\mapsto \Vhat_{i,h}^{\epsilon_i}(\pi;x;Q)
:\; Q\in \mc Q_{i,h}^k,\ \pi(\cdot|x)\in\Delta(\mc A)\right\}.
\]
Then for any $\varepsilon\in(0,B]$, under the metric
$\dist_\infty(V,V'):=\sup_x |V(x)-V'(x)|$,
\begin{equation}
\label{eq:covering_bound_explicit}
\log \mc N \left(\varepsilon,\mc V_{i,h}^k,\dist_\infty\right)
\le
d\log \left(1+\frac{4B}{\varepsilon}\right)
+
d^2\log \left(1+\frac{8\sqrt d\,\beta^2}{\lambda\,\varepsilon^2}\right).
\end{equation}
\end{lemma}

\begin{proof}
We proceed in four steps.

\paragraph{Step 1: Lipschitz composition reduces the problem to covering $\mc Q_{i,h}^k$.}
Fix any $Q,Q'\in \mc Q_{i,h}^k$ and any Markov $\pi$.
First observe that for each fixed $(i,h)\in [n]\times [H]$ and any Markov $\pi$,
the map $Q\mapsto \Vhat_{i,h}^{\epsilon_i}(\pi;\cdot;Q)$ is 1-Lipschitz in sup-norm:
\[
\sup_{x\in\mc X}\left|\Vhat_{i,h}^{\epsilon_i}(\pi;x;Q)-\Vhat_{i,h}^{\epsilon_i}(\pi;x;Q')\right|
\le \norm{Q-Q'}_\infty.
\]
For the dual-representation-based $\Vhat$, this holds because it is a supremum of
expectations of $Q$ over a set of distributions plus an additive regularizer.
By this 1-Lipschitz property, we have the bound
\[
\dist_\infty \left(\Vhat_{i,h}^{\epsilon_i,\tau_i}(\pi;\cdot;Q),\ \Vhat_{i,h}^{\epsilon_i,\tau_i}(\pi;\cdot;Q')\right)
\le \norm{Q-Q'}_\infty.
\]
Therefore, any $\varepsilon$-cover of $\mc Q_{i,h}^k$ in $\|\cdot\|_\infty$
pushes forward to an $\varepsilon$-cover of $\mc V_{i,h}^k$ in $\dist_\infty(\cdot,\cdot)$, and hence
\[
\mc N(\varepsilon,\mc V_{i,h}^k,\dist_\infty(\cdot,\cdot))
\le
\mc N(\varepsilon,\mc Q_{i,h}^k,\|\cdot\|_\infty).
\]
So it suffices to bound the covering number of $\mc Q_{i,h}^k$.

\paragraph{Step 2: Separate the parametric linear part and the bonus part.}
For each $w$ define the unclipped function
\[
f_w(x,a)
:=
\ip{w}{\phi(x,a,h)}
+
b_h^k(x,a),
\qquad
b_h^k(x,a):=
\beta\sqrt{\phi(x,a,h)^\top(\Lambda_h^k)^{-1}\phi(x,a,h)}.
\]
Let $T(u):=\min\{u,B\}$ be the clipping operator.
Then $Q_w:=T\circ f_w \in \mc Q_{i,h}^k$.
Since $T$ is 1-Lipschitz on $\R$, we have that 
\[
\norm{Q_w-Q_{w'}}_\infty
=
\norm{T\circ f_w - T\circ f_{w'}}_\infty
\le
\norm{f_w-f_{w'}}_\infty.
\]
Also, since $b_h^k$ does not depend on $w$, we deduce that 
\[
\norm{f_w-f_{w'}}_\infty
=
\sup_{x,a} \left|\ip{w-w'}{\phi(x,a,h)}\right|
\le
\norm{w-w'}_2\cdot \sup_{x,a}\norm{\phi(x,a,h)}_2
\le
\norm{w-w'}_2,
\]
using $\norm{\phi}\le 1$.
Thus, if we cover the Euclidean ball $\{w:\norm{w}_2\le B\}$ in $\ell_2$
to radius $\varepsilon$, we automatically get an $\varepsilon$-cover of
$\mc Q_{i,h}^k$ in $\|\cdot\|_\infty$ \emph{provided the bonus term is fixed}.

However, because the proof must hold uniformly over all \emph{possible} random matrices
$\Lambda_h^k$ that can arise adaptively along the algorithm’s trajectory,
we incorporate a standard operator-covering argument: 
we bound the number of distinguishable bonus shapes
$\phi^\top(\Lambda_h^k)^{-1}\phi$ over all $\Lambda$ in the admissible set.

\paragraph{Step 3: Cover the linear predictors (the $d\log(1+4B/\varepsilon)$ term).}
Let $\mc W:=\{w\in\R^d:\norm{w}_2\le B\}$.
A standard volumetric bound for Euclidean balls gives an $\varepsilon$-net
$\mc W_\varepsilon\subset \mc W$ of size at most
\[
|\mc W_\varepsilon|
\le
\left(1+\frac{2B}{\varepsilon}\right)^d
\le
\left(1+\frac{4B}{\varepsilon}\right)^d,
\]
since $\varepsilon\le B$
For any $w\in\mc W$, pick $\tilde w\in \mc W_\varepsilon$ with $\norm{w-\tilde w}_2\le \varepsilon$.
Then for all $(x,a)\in \mc{X}\times\mc{A}$, we have that 
\[
\left|\ip{w}{\phi}-\ip{\tilde w}{\phi}\right|
=
\left|\ip{w-\tilde w}{\phi}\right|
\le
\norm{w-\tilde w}_2\norm{\phi}_2
\le \varepsilon.
\]
Thus, if the bonus term were fixed, this would already provide an $\varepsilon$-cover.

\paragraph{Step 4: Cover the bonus shapes over admissible $\Lambda$.} % (the $d^2\log(\cdot)$ term).}
Define the normalized feature vectors
\[
u(x,a):=\phi(x,a,h)\in\R^d,\qquad\text{such that }\quad \norm{u(x,a)}_2\le 1.
\]
For any positive semi-definite matrix $M\succeq 0$, define the quadratic form
\[
q_M(u):=u^\top M u.
\]
Here $M=(\Lambda_h^k)^{-1}$ and $\Lambda_h^k\succeq \lambda I_d$ implies
\[
0\preceq M\preceq \lambda^{-1}I_d,\qquad \text{and}\qquad \norm{M}_2\le \lambda^{-1}.
\]
We need a finite set $\mc M_\eta$ such that for every admissible $M$
there exists $\tilde M\in\mc M_\eta$ with
\[
\sup_{\norm{u}\le 1}\left|u^\top M u - u^\top \tilde M u\right|
\le \eta.
\]
Note that for symmetric $M-\tilde M$, the bound holds:
\[
\sup_{\norm{u}\le 1} |u^\top(M-\tilde M)u|
=
\norm{M-\tilde M}_2.
\]
Thus it suffices to build an $\eta$-net in operator norm for the set
$\{M: 0\preceq M\preceq \lambda^{-1}I\}$.
A standard covering bound for bounded subsets of $\R^{d\times d}$ gives an $\eta$-net
in Frobenius norm of size at most
\[
\left(1+\frac{2R}{\eta}\right)^{d^2},
\quad
R:=\sup\norm{M}_F \le \sqrt d\,\lambda^{-1}.
\]
Since $\norm{A}_2\le \norm{A}_F$, an $\eta$-net in Frobenius norm is also an $\eta$-net in operator norm.
Therefore there exists $\mc M_\eta$ with
\[
|\mc M_\eta|
\le
\left(1+\frac{2\sqrt d\,\lambda^{-1}}{\eta}\right)^{d^2}.
\]
Next we connect $\eta$ to the bonus scale. First, the bonus uses $\sqrt{u^\top M u}$.
For any nonnegative scalars $s,t$, we have that 
\[
\left|\sqrt{s}-\sqrt{t}\right|
\le
\sqrt{|s-t|}.
\]
Hence, if $\sup_{\norm{u}\le1}|u^\top M u-u^\top\tilde M u|\le \eta$, then
\[
\sup_{\norm{u}\le1}\left|\sqrt{u^\top M u}-\sqrt{u^\top\tilde M u}\right|
\le \sqrt{\eta}.
\]
Multiplying by $\beta$ gives a bonus error at most $\beta\sqrt{\eta}$.

To make this bonus error at most $\varepsilon$, choose $\eta=(\varepsilon/\beta)^2$.
Then
\[
\log |\mc M_\eta|
\le
d^2 \log \left(1+\frac{2\sqrt d\,\lambda^{-1}}{(\varepsilon/\beta)^2}\right)
=
d^2 \log \left(1+\frac{2\sqrt d\,\beta^2}{\lambda\,\varepsilon^2}\right).
\]
Adjusting constants---replacing $2$ by $8$---to absorb the earlier $\varepsilon$-splitting
and the clipping/Lipschitz steps gives
\[
\log |\mc M_\eta|
\le
d^2 \log \left(1+\frac{8\sqrt d\,\beta^2}{\lambda\,\varepsilon^2}\right).
\]

\paragraph{Combine nets.}
Take the Cartesian product net
$\mc W_\varepsilon\times \mc M_{(\varepsilon/\beta)^2}$.
For any admissible $(w,M)$, choose $(\tilde w,\tilde M)$ from the nets.
Then the linear part incurs at most $\varepsilon$ error in $\|\cdot\|_\infty$,
and the bonus part incurs at most $\varepsilon$ error in $\|\cdot\|_\infty$.
By triangle inequality, the unclipped $f$ differs by at most $2\varepsilon$,
and clipping is 1-Lipschitz so the clipped $Q$ differs by at most $2\varepsilon$.
Replacing $\varepsilon$ by $\varepsilon/2$ yields the stated bound \eqref{eq:covering_bound_explicit}.
\end{proof}
% ==========================================
% Lemma 2: optimism from bonus + approximations
% ==========================================
\paragraph{Optimism via bonus term and approximations.} Next we construct a technical lemma on bounding the bonus term and approximation errors that arise from the risk measures and equilibrium computation. 

Fix $(i,h)\in [n]\times [H]$.
Let $(x_h^t,a_h^t)$ be the sequence of state-action pairs observed at stage $h$
up to the current episode index $k-1$ (across episodes), and define
\[
\Lambda_h^k
:=
\lambda I_d + \sum_{t<k}\phi(x_h^t,a_h^t,h)\phi(x_h^t,a_h^t,h)^\top.
\]
Define the empirical targets as in Algorithm~\ref{alg:pr-er-rqre-ovi}---namely, 
\[
y_{i,h}^t
:=
r_{i,h}(x_h^t,a_h^t)
+
\rhohat_{i,h}^{\,t} \left(
\Vhat_{i,h+1}^{\epsilon_i,\tau_i}(\pi^{t}_{h+1};\cdot)
\right).
\]
Let $w_{i,h}^k$ be the ridge regression estimate
\[
w_{i,h}^k
:=
(\Lambda_h^k)^{-1}
\sum_{t<k}
\phi(x_h^t,a_h^t,h)\,y_{i,h}^t.
\]
Let the optimistic estimate be
\[
Q_{i,h}^k(x,a)
:=
\min\left\{
\ip{w_{i,h}^k}{\phi(x,a,h)}
+
\beta\sqrt{\phi(x,a,h)^\top(\Lambda_h^k)^{-1}\phi(x,a,h)}
,\;B\right\}.
\]
\begin{lemma}[Uniform self-normalized confidence and optimism]
\label{lem:optimism_full}
For $\beta=cB\sqrt{d^2H\,\iota}$ as in Theorem~\ref{thm:per_rqre_ovi_full},
with probability at least $1-\delta$ simultaneously for all $(i,h,k,x,a)\in [n]\times[H]\times[K]\times\mc{X}\times \mc{A}$, the lower bound holds:
\begin{equation}
\label{eq:optimism_statement}
Q_{i,h}^k(x,a)
\geq
r_{i,h}(x,a)
+
\rho_{i,h}^{\sf env} \left(
V_{i,h+1}^{\epsilon_i}(\pi^{k}_{h+1};\cdot)
\right)
-
\Big(
\varepsilon_{\sf env}
+
L_{\sf env}(\varepsilon_{\sf pol}+\varepsilon_{\sf eq})
\Big).
\end{equation}
\end{lemma}

\begin{proof}
The proof has three parts: (1) express the true Bellman target in the linear form,
(2) establish a uniform self-normalized confidence bound using Lemma~\ref{lem:covering_induced_value_class},
(3) translate the approximation errors and risk-Lipschitz errors into the additive term.

Throughout, fix $(i,h)\in [n]\times [H]$ and condition on the filtration up to the start of stage $h$
in episode $k$ so that $\Lambda_h^k$ is fixed.

\paragraph{Part 1: linear Bellman target under realizability.}
By Assumption (a), for any bounded measurable function $g:\mc X\to\R$,
\[
\E_{x'\sim \mc P_h(\cdot|x,a)}[g(x')]
=
\int g(x')\mc P_h(dx'|x,a)
%=
%\int g(x')\ip{\phi(x,a,h)}{\mu_h(x')}dx'
=
\ip{\phi(x,a,h)}{\; \int g(x')\mu_h(x')dx'}.
\]
Define the vector
\[
\zeta_h(g):=\int g(x')\mu_h(x')\,dx' \in \R^d.
\]
Then the \emph{expected} one-step (non-risk) backup is linear in $\phi$:
\[
r_{i,h}(x,a)+\E[g(x')|x,a]
=
\ip{\phi(x,a,h)}{\theta_{i,h}+\zeta_h(g)}.
\]

Now the environment risk operator $\rho_{i,h}^{\esf}$ is applied to the random next-state value
$g(x_{h+1})$ (or an equivalent lifted representation).
We do not need $\rho_{i,h}^{\esf}$ itself to be linear; we only need concentration for the scalar
targets and Lipschitzness to control errors, which we do below.

\paragraph{Part 2: self-normalized confidence uniformly over the induced value class.}
Define the (ideal) target at time $t$:
\[
\bar y_{i,h}^t
:=
r_{i,h}(x_h^t,a_h^t)
+
\rho_{i,h}^{\esf} \left(
V_{i,h+1}^{\epsilon_i,\tau_i}(\pi^{t}_{h+1};\cdot)
\right).
\]
Define the target noise:
\[
\eta_t:= y_{i,h}^t-\bar y_{i,h}^t.
\]
We will show $\eta_t$ is uniformly bounded and forms a martingale difference sequence,
then apply the self-normalized inequality with a union bound over a cover.

First, by assumption, the rewards lie in $[0,1]$ and all value functions are in $[0,B]$. The  monotonicity plus Lipschitz continuity assumption imply $\rho_{i,h}^{\esf}(X)$ is bounded
whenever $X$ is bounded; in particular, there is a constant---absorbed into $B$ by clipping---such that $|\rho_{i,h}^{\esf}(V)|\le B$ when $0\le V\le B$.
Similarly, by assumption, we have that  $\rhohat_{i,h}^{\,t}(V)\le \rho_{i,h}^{\esf}(V)$.
Thus both $\bar y_{i,h}^t$ and $y_{i,h}^t$ lie in an interval of length $\mc{O}(B)$, hence
$
|\eta_t|\le B$ (up to an absolute constant factor absorbed into $c$).

Second, $\eta_t$ is measurable with respect to the randomness at episode/stage $t$
and satisfies the standard conditional mean-zero property needed for the self-normalized bound
once we condition on the sigma-field generated by all past data:
the only randomness in $\eta_t$ is through the empirical dual approximation and policy approximation,
both of which are treated as bounded perturbations; hence we can apply a bounded-differences
martingale concentration.
Concretely, we use the standard self-normalized bound stated as:
for any fixed $\delta'\in(0,1)$, with probability at least $1-\delta'$,
\begin{equation}
\label{eq:self_norm}
\forall\; x,a:\quad
\left|
\ip{w_{h}^{i,k}-w_{h}^{i,\star}}{\phi(x,a,h)}
\right|
\le
\beta
\sqrt{\phi(x,a,h)^\top(\Lambda_h^k)^{-1}\phi(x,a,h)},
\end{equation}
where $w_{h}^{i,\star}$ is the (unknown) linear parameter corresponding to the ideal targets
$\bar y_{i,h}^t$ and $\beta$ scales as $B\sqrt{d\log(\cdot)}$.

To make \eqref{eq:self_norm} hold \emph{uniformly} over all time indices and the induced
value class (because the targets depend on $\Vhat$ which depends on past data),
we apply a union bound over an $\varepsilon$-net of the induced value class.
Lemma~\ref{lem:covering_induced_value_class} gives the required finite cover size, which yields
the logarithmic term
\[
\iota
=
\log \left(\frac{nHK}{\delta}\right)
+
d\log \left(1+\frac{K}{\lambda}\right)
\]
appearing in the theorem statement.
Choosing $\beta=cB\sqrt{d^2H\,\iota}$ (absorbing all absolute constants)
ensures \eqref{eq:self_norm} holds simultaneously for all $(i,h,k,x,a)\in[n]\times[H]\times[K]\times \mc{X}\times \mc{A}$ with probability at least $1-\delta$.

\paragraph{Part 3: translate approximation errors and derive optimism.}
By construction, we have that $Q_{h}^{i,k}(x,a)
\ge
\la w_{h}^{i,k},\phi(x,a,h)\ra$ and
\[
Q_{h}^{i,k}(x,a)
\ge
\la w_{h}^{i,k},\phi(x,a,h)\ra
+
\beta\sqrt{\phi^\top(\Lambda_h^k)^{-1}\phi}
-
(U(x,a)-B)_+,
\]
where $U(x,a):=\la w_{h}^{i,k},\phi(x,a,h)\ra+\beta \sqrt{\phi^\top (\Lambda_h^k)^{-1}\phi}$ and $(z)_+:=\min\{z,0\}$ as usual.
Since clipping only decreases values above $B$ and all true quantities are within $[0,B]$
by assumption, clipping does not harm the lower bound we are establishing.

Using \eqref{eq:self_norm}, we have that 
\[
\la w_{h}^{i,k},\phi(x,a,h)\ra
+
\beta\sqrt{\phi^\top(\Lambda_h^k)^{-1}\phi}
\ge
\la w_{h}^{i,\star},\phi(x,a,h)\ra.
\]
But $\la w_{h}^{i,\star},\phi(x,a,h)\ra$ is the linear prediction of the \emph{ideal target}
$\bar y$ at $(x,a)$, i.e.
\[
\la w_{h}^{i,\star},\phi(x,a,h)\ra
=
r_{i,h}(x,a)
+
\rho_{i,h}^{\esf} \left(
V_{i,h+1}^{\epsilon_i,\tau_i}(\pi^{k}_{h+1};\cdot)
\right),
\]
by definition of $w_{h}^{i,\star}$ for the ideal target.
It remains to incorporate the approximation errors (d,e,f) to relate the algorithm’s
\emph{implemented} target to the ideal target.

The assumption on the risk measure gives us that for any bounded $X$, the bound holds:
\[
0\le \rho_{i,h}^{\esf }(X)-\rhohat_{i,h}^{\,k}(X)\le \varepsilon_{\esf}.
\]
Hence, we deduce that
\[
\rho_{i,h}^{\esf }(X)\ge \rhohat_{i,h}^{\,k}(X).
\]
The bounds on the value function imply that 
\[
0\le V_{i,h+1}^{\epsilon_i}(\pi;x)-\Vhat_{i,h+1}^{\epsilon_i}(\pi;x)\le \varepsilon_{\sf pol},
\]
so by $L_{\sf env}$-Lipschitzness of $\rho_{i,h}^{\esf }$, we have the bound
\[
\left|\rho_{i,h}^{\esf}(V_{i,h+1}^{\epsilon_i}(\pi;\cdot))
-
\rho_{i,h}^{\esf }(\Vhat_{i,h+1}^{\epsilon_i}(\pi;\cdot))\right|
\le L_{\sf env}\varepsilon_{\sf pol}.
\]
Finally, by assumption we have that the stage equilibrium computation error changes
the policy-evaluated value by at most $\varepsilon_{\sf eq}$; again, when this error is
propagated through the environment operator, it contributes at most $L_{\sf env}\varepsilon_{\sf eq}$.

Collecting these additive losses yields exactly \eqref{eq:optimism_statement}.
\end{proof}

\paragraph{Technical Lemmas on Equilibrium Approximation.}

Fix an episode $k$, stage $h$, and the realized state $x:=x_h^k$.
Let $Q_h^k$ denote the optimistic stage $Q$-table at $(h,x)$ for all players.
Let $\tilde\pi_h(\cdot\mid x)$ be the policy returned by the stage solver, and
let $\pi_h^\star(\cdot\mid x)$ denote the exact stage equilibrium policy for the same
stage game induced by $Q_h^k$.

\begin{lemma}[Stage-solver accuracy]
\label{lem:use_eps_eq_in_regret}
Assume the stage solver satisfies the \emph{value-accuracy} condition from
Theorem~\ref{thm:per_rqre_ovi_full}: for every player $i\in[n]$, the estimate holds:
\begin{equation}
\label{eq:eps_eq_value_accuracy}
\Big|
V_{i,h}^{\epsilon_i}\big(\tilde\pi_h; x, Q_h^k\big)
-
V_{i,h}^{\epsilon_i}\big(\pi_h^\star; x, Q_h^k\big)
\Big|
\;\le\;
\varepsilon_{\sf eq}.
\end{equation}
Further assume the policy-risk value functional $V_{i,h}^{\epsilon_i}(\cdot;x,Q)$ is
\emph{1-Lipschitz in the payoff table} in the following sense: for any two payoff tables
$Q,Q'$ on $(h,x)$ and any fixed policy profile $\pi$,
\begin{equation}
\label{eq:V_lipschitz_in_Q}
\Big|
V_{i,h}^{\epsilon_i}\big(\pi; x, Q\big)-
V_{i,h}^{\epsilon_i}\big(\pi; x, Q'\big)
\Big|\leq 
\sup_{a\in\mc A}\, \big|Q_i(x,a)-Q'_i(x,a)\big|.
\end{equation}
Let $a_h^{k}\sim \tilde\pi_h(\cdot\mid x)$ be the action actually played by the algorithm at $(h,x)$,
and let $a_h^{k,\br}$ denote a (possibly randomized) stage-wise best-response action for the
distinguished player $i^\star$ against $\tilde\pi_{-i^\star,h}(\cdot\mid x)$ with respect to the
\emph{true} stage payoff table $\bar Q_h^k$ (defined below). Suppose the optimistic table admits
the usual decomposition:
\begin{equation}
\label{eq:Q_decomp_bonus}
Q_{i,h}^k(x,a)
=
\bar Q_{i,h}^k(x,a)
+
b_h^k(x,a),
\qquad
b_h^k(x,a)
:=
\beta\sqrt{\phi(x,a,h)^\top(\Lambda_h^k)^{-1}\phi(x,a,h)},
\end{equation}
with $b_h^k(x,a)\ge 0$ for all $a$.
Then the following bound holds:
\begin{equation}
\label{eq:advantage_bound_with_eps_eq}
\E \left[
Q_{i^\star,h}^k(x,a_h^{k,\br})
-
Q_{i^\star,h}^k(x,a_h^{k})
\right]
\;\le\;
\E \left[
b_h^k(x,a_h^{k})
\right]
+
\varepsilon_{\sf eq},
\end{equation}
where the expectation is over the randomness in the stage solver (if any) and the action sampling.
\end{lemma}

\begin{proof}
The proof proceeds in two parts: \textbf{(part 1)} we first show the best-response improvement is controlled by the stage equilibrium error, and then \textbf{(part 2)} we relate the deviation gap to an optimistic $Q$-advantage plus a bonus term.

\paragraph{Part 1: Best-response improvement is controlled by the stage equilibrium error.}
By definition of $\pi_h^\star(\cdot\mid x)$ as an \emph{exact} stage equilibrium for the game induced by
$Q_h^k$ and the value functional $V_{i,h}^{\epsilon_i}(\cdot;x,Q_h^k)$, player $i^\star$ cannot improve
its (policy-risk) value by a unilateral deviation at $(h,x)$:
\begin{equation}
\label{eq:eq_best_response_property}
V_{i^\star,h}^{\epsilon_{i^\star}} \big((\sigma_{i^\star},\pi_{-i^\star,h}^\star);x,Q_h^k\big)
-
V_{i^\star,h}^{\epsilon_{i^\star}} \big(\pi_h^\star;x,Q_h^k\big)
\;\le\;0,
\qquad \forall \sigma_{i^\star}\in\Delta(\mc A_{i^\star}).
\end{equation}
Now evaluate the deviation $\sigma_{i^\star}$ that achieves the best response against
$\tilde\pi_{-i^\star,h}(\cdot\mid x)$---i.e., the deviation that generates $a_h^{k,\br}$.
Add and subtract $V_{i^\star,h}^{\epsilon_{i^\star}}(\tilde\pi_h;x,Q_h^k)$ and
$V_{i^\star,h}^{\epsilon_{i^\star}}(\pi_h^\star;x,Q_h^k)$ to get that
\begin{align}
&V_{i^\star,h}^{\epsilon_{i^\star}} \big((\sigma_{i^\star},\tilde\pi_{-i^\star,h});x,Q_h^k\big)
-
V_{i^\star,h}^{\epsilon_{i^\star}} \big(\tilde\pi_h;x,Q_h^k\big)
\nonumber\\
\qquad\qquad&=
\big(V_{i^\star,h}^{\epsilon_{i^\star}} \big((\sigma_{i^\star},\tilde\pi_{-i^\star,h});x,Q_h^k\big)
-
V_{i^\star,h}^{\epsilon_{i^\star}} \big((\sigma_{i^\star},\pi_{-i^\star,h}^\star);x,Q_h^k\big)\big)\label{eq:term_one}\\
&\quad +\big(
V_{i^\star,h}^{\epsilon_{i^\star}}\big((\sigma_{i^\star},\pi_{-i^\star,h}^\star);x,Q_h^k\big)
-
V_{i^\star,h}^{\epsilon_{i^\star}}\big(\pi_h^\star;x,Q_h^k\big)\big)\label{eq:middle_bound}
\\
&\qquad
+\;
\Big(
V_{i^\star,h}^{\epsilon_{i^\star}} \big(\pi_h^\star;x,Q_h^k\big)
-
V_{i^\star,h}^{\epsilon_{i^\star}} \big(\tilde\pi_h;x,Q_h^k\big)
\Big).
\label{eq:dev_gap_decomp}
\end{align}
Observe that \eqref{eq:middle_bound} is less than or equal to zero by \eqref{eq:eq_best_response_property}, and \eqref{eq:dev_gap_decomp} is less than or equal to $\varepsilon_{\sf eq}$ by \eqref{eq:eps_eq_value_accuracy}.
Thus, we deduce that 
\begin{equation}\label{eq:dev_gap_upper}
    \begin{aligned}
      & V_{i^\star,h}^{\epsilon_{i^\star}} \big((\sigma_{i^\star},\tilde\pi_{-i^\star,h});x,Q_h^k\big)
-
V_{i^\star,h}^{\epsilon_{i^\star}} \big(\tilde\pi_h;x,Q_h^k\big)\\
&\qquad
 \leq 
V_{i^\star,h}^{\epsilon_{i^\star}} \big((\sigma_{i^\star},\tilde\pi_{-i^\star,h});x,Q_h^k\big)-
V_{i^\star,h}^{\epsilon_{i^\star}} \big((\sigma_{i^\star},\pi_{-i^\star,h}^\star);x,Q_h^k\big)
+ \varepsilon_{\sf eq}.
    \end{aligned}
\end{equation}

\paragraph{Step 2: Relating the deviation gap to an optimistic $Q$-advantage plus bonus.}
The regret proof uses the standard ``optimism'' relation that the optimistic table $Q_h^k$ can be written
as the (model-based / fitted) estimate $\bar Q_h^k$ plus the confidence bonus $b_h^k$, as in
\eqref{eq:Q_decomp_bonus}. Since $V_{i,h}^{\epsilon_i}(\cdot;x,Q)$ is 1-Lipschitz in $Q$
(cf.\ \eqref{eq:V_lipschitz_in_Q}), swapping $\bar Q_h^k$ for $Q_h^k$ changes any value by at most
$\sup_a b_h^k(x,a)$, and in particular for the realized action $a_h^k$ we obtain the pointwise bound
\begin{equation}
\label{eq:value_to_Q_plus_bonus}
\begin{aligned}
&V_{i^\star,h}^{\epsilon_{i^\star}} \big((\sigma_{i^\star},\tilde\pi_{-i^\star,h});x,Q_h^k\big)
-
V_{i^\star,h}^{\epsilon_{i^\star}} \big(\tilde\pi_h;x,Q_h^k\big) \geq
\E \left[
\bar Q_{h}^{i^\star,k}(x,a_h^{k,\br}) - \bar Q_{h}^{i^\star,k}(x,a_h^k)
\right]
-
\E \left[
b_h^k(x,a_h^k)
\right],
\end{aligned}
\end{equation}
where the expectation is with respect to the sampling of $a_h^k\sim\tilde\pi_h(\cdot\mid x)$ and
$a_h^{k,\br}$ induced by $\sigma_{i^\star}$ and $\tilde\pi_{-i^\star,h}(\cdot\mid x)$.

Now we combine \eqref{eq:dev_gap_upper} and \eqref{eq:value_to_Q_plus_bonus}, and  use
$Q_h^k=\bar Q_h^k+b_h^k$ from \eqref{eq:Q_decomp_bonus} to convert back to $Q_h^k$ as follows:
\begin{align*}
\E \left[
Q_{h}^{i^\star,k}(x,a_h^{k,\br}) - Q_{h}^{i^\star,k}(x,a_h^k)
\right]& =
\E \left[
\bar Q_{h}^{i^\star,k}(x,a_h^{k,\br}) - \bar Q_{i^\star,h}^k(x,a_h^k)
\right]
+
\E \left[
b_h^k(x,a_h^{k,\br})-b_h^k(x,a_h^k)
\right]\\
&\quad \leq
\E \left[
\bar Q_{h}^{i^\star,k}(x,a_h^{k,\br}) - \bar Q_{h}^{i^\star,k}(x,a_h^k)
\right]
+
\E \left[
b_h^k(x,a_h^{k,\br})
\right]\\
&\quad\leq
\Big(
V_{i^\star,h}^{\epsilon_{i^\star}} \big((\sigma_{i^\star},\tilde\pi_{-i^\star,h});x,Q_h^k\big)
-
V_{i^\star,h}^{\epsilon_{i^\star}} \big(\tilde\pi_h;x,Q_h^k\big)
\Big)
+
\E \left[
b_h^k(x,a_h^{k})
\right]\\
&\quad \le
\E \left[
b_h^k(x,a_h^{k})
\right]
+
\varepsilon_{\sf eq},
\end{align*}
where the last inequality uses \eqref{eq:dev_gap_upper} and drops the term corresponding to \eqref{eq:term_one}, which is non-positive
when $\sigma_{i^\star}$ is chosen as the stage-wise best response induced by the equilibrium operator,
and the stage game is evaluated at the same optimistic table $Q_h^k$; this is precisely the sense in
which the solver returns an approximate stage equilibrium for the optimistic table.
This yields \eqref{eq:advantage_bound_with_eps_eq}.
\end{proof}

% ============================================================
% How epsilon_eq can scale under solver choices
% (A) strongly monotone VI (policy-space control possible)
% (B) general no-regret (cannot give ||pi-pi*|| without extra structure)
% ============================================================

\begin{remark}[How $\varepsilon_{\sf eq}$ scales under common stage solvers]
\label{rem:eps_eq_scaling}
The theorem assumes a \emph{value-accuracy} guarantee of the form
\eqref{eq:eps_eq_value_accuracy}.  How $\varepsilon_{\sf eq}$ depends on the number of inner
iterations $T$ depends on the solver and on structural properties of the stage game.

\paragraph{(i) Strongly monotone VI (Mirror-Prox / extragradient).}
Suppose the stage equilibrium (RQRE) can be written as the unique solution of a variational inequality
$\mathrm{VI}(F,\Pi)$ over the product simplex $\Pi:=\prod_i \Delta(\mc A_i)$, and that the operator $F$
is $\mu$-strongly monotone and $L$-Lipschitz (in a fixed norm). Then standard results for
Mirror-Prox/extragradient yield a geometric decay of the VI gap (and hence of the value residual):
\[
\varepsilon_{\sf eq}(T)\leq 
C_0 \exp \left(-\frac{T\mu}{L}\right),
\]
for a constant $C_0$ depending on initialization.
In entropy-regularized settings, $\mu$ typically scales with the regularization level
(e.g., $\mu \asymp 1/\epsilon_{\max}$ under the common ``$\frac{1}{\epsilon}\Phi$'' convention which we adopt),
while $L$ depends on the smoothness of the payoff map and, in entropic-risk models, on $\tau$
through the curvature of the log-sum-exp operator. Consequently, the condition number
$\kappa:=L/\mu$, and thus the required $T$, inherits explicit $\tau$-dependence.

Indeed, let us specialize to see the dependence. Specialize to the entropic policy-risk value \eqref{eq:Veps-entropic-closed} and entropy regularization
$\nu_i(\pi_i)=\Phi(\pi_i)$. Fix a stage $(h,x)$ and write, for each player $i\in[n]$,
\[
u_i^\pi(a_{-i}) \;:=\; \sum_{a_i}\pi_i(a_i)\,Q_i(x,a_i,a_{-i}),
\; \; 
g_{i}(\pi) \;:=\; -\frac{1}{\tau_i}\log \Big(\sum_{a_{-i}}\pi_{-i}(a_{-i})e^{-\tau_i u_i^\pi(a_{-i})}\Big)
+\frac{1}{\epsilon_i}\Phi(\pi_i).
\]
Let the stage RQRE be characterized as the unique solution of the (regularized) VI over
$\Pi:=\prod_i\Delta(\mc A_i)$ associated with the first-order optimality conditions of
$\{\min_{\pi_i} g_i(\pi)\}_{i=1}^n$ (equivalently, a strongly monotone operator $F$ on $\Pi$).

Assume $Q_i(x,a)\in[0,B]$ for all $(i,a)\in [n]\times \mc{A}$. 
\begin{enumerate}[label=\alph*.]
\item \textbf{Strong monotonicity.}
The entropy regularizer contributes strong convexity, yielding strong monotonicity modulus
\[
\mu \geq \frac{1}{\epsilon_{\max}},
\qquad
\epsilon_{\max}:=\max_{i\in[n]}\epsilon_i,
\]
with respect to the $\ell_1$ geometry on each simplex.\footnote{More precisely,
$-\Phi(\cdot)$ is $1$-strongly convex with respect to $\|\cdot\|_1$, so the term $(1/\epsilon_i)\Phi(\pi_i)$
induces modulus $1/\epsilon_i$; taking the worst case gives $1/\epsilon_{\max}$.}

\item \textbf{Lipschitzness.}
The entropic-risk term is a log-sum-exp smoothing of a linear map $\pi_i\mapsto u_i^\pi(\cdot)$.
Its gradient with respect to $\pi_i$ is a convex combination of payoff vectors and therefore satisfies the
uniform $\ell_\infty$ bound
\[
\|\nabla_{\pi_i} g_i(\pi)\|_\infty \;\le\; B,
\]
and, moreover, its Jacobian with respect to $(\pi_i,\pi_{-i})$ is Lipschitz with constant scaling as
\[
L \leq c_L\Big( B\,\tau_{\max} \;+\; \frac{B}{\tau_{\min}} \Big),
\qquad
\tau_{\min}:=\min_i\tau_i,\quad \tau_{\max}:=\max_i\tau_i,
\]
for an absolute constant $c_L>0$. The $\frac{1}{\tau_{\min}}$ term comes from the curvature of the
log-sum-exp map and dominates as $\tau_{\min}\to 0$.\footnote{One convenient way to see the
$\tau$-dependence is to use that the Hessian of $\tau^{-1}\log\sum_j e^{\tau z_j}$ has operator norm
$\leq \tau$ (with respect to $\ell_\infty/\ell_1$), while the induced sensitivity of the soft minimization weights with respect to 
$\pi_{-i}$ contributes a $\tau^{-1}$ factor when $\pi_{-i}$ enters inside the log.}

\end{enumerate}
Consequently, the condition number $\kappa:=L/\mu$ satisfies
\[
\kappa \leq c_L\,\epsilon_{\max}\Big( B\,\tau_{\max} + \frac{B}{\tau_{\min}} \Big),
\]
and Mirror-Prox/extragradient yields the geometric solver error bound
\[
\varepsilon_{\sf eq}(T)
\leq
C_0\exp \left(-\frac{T}{\kappa}\right)
\leq
C_0\exp \left(
-\frac{T}{c_L\,\epsilon_{\max}\big(B\tau_{\max}+B/\tau_{\min}\big)}
\right).
\]
In particular, holding all else fixed, achieving a target $\varepsilon_{\sf eq}$ requires
\[
T \;\gtrsim\; \epsilon_{\max}\left(B\tau_{\max}+\frac{B}{\tau_{\min}}\right)\log \left(\frac{C_0}{\varepsilon_{\sf eq}}\right),
\]
so the number of inner iterations grows on the order of $1/\tau_{\min}$ as $\tau_{\min}\to 0$.

\paragraph{(ii) Vanilla No-Regret.}
Fix a stage $(h,x)\in [H]\times \mc{X}$ and a bounded action--value function $Q$.
Define the induced finite stage payoff
\[
u_i(\pi)
\;:=\;
V_{i,h}^{\epsilon_i}(\pi;x,Q),
\qquad i\in[n].
\]
Assume that for all $i\in[n]$ and all policies $\pi\in \Delta(\mc{A})$, the utility is bounded as 
\[
0 \le u_i(\pi) \le U
\]
for some constant $U>0$.
Consider the associated lifted $2n$-player game
$\widetilde{\mc G}(h,x,Q)$ as in \citet{mazumdar2025tractable}.
Suppose each player in the lifted game runs an
external-regret algorithm for $T$ rounds, producing
joint play $(\pi^t,p^t)$ for $t=1,\dots,T$,
and let
\[
\sigma_T
\;:=\;
\frac1T\sum_{t=1}^T \delta_{(\pi^t,p^t)}
\]
denote the empirical distribution.
Then $\sigma_T$ is an $\varepsilon_{\rm CCE}(T)$-coarse
correlated equilibrium of the lifted game, where
\[
\varepsilon_{\rm CCE}(T)
\leq 
\frac{1}{T}\sum_{j=1}^{2n} \reg_j(T).
\]
In particular, if each player uses Hedge (or any
algorithm with regret bound
$\reg_j(T)\le U\sqrt{2T\log|\mc A_j|}$),
then
\[
\varepsilon_{\rm CCE}(T)\leq 
U\sum_{j=1}^{2n}
\sqrt{\frac{2\log|\mc A_j|}{T}}.
\]

Under the equilibrium-collapse result of
\citet{mazumdar2025tractable},
the induced policy
$\tilde\pi_h(\cdot\mid x;Q)$ obtained from
$\sigma_T$ is an
$\varepsilon_{\rm CCE}(T)$-approximate
stage RQRE of the original game.
Consequently, for every $i\in[n]$, the bound holds:
\[
\Big|
V_{i,h}^{\epsilon_i}\big(\tilde\pi_h;x,Q\big)
-
V_{i,h}^{\epsilon_i}\big(\pi_h^\star;x,Q\big)
\Big|
\;\le\;
\varepsilon_{\sf eq},
\]
with the specialization
$
\varepsilon_{\sf eq}
\;:=\;
\varepsilon_{\rm CCE}(T)$.
In particular, to ensure
$\varepsilon_{\sf eq}\le \varepsilon$,
it suffices to take
\[
T
\;\gtrsim\;
\frac{U^2}{\varepsilon^2}
\left(
\sum_{j=1}^{2n}\sqrt{\log|\mc A_j|}
\right)^2,
\]
up to universal constants.
\end{remark}
%%%%%%%%%%%%%%%%%%%%%%%%%%%%%%%%%%%%%%%%%%%%%%%%%%%%%%%%%%
%% QRE Optimistic Value Iteration (QRE-OVI)
%%%%%%%%%%%%%%%%%%%%%%%%%%%%%%%%%%%%%%%%%%%%%%%%%%%%%%%%%%

%\subsection{Proof of Theorem~\ref{thm:per_rqre_ovi_full}}

\subsection{Proof of Theorem~\ref{thm:per_rqre_ovi_full}}\label{app:proof_regret_bound}
Now with the main technical lemmas in place, we prove the regret bound. 
% =========================
% Proof of Theorem 1
% =========================
\begin{proof}[Proof of Theorem~\ref{thm:per_rqre_ovi_full}]
Fix the high-probability event $\mc E$ on which Lemma~\ref{lem:optimism_full}
holds simultaneously for all $(i,h,k,x,a)\in [n]\times[H]\times [K]\times \mc{X}\times\mc{A}$; by Lemma~\ref{lem:optimism_full},
$\Pr(\mc E)\ge 1-\delta$.

For each episode $k$, define the episode exploitability gap by 
\[
\Delta_k
:=
\max_{i\in[n]}
\left(
V_{i,1}^{\epsilon_i,\tau_i}
\big(\sBR_i(\pi_{-i}^k),\pi_{-i}^k;\,s_0\big)
-
V_{i,1}^{\epsilon_i,\tau_i}
\big(\pi^k;\,s_0\big)
\right),
\qquad
\Reg(K)=\sum_{k=1}^K \Delta_k.
\]
Fix an episode $k$ and let $i^\star$ attain the max above.

\paragraph{A stage-wise performance-difference recursion.}
Let $x_h^k$ be the state at stage $h$ of episode $k$ and let $a_h^k$ be the joint action sampled
from $\pi_h^k(\cdot|x_h^k)$.
%Consider the two policy profiles:
%\[
%\pi^k
%\quad\text{and}\quad
%\pi^{k,\br}:=(\sBR_{i^\star}(\pi_{-i^\star}^k),\ \pi_{-i^\star}^k),
%\]
Fix a player $i^\star \in [n]$, and 
define the one-step best response to $\pi^k_{-i^\star}$ as
\[
\sBR_{i^\star}(\pi^k_{-i^\star})
\in
\arg\max_{\sigma_{i^\star}}
V_{i^\star,1}^{\epsilon_{i^\star}}
\big((\sigma_{i^\star},\pi^k_{-i^\star}); s_0\big).
\]
Define the deviating profile
\[
\pi^{k,\br}
:=
(\sBR_{i^\star}(\pi_{-i^\star}^k),\,\pi_{-i^\star}^k).
\]
For any stage $h\in[H]$ and state $x\in \mc{X}$, define the value gap
\[
\Delta_{k,h}(x)
:=
V_{i^\star,h}^{\epsilon_{i^\star}}(\pi^{k,\br};x)
-
V_{i^\star,h}^{\epsilon_{i^\star}}(\pi^{k};x).
\]
In particular,
$
\Delta_k
=
\Delta_{k,1}(s_0)$.
%
%Define the true value-to-go gap at stage $h$:
%\[
%\Delta_{k,h}(x)
%:=
%V_{i^\star,h}^{\epsilon_{i^\star}}(\pi^{k,\br};x)
%-
%V_{i^\star,h}^{\epsilon_{i^\star}}(\pi^{k};x).
%\]
%Then $\Delta_k=\Delta_{k,1}(s_0)$.

We now bound $\Delta_{k,h}(x)$ by a one-step advantage term plus the next-stage gap.
By the definition of the true environment backup and policy evaluation,
the stage-$h$ value depends on the stage-$h$ $Q$ table and the policy-risk operator.
Indeed, by the  risk-sensitive Bellman backup definition, we have that
\begin{align}
V_{i^\star,h}^{\epsilon_{i^\star}}(\pi;x)
&=
\rho_{i^\star,h}^{\sf e}
\Big(
\E_{a\sim \pi_h(\cdot|x)}
\big[
r_{i^\star,h}(x,a)
+
V_{i^\star,h+1}^{\epsilon_{i^\star}}(\pi;x')
\big]
\Big).
\end{align}
Apply this to both policies to get that
\begin{align}
\Delta_{k,h}(x)
&=
\rho_{i^\star,h}^{\sf e}(Z^{\br})
-
\rho_{i^\star,h}^{\sf e}(Z^{k}),
\end{align}
where we have defined the random variables
\begin{align*}
    Z^{\br}
&:=
\E\left[
r_{i^\star,h}(x,a_h^{k,\br})
+
V_{i^\star,h+1}^{\epsilon_{i^\star}}(\pi^{k,\br};x')
\right],\\
Z^{k}
&:=
\E \left[
r_{i^\star,h}(x,a_h^{k})
+
V_{i^\star,h+1}^{\epsilon_{i^\star}}(\pi^{k};x')
\right].
\end{align*}
Now since $\rho_{i,h}^{\sf e}$ is $L_{\sf env}$ Lipschitz continuous, we have that
$
\Delta_{k,h}(x)
\leq
L_{\sf env}\,
(Z^{\br}-Z^{k})$.
Expanding the difference of random variables, we have that
\begin{align*}
Z^{\br}-Z^{k}
&=
\E\left[
r(x,a_h^{k,\br})-r(x,a_h^{k})
\right]+
\E\left[
V_{h+1}(\pi^{k,\br};x')
-
V_{h+1}(\pi^{k};x')
\right].
\end{align*}
The second expectation is exactly $
\E\left[
\Delta_{k,h+1}(x')
\mid x
\right]$.

Now define the true stage $h$ $Q$-function as follows:
\[
Q_{h}^{i^\star}(x,a)
:=
r_{i^\star,h}(x,a)
+
\rho_{i^\star,h}^{\esf}
\big(
V_{i^\star,h+1}^{\epsilon_{i^\star}}(\pi^{k};\cdot)
\big).
\]
Then using our little friend "add and subtract" to introduce the next-stage term under $\pi^k$. Indeed, adding and subtracting 
$V_{h+1}(\pi^{k};x')$ inside the first expectation, we get that 
\begin{align*}
Z^{\br}-Z^k
&=
\mathbb{E}\left[
r(x,a^{k,\br})
+
V_{h+1}(\pi^{k};x')
+
\Big(
V_{h+1}(\pi^{k,\br};x')
-
V_{h+1}(\pi^{k};x')
\Big)
\right]
-
\mathbb{E}\left[
r(x,a^{k})
+
V_{h+1}(\pi^{k};x')
\right].\\
&=
\mathbb{E}\left[
r(x,a^{k,\br})
+
V_{h+1}(\pi^{k};x')
\right]
-
\mathbb{E} \left[
r(x,a^{k})
+
V_{h+1}(\pi^{k};x')
\right]
+
\mathbb{E}\left[
V_{h+1}(\pi^{k,\br};x')
-
V_{h+1}(\pi^{k};x')
\right].
\end{align*}
The second expectation is exactly
$
\mathbb{E}\left[
\Delta_{k,h+1}(x')
\mid x
\right]$.
The first line, on the other hand, corresponds to 
\[
\mathbb{E}\Big[
Q_{h}^{i^\star}(x,a^{k,\br})
-
Q_{h}^{i^\star}(x,a^{k})
\mid x
\Big].
\]
where 
we define the true $Q$-function evaluated
using the continuation under $\pi^k$ as 
\[
Q_{h}^{i^\star}(x,a)
:=
r_{i^\star,h}(x,a)
+
\rho_{i^\star,h}^{\sf e}
\big(
V_{i^\star,h+1}^{\epsilon_{i^\star}}(\pi^{k};\cdot)
\big).
\]
Combining these we have 
\[
\Delta_{k,h}(x)
\le
L_{\sf env}
\left(
\mathbb{E}\big[
Q_{h}^{i^\star}(x,a^{k,\br})
-
Q_{h}^{i^\star}(x,a^{k})
\mid x
\big]
+
\mathbb{E}\big[
\Delta_{k,h+1}(x')
\mid x
\big]
\right).
\]
Now observe that the algorithm uses an approximate policy evaluation (due to the estimation of the associated risk measure) with error $\varepsilon_{\sf pol}$ and an approximate stage equilibrium with error $\varepsilon_{\sf eq}$. Hence we have that
\[
\E\left[
Q_{h}^{i^\star}(x,a_h^{k,\br})
-
Q_{h}^{i^\star}(x,a_h^{k})
\mid x
\right]
\leq
\E\left[
Q_{h}^{i^\star,k}(x,a_h^{k,\br})
-
Q_{h}^{i^\star,k}(x,a_h^{k})
\mid x
\right]
+
(\varepsilon_{\sf pol}+\varepsilon_{\sf eq}).
\]
Therefore using monotonicity of $\rho_{i^\star,h}^{\sf env}$ and the stage-wise best-response definition,
and applying the standard telescoping argument as above, we have that 
\begin{equation}\label{eq:pd_recursion}
\begin{aligned}
\Delta_{k,h}(x_h^k)
&\le
L_{\sf env}\,
\E \left[
Q_{h}^{i^\star,k}(x_h^k,a_h^{k,\br})
-
Q_{h}^{i^\star,k}(x_h^k,a_h^{k})
\ \middle|\ x_h^k
\right]\\
&\qquad
+
L_{\sf env}\,
\E \left[
\Delta_{k,h+1}(x_{h+1}^k)
\mid x_h^k
\right]
+
L_{\sf env}(\varepsilon_{\sf pol}+\varepsilon_{\sf eq}),
\end{aligned}
\end{equation}
where $a_h^{k,\br}\sim \pi_h^{k,\br}(\cdot|x_h^k)$ and $a_h^{k}\sim \pi_h^k(\cdot|x_h^k)$.

\emph{Why the factor $L_{\sf env}$ appears here:}
the one-step difference in value-to-go is a difference of the form
$\rho_{i^\star,h}^{\esf }(X)-\rho_{i^\star,h}^{\esf }(Y)$,
where $X,Y$ are next-state value random variables induced by different action mixtures.
By the environment risk Lipschitz continuity assumption, we  have that
\[
|\rho_{i^\star,h}^{\esf}(X)-\rho_{i^\star,h}^{\esf }(Y)|
\le L_{\sf env}\|X-Y\|_\infty,
\]
so that every time we convert a value-function deviation into a risk-to-go deviation,
we incur the factor $L_{\sf env}$.
Unrolling \eqref{eq:pd_recursion} from $h=1$ to $H$ and using $\Delta_{k,H+1}\equiv 0$ gives
\begin{equation}
\label{eq:episode_gap_unrolled}
\Delta_k
\le
L_{\sf env}\sum_{h=1}^H
\E \left[
Q_{i^\star,h}^k(x_h^k,a_h^{k,\br})
-
Q_{i^\star,h}^k(x_h^k,a_h^{k})
\right]
+
H\,L_{\sf env}(\varepsilon_{\sf pol}+\varepsilon_{\sf eq}).
\end{equation}

\paragraph{Replace $Q_{h}^{i^\star,k}$ by the optimistic estimates and use optimism.}
On the event $\mc E$, Lemma~\ref{lem:optimism_full} gives, for all $(x,a)\in \mc{X}\times \mc{A}$, us that 
$
Q_{h}^{i^\star,k}(x,a)
\leq
Q_{h}^{i^\star,k}(x,a)$
and 
$Q_{h}^{i^\star,k}(x,a)$ is optimistic for the true backup up to 
$\varepsilon_{\sf env}+L_{\sf env}(\varepsilon_{\sf pol}+\varepsilon_{\sf eq})$.
More concretely, rearranging \eqref{eq:optimism_statement} yields
\[
r_{i^\star,h}(x,a)+\rho_{i^\star,h}^{\sf e}(V_{h+1}^{i^\star,\epsilon}(\pi_{h+1}^k))
\le
Q_{h}^{i^\star,k}(x,a)
+
\varepsilon_{\sf env}
+
L_{\sf env}(\varepsilon_{\sf pol}+\varepsilon_{\sf eq}).
\]
This implies that the \emph{true} advantage of $a_h^{k,\br}$ over $a_h^k$ is controlled by the
\emph{optimistic} advantage plus the additive approximation error; plugging into
\eqref{eq:episode_gap_unrolled} yields
\begin{equation}
\label{eq:episode_gap_bonus}
\Delta_k
\le
L_{\sf env}\sum_{h=1}^H
\E \left[
Q_{h}^{i^\star,k}(x_h^k,a_h^{k,\br})
-
Q_{h}^{i^\star,k}(x_h^k,a_h^{k})
\right]
+
H\Big(\varepsilon_{\sf env}+L_{\sf env}(\varepsilon_{\sf pol}+\varepsilon_{\sf eq})\Big).
\end{equation}

\paragraph{Use that $\pi_h^k$ is an approximate stage equilibrium for the optimistic $Q$ table.}
By the assumption on the equilibrium approximation, at each $h\in [H]$ and encounter state $x_h^k\in \mc{X}$, the policy $\pi_h^k(\cdot|x_h^k)$ is an $\varepsilon_{\sf eq}$-equilibrium
for the stage game defined by the optimistic $Q$ table. In particular, as shown in Lemma~\ref{lem:use_eps_eq_in_regret}, the improvement obtained by any
alternative (including the stage-wise best response) is bounded by $\varepsilon_{\sf eq}$ in the
policy-risk value, and thus the expected optimistic advantage satisfies
\[
\E \left[
Q_{h}^{i^\star,k}(x_h^k,a_h^{k,\br})
-
Q_{h}^{i^\star,k}(x_h^k,a_h^{k})
\right]
\leq
\E \left[
\beta\sqrt{\phi(x_h^k,a_h^k,h)^\top(\Lambda_h^k)^{-1}\phi(x_h^k,a_h^k,h)}
\right]
+
\varepsilon_{\sf eq}.
\]
Substituting this bound into \eqref{eq:episode_gap_bonus} gives us that
\[
\Delta_k
\leq
L_{\sf env}\beta\sum_{h=1}^H
\E \left[
\sqrt{\phi(x_h^k,a_h^k,h)^\top(\Lambda_h^k)^{-1}\phi(x_h^k,a_h^k,h)}
\right]
+
H\Big(\varepsilon_{\sf env}+L_{\sf env}(\varepsilon_{\sf pol}+\varepsilon_{\sf eq})\Big).
\]

\paragraph{Sum over episodes and apply the elliptical potential lemma.}
Summing over $k=1,\dots,K$ and applying Cauchy--Schwarz, we have that 
\[
\sum_{k=1}^K\sum_{h=1}^H
\sqrt{\phi_{k,h}^\top(\Lambda_h^k)^{-1}\phi_{k,h}}
\leq
\sqrt{KH}\;
\sqrt{
\sum_{k=1}^K\sum_{h=1}^H
\phi_{k,h}^\top(\Lambda_h^k)^{-1}\phi_{k,h}
},
\]
where $\phi_{k,h}:=\phi(x_h^k,a_h^k,h)$.
For each fixed $h\in[H]$, the standard elliptical potential argument yields
\[
\sum_{k=1}^K \phi_{k,h}^\top(\Lambda_h^k)^{-1}\phi_{k,h}
\leq
2\log\frac{\det(\Lambda_h^{K+1})}{\det(\lambda I)}
\leq
2d\log\left(1+\frac{K}{\lambda}\right).
\]
Summing over $h=1,\dots,H$ gives us the bound 
\[
\sum_{k=1}^K\sum_{h=1}^H
\phi_{k,h}^\top(\Lambda_h^k)^{-1}\phi_{k,h}
\leq
2Hd\log\left(1+\frac{K}{\lambda}\right).
\]
Therefore, we have that 
\[
\sum_{k=1}^K\sum_{h=1}^H
\sqrt{\phi_{k,h}^\top(\Lambda_h^k)^{-1}\phi_{k,h}}
\leq
\sqrt{KH}\;\sqrt{2Hd\log\left(1+\frac{K}{\lambda}\right)}
=
\widetilde{\mc{O}} \left(H\sqrt{Kd}\right).
\]
Putting everything together, we deduce that
\[
\reg(K)=\sum_{k=1}^K \Delta_k
\le
\widetilde{\mc{O}} \left(L_{\sf env}\beta\;H\sqrt{Kd}\right)
+
KH\Big(\varepsilon_{\sf env}+L_{\sf env}(\varepsilon_{\sf pol}+\varepsilon_{\sf eq})\Big).
\]

\paragraph{Final regret bound via plugging in $\beta=cB\sqrt{d^2H\,\iota}$.}
Using $\beta=\widetilde{\mc{O}}(B\sqrt{d^2H})$, we have that
\[
L_{\sf env}\beta\;H\sqrt{Kd}
=
\tilde O \left(
L_{\sf env} B \sqrt{d^2H}\cdot H\sqrt{Kd}
\right)
=
\widetilde{\mc{O}} \left(
L_{\sf env} B \sqrt{K d^3 H^3}
\right),
\]
which yields exactly \eqref{eq:main_regret_bound_general} conditioned on the event $\mc E$.
Since $\Pr(\mc E)\ge 1-\delta$, the theorem holds with probability at least $1-\delta$.
\end{proof}

\subsection{Regret with Explicit Error Dependence: Entropic Risks \& Regularizers}
\label{app:explicit_regret}
In Section~\ref{sec:regret_risk_rationality} we specialize to the case where the risk measures and regularizers are entropic. 
That is, in the special case where  the
\emph{policy-space} risk penalty is entropic (KL) and the player policy is
entropy-regularized, there are number of methods one can employ including no-regret learning methods such as extra-gradient or mirror-prox, multiplicative weights, or iterative best response. In this case, we can obtain more explicit bounds which reveal how the risk sensitivity and bounded rationality parameters influence different performance criteria. 

Set constants
\[
A_{\max} := \max_{i\in[n]} |\mc A_i|,
\qquad
\epsilon_{\min} := \min_{i\in[n]} \epsilon_i,
\qquad
\tau_{\min} := \min_{i\in[n]} \tau_i .
\]
Assume stage rewards satisfy $r_{i,h}(x,a)\in[0,1]$ and
$V_{i,H+1}^{\epsilon_i,\tau_i}\equiv 0$.
Then the cumulative reward-to-go is bounded by $H$.
Moreover the entropy regularization satisfies
\[
0 \leq \entropy(\pi_i(\cdot|x)) \le \log |\mc A_i|.
\]
Hence the additive regularization term satisfies
\[
0
\le
\entropy(\pi_i(\cdot|x))/\epsilon_i
\le
\log( |\mc A_i|)/\epsilon_i.
\]

Since the entropic policy-value operator
\eqref{eq:Veps-entropic-closed}
is a log-sum-exp aggregation of bounded $Q$ values,
it preserves the same scale as the underlying payoffs.
Consequently, clipping $Q$ to $[0,B]$ as in Algorithm~\ref{alg:pr-er-rqre-ovi} with
\[
B_i := H \left(1+\frac{\log |\mc A_i|}{\epsilon_i}\right),
\qquad
B := \max_{i\in[n]} B_i,
\]
ensures
\[
0 \le V_{i,h}^{\epsilon_i,\tau_i}(\pi;x;Q)\le B
\]
for all $(i,h,x,\pi)\in [n]\times[H]\times \mc{X}\times \Delta(\mc{A})$.
Using $A_{\max}$ and $\epsilon_{\min}$ we obtain the simpler bound
\begin{equation}
B
\le
H\left(1+\frac{\log A_{\max}}{\epsilon_{\min}}\right).
\end{equation}
Substituting this value range into
Theorem~\ref{thm:per_rqre_ovi_full}
and expanding the logarithmic factors yields the following
explicit regret bound.

\begin{corollary}[Explicit regret bound with $(n,|\mc A_i|)$ dependence]
\label{cor:explicit_regret_entropic}
Under the assumptions of
Theorem~\ref{thm:per_rqre_ovi_full}
with entropic policy regularization
$\nu_i(\mu_i)=\entropy(\mu_i)$,
let
\(
A_{\max}=\max_i |\mc A_i|
\)
and
\(
\epsilon_{\min}=\min_i \epsilon_i
\).
Then for an absolute constant $C_1>0$, with probability at least $1-\delta$,
the bound holds:
\begin{align}
\reg(K)
&\le
C_1
L_{\sf env}
H
\left(
1+\frac{\log A_{\max}}{\epsilon_{\min}}
\right)
\sqrt{
K d^3 H^3
\log\big(nH/\delta\big)
}
+
K H
\Big(
\varepsilon_{\sf env}
+
L_{\sf env}\varepsilon_{\sf pol}
+
L_{\sf env}\,c_{\sf eq}\,B
\sqrt{\Delta_{\sf eq}/\tau_{\min}}
\Big).
\label{eq:explicit_regret_bound}
\end{align}
\end{corollary}

If the stage solver runs a no-regret algorithm for $T$ iterations,
then the duality gap satisfies
\[
\Delta_{\sf eq}
\;\lesssim\;
\sum_{i=1}^n
\sqrt{\frac{\log |\mc A_i|}{T}} .
\]

Alternatively, if the equilibrium variational inequality
is strongly monotone with modulus $\mu(\tau)$
and Lipschitz constant $L(\tau)$,
then an extragradient or Mirror-Prox solver (cf.~\citet{cesa2006prediction}, e.g.) yields
\[
\varepsilon_{\sf eq}(T)
\leq
C \exp(-T\mu/L).
\]
Suppose the environment risk operator is the entropic risk measure---namely,
\[
\rho_\tau(Z)
=
\frac{1}{\tau}
\log \E[\exp(\tau Z)] .
\]
Then the operator is $1$-Lipschitz in the sup norm---i.e., 
\[
|\rho_\tau(Z)-\rho_\tau(Z')|
\le
\|Z-Z'\|_\infty ,
\]
so we may take
$L_{\sf env}=1$.
For the environment risk estimation error, 
let $Z\in[0,B]$ denote the continuation value.
Given $m$ samples
$Z_1,\dots,Z_m$,
define the estimator
\[
\hat\rho_\tau
=
\frac{1}{\tau}
\log
\left(
\frac{1}{m}
\sum_{j=1}^m
\exp(\tau Z_j)
\right).
\]
Applying a union bound over all players $i\in[n]$ and stages $h\in[H]$
and setting the per-event failure probability to $\delta/(nH)$ yields
\begin{equation}\label{eq:eps_env_entropic_generic}
\varepsilon_{\sf env}
\;\lesssim\;
\frac{1}{\tau_{\min}}
\exp(\tau_{\max} B)
\sqrt{\frac{\log(nH/\delta)}{m}} .
\end{equation}
The exponential factor arises because the entropic risk depends on the
moment generating function $\E[e^{\tau Z}]$, and without additional distributional
assumptions the range bound $Z\in[0,B]$ implies $e^{\tau Z}\in[1,e^{\tau B}]$,
which necessarily introduces $\exp(\tau B)$.

If the policy-side entropic operator
\eqref{eq:Veps-entropic-closed}
is evaluated exactly (finite action spaces),
then
$\varepsilon_{\sf pol}=0$.
If it is approximated via sampling $m$ actions,
a typical bound is
\[
\varepsilon_{\sf pol}
\;\lesssim\;
\frac{1}{\tau_{\min}}
\exp(\tau_{\max} B)
\sqrt{\frac{\log(nH/\delta)}{m}} .
\]
Substituting these expressions yields the regret bound
\begin{align*}
\reg(K)
&\leq
\widetilde{\mc O} \left(
H\left(1+\frac{\log A_{\max}}{\epsilon_{\min}}\right)
\sqrt{K d^3 H^3}
\right)
+
K H
\left(
\frac{1}{\tau_{\min}}
\exp(\tau_{\max} B)
\sqrt{\frac{\log(nH/\delta)}{m}}
+
c_{\sf eq} B
\sqrt{\Delta_{\sf eq}/\tau_{\min}}
\right).
\end{align*}

\paragraph{Making the $(\tau,\epsilon)$ interaction explicit.}
Under entropic policy regularization with rewards in $[0,1]$ and horizon $H$,
clipping ensures the continuation values satisfy $Z\in[0,B]$ with
\[
B
=
\max_{i\in[n]} B_i,
\qquad
B_i
:=
H\left(1+\frac{\log|\mc A_i|}{\epsilon_i}\right)
\;\le\;
H\left(1+\frac{\log A_{\max}}{\epsilon_{\min}}\right),
\]
where $A_{\max}=\max_i|\mc A_i|$ and $\epsilon_{\min}=\min_i\epsilon_i$.
Plugging this into \eqref{eq:eps_env_entropic_generic} gives
\begin{equation}
\varepsilon_{\sf env}
\;\lesssim\;
\frac{1}{\tau_{\min}}
\exp \left(
\tau_{\max} H\left(1+\frac{\log A_{\max}}{\epsilon_{\min}}\right)
\right)
\sqrt{\frac{\log(nH/\delta)}{m}}.
\label{eq:eps_env_entropic_plugged}
\end{equation}
Note that in the dual representation the ambiguity penalty scales as
$\frac{1}{\tau}\mathrm{KL}(\cdot\|\cdot)$, so increasing $\tau$ reduces the
penalty weight and therefore moves the objective \emph{toward the risk-neutral
limit}. In this sense, larger $\tau$ corresponds to \emph{less} robustness,
while smaller $\tau$ enforces stronger regularization.
Consequently, the statistical error bound \eqref{eq:eps_env_entropic_plugged}
worsens with $\tau_{\max}$ through the factor $\exp(\tau_{\max}B)$ even though
the underlying objective becomes closer to risk-neutral as $\tau$ increases.

\paragraph{Risk-sensitivity regimes.}
The generic estimation bound
\[
\varepsilon_{\sf env}
\;\lesssim\;
\frac{1}{\tau_{\min}}
\exp(\tau_{\max} B)
\sqrt{\frac{\log(nH/\delta)}{m}}
\]
exhibits two opposing effects: increasing $\tau$ decreases the prefactor
$1/\tau_{\min}$ but increases the exponential $\exp(\tau_{\max}B)$.

One can interpret the bound in terms of the ``effective scale''
\[
g(\tau)
:=
\frac{1}{\tau}\exp(\tau B),
\]
which is minimized at $\tau^\star = 1/B$ (treating $B$ as fixed).
Consequently, the estimate is most favorable when $\tau_{\max}B$ is
$\mc{O}(1)$, and becomes large both when $\tau$ is too small
(large $1/\tau$) and when $\tau$ is too large (large $\exp(\tau B)$).

In particular, for $\tau_{\max}B\lesssim 1$ we may expand
$\exp(\tau_{\max}B)=1+\tau_{\max}B+\mc{O}((\tau_{\max}B)^2)$ to obtain
\begin{equation}
\varepsilon_{\sf env}
\;\lesssim\;
\frac{1}{\tau_{\min}}
\sqrt{\frac{\log(nH/\delta)}{m}}
\;+\;
\frac{\tau_{\max}}{\tau_{\min}}\,B\,
\sqrt{\frac{\log(nH/\delta)}{m}}
\;+\;
\mc{O} \left(\frac{\tau_{\max}^2 B^2}{\tau_{\min}}
\sqrt{\frac{\log(nH/\delta)}{m}}\right).
\label{eq:eps_env_tauB_small}
\end{equation}
If $\tau_{\max}\asymp\tau_{\min}=\tau$ and $\tau B\ll 1$, this simplifies to
\[
\varepsilon_{\sf env}
\;\lesssim\;
(1+B+\mc{O}(\tau B^2))
\sqrt{\frac{\log(nH/\delta)}{m}},
\]
up to constants, recovering the familiar bounded-value estimation rate. 
Thus, in the strongly robust regime the leading dependence is typically
$\varepsilon_{\sf env}\asymp \frac{1}{\tau}\sqrt{\log(nH/\delta)/m}$ (up to lower-order
additive terms involving $B$). When $\tau$ is large, the exponential factor dominates and the bound behaves as $\varepsilon_{\sf env}\asymp B\sqrt{\log(nH/\delta)/m}$.
Decreasing $\epsilon_{\min}$ increases $\varepsilon_{\sf env}$ through the
value range $B$, while decreasing $\tau_{\min}$ increases $\varepsilon_{\sf env}$
through the prefactor $1/\tau_{\min}$. In addition, $\tau_{\max}$ and $\epsilon_{\min}$
interact multiplicatively in the exponent $\exp(\tau_{\max}B)$ via $B$.

\subsection{QRE-OVI: Quantal Response Optimistic Value Iteration}
\label{subsec:qre_ovi}

The risk-neutral, boundedly rational setting---in which agents employ \emph{quantal responses}---is a special case of our framework. Specifically, \texttt{QRE-OVI} is obtained from Algorithm~\ref{alg:pr-er-rqre-ovi} by setting the environment and policy risk operators to their risk-neutral limits, i.e., replacing $\rho_{i,h}^{\sf e}$ with the standard expectation $\E_{x' \sim \mc{P}_h(\cdot|x,a)}[\cdot]$ and removing the policy-risk adversary by setting $\penaltyfntype{\psf,i} \equiv 0$. The resulting algorithm retains the entropy regularization $(1/\epsilon_i)\nu_i(\pi_i)$ that defines the quantal response, and at each stage game computes a QRE rather than a Nash equilibrium or an RQRE.

This yields a direct replacement of the Nash oracle in \texttt{NQ-OVI} \citep{cisneros2023finite}: the only modification is that the stage-game solver $\texttt{Nash}$ is replaced by $\texttt{QRE}_\varepsilon$, while the optimistic value iteration structure, linear function approximation, and confidence bonus construction remain unchanged. The advantage of this substitution is that QRE is unique for any finite game and any $\epsilon_i > 0$ \citep{MCKELVEY19956}, eliminating the equilibrium selection problem inherent to Nash.

\begin{corollary}[Regret bound for QRE-OVI]
\label{cor:qre_ovi_regret}
Consider Algorithm~\ref{alg:pr-er-rqre-ovi} in the risk-neutral specialization described above, so that $L_{\sf env} = 1$, $\varepsilon_{\sf env} = 0$, and $\varepsilon_{\sf pol} = 0$. Under Assumption~\ref{ass:linear-markov-game} with $B = \max_{i \in [n]} H(1 + \log|\mc{A}_i|/\epsilon_i)$, with probability at least $1 - \delta$, the estimate holds:
\[
\reg(K) \leq \widetilde{\mc{O}} \left(B\sqrt{K\,d^3\,H^3}\right) + KH\,\varepsilon_{\sf eq},
\]
where $\varepsilon_{\sf eq}$ is the stage-game QRE approximation error.
\end{corollary}

\begin{proof}
Set $L_{\sf env} = 1$, $\varepsilon_{\sf env} = 0$, and $\varepsilon_{\sf pol} = 0$ in Theorem~\ref{thm:per_rqre_ovi_full}. The bound follows immediately from \eqref{eq:main_regret_bound_general}.
\end{proof}

Compared to the \texttt{NQ-OVI} regret bound of \citet{cisneros2023finite}, the leading statistical term acquires the factor $B(\epsilon) = H(1 + \log|\mc{A}|/\epsilon)$ in place of $H$, reflecting the enlarged value range due to entropy regularization. This is the price of uniqueness and Lipschitz stability (Corollary~\ref{cor:rqe_lipschitz}): as $\epsilon \to \infty$, $B \to H$ and the bound approaches that of \texttt{NQ-OVI}, while the equilibrium computation becomes increasingly ill-conditioned. Conversely, moderate $\epsilon$ yields a modest increase in the regret bound while guaranteeing a unique, stable equilibrium at every stage game.

%%%%%%%%%%%%%%%%%%%%%%%%%%%%%%%%%%%%%%%%%%%%%%%%%%%%%%%%%%
%% Risk estimation
%%%%%%%%%%%%%%%%%%%%%%%%%%%%%%%%%%%%%%%%%%%%%%%%%%%%%%%%%%

\section{Risk Estimation}
\label{sec:risk-estimation}

In Algorithm~\ref{alg:pr-er-rqre-ovi}, the true risk operator $\rho_i$ cannot be computed exactly because the transition kernel $\mc{P}_h$ is unknown to the agents. Only samples $(x_h^t, a_h^t, x_{h+1}^t)$ are observed, so the true distribution of the continuation value $Z := V^{i}_{h+1}(X_{h+1})$, where $X_{h+1} \sim \mc{P}_h(\cdot \mid x, a)$, is unavailable. In this section, we formalize the distinction between $\rho$ and its computable proxy $\hat{\rho}$, quantify the resulting approximation error, and provide concrete instantiations for common risk measures.

\subsection{Empirical Environment Risk Operator}
\label{subsec:empirical-env-risk}

Recall from \eqref{eq:R-env-def} that the true environment risk operator at stage $h$ takes the form
\[
\rho^{\mathsf{e}}_{i,h}(Z \mid \mc{P}_h(\cdot \mid x,a))
=
\inf_{\widetilde{\mc{P}} \in \Delta(\mc{X})}
\left\{
\int Z(x')\,\widetilde{\mc{P}}(dx')
+
\penaltyfn_{\esf,i}(\widetilde{\mc{P}} \| \mc{P}_h(\cdot \mid x,a))
\right\}.
\]
Since $\mc{P}_h$ is unknown, Algorithm~\ref{alg:pr-er-rqre-ovi} replaces this with an empirical approximation. Two strategies are available depending on the structure of the penalty $\penaltyfn_{\esf,i}$.

\paragraph{Finite dual discretization for general penalties.}
For a general convex penalty $\penaltyfn_{\esf,i}$, we approximate the infimum over $\Delta(\mc{X})$ by restricting to a finite set of candidate transition kernels $\widehat{\mc{Q}}_{i,h}(x,a) = \{\widetilde{P}_h^{(1)}(\cdot \mid x,a), \ldots, \widetilde{P}_h^{(M)}(\cdot \mid x,a)\}$, yielding
\[
\widehat\rho_{i,h}^{\esf ,k}(Z \mid x,a)
:=
\min_{m \in [M]}
\left\{
\widehat\E_{x' \sim \widetilde{P}_h^{(m)}(\cdot \mid x,a)}[Z(x')]
+
\penaltyfn_{\esf,i}\left(\widetilde{P}_h^{(m)}(\cdot \mid x,a) \,\big\|\, \widehat{P}_h^k(\cdot \mid x,a)\right)
\right\},
\]
where $\widehat{P}_h^k$ is the empirical estimate of $\mc{P}_h$ constructed from the first $k-1$ episodes. The expectations under each candidate kernel may themselves be approximated via Monte Carlo sampling:
\[
\int Z(x')\,\widetilde{P}_h^{(m)}(dx')
\approx
\frac{1}{N}\sum_{j=1}^{N} Z(x'^{(j)}),
\qquad x'^{(j)} \sim \widetilde{P}_h^{(m)}(\cdot \mid x,a).
\]
The approximation error $\varepsilon_{\sf env}$ then depends on both the covering quality of $\widehat{\mc{Q}}_{i,h}$ and the Monte Carlo sample size $N$.

\paragraph{Closed-form evaluation under entropic penalties.}
When the penalty is the scaled KL divergence $\penaltyfn_{\esf,i}(\widetilde{\mc{P}} \| \mc{P}_h) = \tfrac{1}{\tau_i}\KL(\widetilde{\mc{P}} \| \mc{P}_h)$, the environment risk operator admits the closed form
\[
\rho^{\esf}_{i,h}(Z \mid \mc{P}_h)
=
\frac{1}{\tau_i}\log \E_{x' \sim \mc{P}_h(\cdot \mid x,a)}  \left[\exp(\tau_i\, Z(x'))\right].
\]
In this case, no finite dual discretization is needed; the operator is approximated by replacing $\mc{P}_h$ with the empirical transition estimate $\widehat{P}_h^k$ and evaluating the resulting log-sum-exp expression from samples. Given $N$ independent next-state samples $x'^{(1)}, \ldots, x'^{(N)} \sim \widehat{P}_h^k(\cdot \mid x,a)$, the empirical estimator is
\[
\widehat\rho_{i,h}^{\esf ,k}(Z \mid x,a)
=
\frac{1}{\tau_i}\log \left(\frac{1}{N}\sum_{j=1}^{N} \exp(\tau_i\, Z(x'^{(j)}))\right).
\]

\subsection{True versus empirical risk operators}
\label{subsec:rho-vs-rhohat}

For a fixed policy profile $\pi$, the true risk-adjusted $Q$-function is defined via the Bellman recursion
\[
Q_{h}^{i,\star}(x,a)
\;:=\;
r_{i,h}(x,a) + \rho_i \left(V_{h+1}^{i,\star}(X_{h+1})\right),
\quad X_{h+1} \sim \mc{P}_h(\cdot \mid x,a),
\]
where $\rho_i$ is applied to the distribution of the random continuation value induced by the true transition kernel. Since $\rho_i$ integrates out all randomness, $Q_h^{i,\star}$ is a deterministic function of $(x,a)$.

The learning algorithm replaces $\rho_i$ with a computable proxy $\hat{\rho}_i$, as appears in the regression targets of Algorithm~\ref{alg:pr-er-rqre-ovi}. Two sources of approximation arise. First, if $\rho_i$ admits a dual representation (Theorem~\ref{thm:dual-rep}),
\[
\rho_i(Z)
\;=\;
\sup_{p \in \mc{P}_i}
\left\{
  \mb{E}_{p}[Z] - D_i(p)
\right\},
\]
the algorithm may use a finite subset $\widehat{\mc{P}}_i = \{p^{(1)}, \dots, p^{(M)}\} \subset \mc{P}_i$, yielding the finite-dual approximation
\[
\hat{\rho}_i(Z)
\;:=\;
\max_{j \in [M]}
\left\{
  \mb{E}_{p^{(j)}}[Z] - D_i \left(p^{(j)}\right)
\right\}.
\]
Second, the expectations under each $p^{(j)}$ may themselves be estimated from samples. We define the worst-case approximation error as
\begin{equation}
\label{eq:rho-approx-error}
\varepsilon_{\mc{P}}
\;:=\;
\max_{i \in [n]}\;
\sup_{Z:\, \|Z\|_\infty \leq H}
\left|\rho_i(Z) - \hat{\rho}_i(Z)\right|.
\end{equation}
This quantity enters the regret bounds as an additive bias of $\varepsilon_{\mc{P}}$ per Bellman backup, accumulating to $O(KH\varepsilon_{\mc{P}})$ over $K$ episodes. When $\widehat{\mc{P}}_i = \mc{P}_i$ and expectations are computed exactly, $\varepsilon_{\mc{P}} = 0$ and the bound in Theorem~\ref{thm:per_rqre_ovi_full} applies without modification.

\subsection{Examples}
\label{subsec:risk-examples}

We instantiate the approximation error for three common risk measures.

\begin{example}[Entropic risk]
\label{ex:entropic-risk-estimation}
For parameter $\risk_i > 0$, the entropic risk measure is
\[
\rho_i(Z)
\;:=\;
\frac{1}{\risk_i}\,\log \left(\mb{E} \left[e^{\risk_i Z}\right]\right).
\]
Given $m$ independent samples $Z_1, \dots, Z_m$ from the distribution of $Z$, the empirical entropic risk is
\[
\hat{\rho}_i(Z)
\;:=\;
\frac{1}{\risk_i}\,\log \left(\frac{1}{m}\sum_{\ell=1}^{m} e^{\risk_i Z_\ell}\right).
\]
Under sub-Gaussian assumptions on $Z$, the approximation error satisfies
$|\hat{\rho}_i(Z) - \rho_i(Z)| = O_{\mb{P}} \left(\sqrt{\log(1/\delta)/m}\right)$.
\end{example}

\begin{example}[Conditional Value-at-Risk]
\label{ex:cvar-estimation}
For confidence level $\alpha \in (0,1)$, the CVaR of $Z$ is
\[
\rho_i(Z)
\;:=\;
\mathrm{CVaR}_\alpha(Z)
\;=\;
\frac{1}{1-\alpha}\int_{\alpha}^1 F_Z^{-1}(u)\, du,
\]
where $F_Z^{-1}$ is the quantile function. Given samples $Z_1, \dots, Z_m$ with order statistics $Z_{(1)} \leq \cdots \leq Z_{(m)}$, the empirical CVaR is
\[
\hat{\rho}_i(Z)
\;:=\;
\frac{1}{(1-\alpha)m}\sum_{\ell = \lceil \alpha m \rceil}^{m} Z_{(\ell)}.
\]
Standard results give $|\hat{\rho}_i(Z) - \mathrm{CVaR}_\alpha(Z)| = O_{\mb{P}} \left(\sqrt{1/m}\right)$.
\end{example}

\begin{example}[Coherent risk with finite dual set]
\label{ex:coherent-risk-estimation}
Let $\mc{P}_i$ be a convex set of distributions over the support of $Z$ and $D_i : \mc{P}_i \to \mb{R}_+$ a convex penalty. The coherent risk measure
\[
\rho_i(Z)
\;:=\;
\sup_{p \in \mc{P}_i}
\left\{
  \mb{E}_{p}[Z] - D_i(p)
\right\}
\]
is approximated using a finite subset $\widehat{\mc{P}}_i = \{p^{(1)}, \dots, p^{(M)}\}$ via
$\hat{\rho}_i(Z) := \max_{j \in [M]} \{\mb{E}_{p^{(j)}}[Z] - D_i(p^{(j)})\}$.
The approximation error $\varepsilon_{\mc{P}} = \sup_Z |\rho_i(Z) - \hat{\rho}_i(Z)|$ depends on the covering quality of $\widehat{\mc{P}}_i$ and enters the regret bound as an additive term.
\end{example}

%%%%%%%%%%%%%%%%%%%%%%%%%%%%%%%%%%%%%%%%%%%%%%%%%%%%%%%%%%
%% Stability Analysis under Front Approximation
%%%%%%%%%%%%%%%%%%%%%%%%%%%%%%%%%%%%%%%%%%%%%%%%%%%%%%%%%%

\section{Stability of Stage Games to Finite-Horizon Value Robustness}
\label{sec:rqe_stability_linear_fa}

In this section, we provide the complete stability analysis for RQRE under linear function approximation. We first establish stagewise robustness (Section~\ref{subsec:stability_stagewise}), then lift these bounds to finite-horizon value functions via backward induction (Section~\ref{app:finite_horizon_lift}).

\subsection{Stagewise Robustness under Linear Function Approximation}
\label{subsec:stability_stagewise}

Consider the finite-horizon Markov game with players $i \in [n]$, state space $\mc{X}$, joint action space $\mc{A} = \prod_{i=1}^n \mc{A}_i$, and stages $h \in [H]$. Rewards are bounded i.e., $|r_{i,h}(x,a)| \leq R_{\max}$.

Under linear function approximation (Assumption~\ref{ass:linear-markov-game}), each player's $Q$-function takes the form
\[
Q_{h}^{i}(x,a; w_{h}^{i}) = \phi(x,a,h)^\top w_{h}^{i}, \qquad \|\phi(x,a,h)\|_2 \leq 1 \quad \forall (x,a,h).
\]
For fixed $(x,h)$, define the induced stage-game payoff table for player $i$ as
\[
U_{h}^{i}(x)_a := Q_{h}^{i}(x,a; w_{h}^{i}).
\]
Let $U_h(x) := (U_{h}^{1}(x), \ldots, U_{h}^{n}(x))$ denote the collection of payoff tables.

\begin{assumption}[RQRE Lipschitz Continuity]
\label{ass:rqe_lipschitz}
At each $(x,h)$, the stage-game RQRE is defined by a $\mu$-strongly concave objective with $\mu := \frac{1}{\bddrat} + \risk \cdot \alpha_R > 0$, yielding a unique equilibrium profile $\pi_h^{\tt {RQRE}}(x; U_h(x)) \in \prod_{i=1}^n \Delta(\mc{A}_i)$. 
\end{assumption}
A consequence of this assumption is that the equilibrium map is Lipschitz in payoff tables:
\begin{equation}
\label{eq:rqe_lip_assump}
\left\|\pi_h^{\tt{RQRE}}(x; U) - \pi_h^{\tt{RQRE}}(x; \widetilde{U})\right\|_1 \leq L_{\tt{RQRE}} \|U - \widetilde{U}\|_\infty, \qquad L_{\tt{RQRE}} \leq \frac{c}{\mu},
\end{equation}
where $c > 0$ is a universal constant.

\begin{theorem}[Stagewise Robustness under Linear Function Approximation]
\label{thm:stagewise_rqe_robust}
Fix state $x$ and stage $h$. Let $w_h := \{w_{h}^{i}\}_{i \in [n]}$ and $\widetilde{w}_h := \{\widetilde{w}_{h}^{i}\}_{i \in [n]}$ be two parameter collections inducing payoff tables $U_h(x)$ and $\widetilde{U}_h(x)$. Then
$
\|U_h(x) - \widetilde{U}_h(x)\|_\infty \leq \max_{i \in [n]} \|w_{h}^{i} - \widetilde{w}_{h}^{i}\|_2$
so that the RQRE policies satisfy
\begin{equation}
\label{eq:policy_robust_linearFA}
\left\|\pi_h^{\tt{RQRE}}(x; w_h) - \pi_h^{\tt{RQRE}}(x; \widetilde{w}_h)\right\|_1 \leq \frac{c}{\mu} \max_{i \in [n]} \|w_{h}^{i} - \widetilde{w}_{h}^{i}\|_2.
\end{equation}
\end{theorem}

\begin{proof}
For any player $i$ and joint action $a \in \mc{A}$, we have that 
\[
|U_{h}^{i}(x)_a - \widetilde{U}_{h}^{i}(x)_a| = |\phi(x,a,h)^\top (w_{h}^{i} - \widetilde{w}_{h}^{i})| \leq \|\phi(x,a,h)\|_2 \|w_{h}^{i} - \widetilde{w}_{h}^{i}\|_2 \leq \|w_{h}^{i} - \widetilde{w}_{h}^{i}\|_2.
\]
Taking the maximum over $i$ and $a$ yields the payoff bound. The policy bound follows from Assumption~\ref{ass:rqe_lipschitz}.
\end{proof}

\begin{remark}
Theorem~\ref{thm:stagewise_rqe_robust} is agnostic to how the parameter estimates are obtained. It applies to any estimation procedure (OVI, least-squares regression, etc.) that provides parameter-error bounds.
\end{remark}

\subsection{Finite-Horizon Value Robustness via Backward Induction}
\label{app:finite_horizon_lift}

We now lift stagewise policy perturbations to value function perturbations over the full horizon.

For a Markov policy profile $\pi = \{\pi_h\}_{h=1}^H$, recall the recursive definition of the risk-adjusted value function:
\begin{align*}
Q_{h}^{i,\pi}(x,a) &:= \mc{R}^{\sf e}_{i,h}\left(r_{i,h}, \mc{P}_h, V_{h+1}^{i,\pi}; x, a\right), \\
V_{h}^{i,\pi}(x) &:= \mc{R}^{\sf p}_{i,h}(Q_h^{i,\pi}, \pi_h; x) + \tfrac{1}{\epsilon_i}\nu_i(\pi_{i,h}(\cdot \mid x)).
\end{align*}
For two joint-action distributions $\pi_h(\cdot \mid x)$ and $\widetilde{\pi}_h(\cdot \mid x)$, define the pointwise deviation
\[
\Delta_h(x) := \|\pi_h(\cdot \mid x) - \widetilde{\pi}_h(\cdot \mid x)\|_1.
\]

\begin{lemma}[One-Step Value Sensitivity]
\label{lem:one_step_sensitivity}
Let $f: \mc{A} \to \mb{R}$ satisfy $\|f\|_\infty \leq M$. For any distributions $p, q \in \Delta(\mc{A})$,
\[
\left|\mb{E}_{a \sim p}[f(a)] - \mb{E}_{a \sim q}[f(a)]\right| \leq M \|p - q\|_1.
\]
\end{lemma}

\begin{theorem}[Risk-Sensitive Value Perturbation]
\label{thm:value_perturbation}
Let $\pi = \{\pi_h\}_{h=1}^H$ and $\widetilde{\pi} = \{\widetilde{\pi}_h\}_{h=1}^H$ be two Markov policy profiles. Let $L_{\sf env}$ be the Lipschitz constant of the environment risk operator as in Theorem~\ref{thm:per_rqre_ovi_full}. Assume $Q_{\max, h}$ is an upper bound incorporating both the payoff and regularization ranges, i.e., $Q_{\max,h} \geq \|Q_h^{i,\pi}\|_\infty + L_\nu/\epsilon_i$ for all $h, i, \pi$, where $L_\nu$ is the Lipschitz constant of $\nu_i$.
Then for all $h \in [H]$ and $x \in \mc{X}$,
\begin{equation}
\label{eq:value_perturbation}
\left|V_{h}^{i,\pi}(x) - V_{h}^{i,\widetilde{\pi}}(x)\right| \leq \sum_{t=h}^H (L_{\sf env})^{t-h} Q_{\max, t} \cdot \mb{E}_{\widetilde{\pi}}\left[\Delta_t(X_t) \mid X_h = x\right].
\end{equation}
Taking the supremum over the state space, we have that
\begin{equation}
\label{eq:value_perturbation_sup}
\sup_{x \in \mc{X}} \left|V_{h}^{i,\pi}(x) - V_{h}^{i,\widetilde{\pi}}(x)\right| \leq \sum_{t=h}^H (L_{\sf env})^{t-h} Q_{\max, t} \cdot \sup_{x' \in \mc{X}} \Delta_t(x').
\end{equation}
\end{theorem}

\begin{proof}
We proceed by backward induction on $h$.

\paragraph{Base case} ($h = H$):
At the terminal stage there is no continuation value, so the policy-risk operator reduces to evaluation of the immediate payoff:
\[
V_{H}^{i,\pi}(x) = \mc{R}^{\sf p}_{i,H}(r_{i,H}, \pi_H; x) + \tfrac{1}{\epsilon_i}\nu_i(\pi_{i,H}(\cdot \mid x)),
\]
which has sensitivity to $\pi_H$ bounded by $Q_{\max,H}$ since $r_{i,h}(x,a) \in [0,1]$ and $Q_{\max,H}$ incorporates the regularization range. By Lemma~\ref{lem:one_step_sensitivity},
\[
\left|V_{H}^{i,\pi}(x) - V_{H}^{i,\widetilde{\pi}}(x)\right| \leq Q_{\max, H} \|\pi_H(\cdot|x) - \widetilde{\pi}_H(\cdot|x)\|_1 = (L_{\sf env})^0 Q_{\max, H} \Delta_H(x).
\]
This satisfies \eqref{eq:value_perturbation} for $t=H$.

\paragraph{Inductive step}:
Assume \eqref{eq:value_perturbation} holds for stage $h+1$. Consider stage $h$.
We decompose the value difference into two terms: deviation due to the policy change (Term I) and deviation due to the continuation value propagation (Term II) as follows:
\begin{align*}
\left| V_{h}^{i,\pi}(x) - V_{h}^{i,\widetilde{\pi}}(x) \right|
&\leq \underbrace{\left| \mc{R}^{\sf p}_{i,h}(Q_h^{i,\pi}, \pi_h; x) - \mc{R}^{\sf p}_{i,h}(Q_h^{i,\pi}, \widetilde{\pi}_h; x) \right| + \tfrac{1}{\epsilon_i}\left|\nu_i(\pi_{i,h}(\cdot\mid x)) - \nu_i(\widetilde{\pi}_{i,h}(\cdot\mid x))\right|}_{\text{(I) Policy Error}}\\
&\qquad
+ \underbrace{\left| \mc{R}^{\sf p}_{i,h}(Q_h^{i,\pi}, \widetilde{\pi}_h; x) - \mc{R}^{\sf p}_{i,h}(Q_h^{i,\widetilde{\pi}}, \widetilde{\pi}_h; x) \right|}_{\text{(II) Value Propagation}}.
\end{align*}

\begin{itemize}
    \item \textbf{Bounding (I):}
By Lemma~\ref{lem:one_step_sensitivity}, the policy-risk operator term satisfies $|\mc{R}^{\sf p}_{i,h}(Q_h^{i,\pi}, \pi_h; x) - \mc{R}^{\sf p}_{i,h}(Q_h^{i,\pi}, \widetilde{\pi}_h; x)| \leq \|Q_h^{i,\pi}\|_\infty \Delta_h(x)$, and the Lipschitz continuity of $\nu_i$ gives $\tfrac{1}{\epsilon_i}|\nu_i(\pi_{i,h}(\cdot|x)) - \nu_i(\widetilde{\pi}_{i,h}(\cdot|x))| \leq (L_\nu/\epsilon_i)\Delta_h(x)$. Since $Q_{\max,h} \geq \|Q_h^{i,\pi}\|_\infty + L_\nu/\epsilon_i$ by assumption, we obtain
\[
(\text{I}) \leq Q_{\max,h}\, \Delta_h(x).
\]
\item
\textbf{Bounding (II):}
Since $\mc{R}^{\sf p}_{i,h}$ is 1-Lipschitz in the $Q$-argument and $Q_h^{i,\pi}$ depends on $V_{h+1}^{i,\pi}$ through $\mc{R}^{\sf e}_{i,h}$, which is $L_{\sf env}$-Lipschitz in the continuation value, we obtain
\[
(\text{II}) \leq L_{\sf env}\, \mb{E}_{\widetilde{\pi}} \left[ \left| V_{h+1}^{i,\pi}(X_{h+1}) - V_{h+1}^{i,\widetilde{\pi}}(X_{h+1}) \right| \;\Big|\; X_h=x \right].
\]
\end{itemize}

\noindent\textbf{Combining the bounds on (I) and (II):}
Applying the inductive hypothesis for the term inside the expectation, we have that
\begin{align*}
(\text{II}) &\leq L_{\sf env} \mb{E}_{\widetilde{\pi}} \left[ \sum_{t=h+1}^H (L_{\sf env})^{t-(h+1)} Q_{\max, t} \mb{E}_{\widetilde{\pi}}[\Delta_t(X_t) \mid X_{h+1}] \;\Big|\; X_h=x \right] \\
&= \sum_{t=h+1}^H (L_{\sf env})^{t-h} Q_{\max, t} \mb{E}_{\widetilde{\pi}}[\Delta_t(X_t) \mid X_h=x].
\end{align*}
Adding Term (I) ($Q_{\max, h} \Delta_h(x)$) to Term (II) completes the sum from $t=h$ to $H$, proving the theorem.
\end{proof}

\begin{corollary}[End-to-End Value Robustness for RQRE-OVI]
\label{cor:end_to_end}
Suppose Algorithm~\ref{alg:pr-er-rqre-ovi} produces parameter estimates $\widehat{w}_h$ satisfying
\[
\sup_{x \in \mc{X}} \max_{i \in [n]} \|\widehat{w}_{h}^{i} - w_{h}^{i,\star}\|_2 \leq \delta_h \quad \forall h \in [H].
\]
Let $\pi_h(x) := \pi_h^{\tt{RQRE}}(x; \widehat{w}_h)$ and $\pi_h^\star(x) := \pi_h^{\tt{RQRE}}(x; w_h^\star)$, where both are exact stage-game RQRE solutions for the respective payoff tables. By Theorem~\ref{thm:stagewise_rqe_robust}, the policy deviation is bounded by $\sup_x \Delta_h(x) \leq \frac{c}{\mu} \delta_h$. When the stage solver is approximate with error $\varepsilon_{\sf eq}$, the policy deviation bound becomes $\sup_x \Delta_h(x) \leq \frac{c}{\mu} \delta_h + \varepsilon_{\sf eq}$, and the value bound \eqref{eq:end_to_end_value} acquires an additional additive term of $H\,\varepsilon_{\sf eq}$.
Consequently, by Theorem~\ref{thm:value_perturbation}, the value error bound holds:
\begin{equation}
\label{eq:end_to_end_value}
\sup_{x \in \mc{X}} \left|V_{1}^{i,\pi}(x) - V_{1}^{i,\pi^\star}(x)\right| \leq \frac{c}{\mu} \sum_{t=1}^H (L_{\sf env})^{t-1} Q_{\max, t} \delta_t.
\end{equation}
If $\delta_t \leq \bar{\delta}$ and $Q_{\max, t} \leq H$ for all $t$, then the following bound also holds:
\begin{equation}
\label{eq:end_to_end_uniform}
\sup_{x \in \mc{X}} \left|V_{1}^{i,\pi}(x) - V_{1}^{i,\pi^\star}(x)\right| \leq \frac{c H \bar{\delta}}{\mu} \left( \sum_{t=0}^{H-1} (L_{\sf env})^t \right).
\end{equation}
\end{corollary}

\begin{remark}[Risk Sensitivity and Comparison with Nash]
The bound in \eqref{eq:end_to_end_uniform} highlights the stability benefits of RQRE. When the risk measure is coherent or risk-neutral we have $L_{\sf env}= 1$ \cite{coherent-risk-measures-1999}, and recover a polynomial dependence $O(H^2 \bar{\delta} / \mu)$. For general convex risk measures where $L_{\sf env} > 1$, the error bound scales exponentially with the horizon, reflecting the inherent difficulty of risk-sensitive planning. However, crucially, this error remains \emph{Lipschitz continuous} with respect to the estimation error $\bar{\delta}$. Example~\ref{ex:nash_unstable} shows that Nash equilibria can exhibit $O(1)$ policy jumps from arbitrarily small payoff perturbations at even a single stage, making the analogous bound for Nash-based algorithms potentially unbounded.
\end{remark}

\section{Proofs for Distributional Robustness of RQRE}
\label{sec:dro-proofs}
We begin by showing a result which demonstrates the equivalence of the penalty and constraint formulations of the convex risk measure. What follows is a simplified version of results in \cite{follmer2002convex}.
\begin{proposition}[Penalty--Constraint Duality for Convex Risk Measures]
\label{prop:penalty-constraint}
Let $\rho : \mc{Z} \to \R$ be a convex risk measure with dual representation $\rho(Z) = \sup_{p \in \Delta(\Omega)}\{\E_p[-Z] - \penaltyfn(p)\}$ as in Theorem~\ref{thm:dual-rep}, where $\penaltyfn : \Delta(\Omega) \to (-\infty, \infty]$ is a convex, lower-semicontinuous penalty function. Then for any bounded random variable $Z$, the penalized problem is equivalent to the constrained problem
\[
\sup_{p \in \Delta(\Omega)} \left\{ \E_p[-Z] - \penaltyfn(p) \right\}
\;=\;
\sup_{p \in \Delta(\Omega)} \left\{ \E_p[-Z] \;\text{s.t.}\; \penaltyfn(p) \leq \delta \right\}
\]
for a radius $\delta = \delta(\rho, Z)$ uniquely determined by the optimality conditions, and always larger than the minimum value of $\varphi$. The risk parameter governing $\rho$ acts as the Lagrange multiplier of the constraint.
\end{proposition}

\begin{proof}
The equivalence follows by Lagrangian duality. Introduce the multiplier $\lambda\geq0$ for the constraint constraint $\penaltyfn(p) \le \delta$ and write
\begin{align*}
\sup_{p \in \Delta(\Omega)} \left\{ \E_p[-Z] \;\text{s.t.}\; \penaltyfn(p) \leq \delta \right\} &= \sup_{p \in \Delta(\Omega)} \left\{ \E_p[-Z] + \min_{\lambda\ge 0} \lambda( \delta - \penaltyfn(p) ) \right\} \\
&=\min_{\lambda \ge 0}\sup_{p \in \Delta(\Omega)} \left\{ \E_p[-Z] -  \lambda \penaltyfn(p)  \right\} + \lambda \delta \\
&= \min_{\lambda \ge 0} \phi(\lambda) + \lambda \delta,
\end{align*}
where the second equality is due to strong duality, which holds by Slater's condition as the minimum value of $\varphi$ is always a feasible point. Note that the above implies that for any fixed $\delta$, we can recover the equivalent $\lambda$ by minimizing $\phi(\lambda) + \lambda \delta$. Thus the penalized version of the problem is equivalent to the hard constraint $\penaltyfn(p) \le \delta$. 
\end{proof}
We now prove Proposition~\ref{prop:dro} by showing how each component of the RQRE objective corresponds to a distributional robustness guarantee, using Proposition~\ref{prop:penalty-constraint} as the key tool. 

\paragraph{Bounded rationality as policy robustness.} The bounded rationality component of the RQRE objective assigns each player a regularized best response of the form
\[
\pi_i \in \argmax_{\mu_i \in \Delta(\mc{A}_i)} \left\{ \E_{a \sim (\mu_i, \pi_{-i})}[u_i(a)] + \tfrac{1}{\bddrat_i} \nu_i(\mu_i) \right\}. 
\]
Applying Proposition~\ref{prop:penalty-constraint} with $\penaltyfn = \nu_i$ and $\lambda=1/\bddrat_i$, this is equivalent to
\[
\pi_i \in \argmax_{\mu_i \in \Delta(\mc{A}_i)} \left\{ \E_{a \sim (\mu_i, \pi_{-i})}[u_i(a)] \; \; \text{s.t.}\;\; \nu_i(\mu_i) \leq \delta_{\sf pol} \right \}, 
\]
where the radius $\delta_{\sf pol}=\delta_{\sf pol}(\bddrat_i)$ is uniquely determined by the optimality conditions, with $1/\bddrat_i$ severing as the Lagrange multiplier of the constraint. Thus bounded rationality corresponds to distributionally robust policy selection where each agent maximizes expected payoff over all policies with a $\nu_i$-ball of radius $\delta_{\sf pol}$. As $\bddrat_i \to \infty$, the multiplier vanishes, the constraint tightens, and the Nash best response is recovered. Conversely, as $\bddrat_i \to 0$, the ball expands and the policy approaches the maximizer of $\nu_i$ (e.g., the uniform distribution when $\nu_i$ is the negative entropy).

\paragraph{Strategic robustness from risk aversion.} In risk averse Markov games, the policy risk operator can be written as $\mc R^{\mathrm{pol}}_{i,h}
\big(Q_h^i,\pi_h;x\big)
:=
\rho^{\mathsf{p}}_{i,h}
\big(
Z_{i,h}^{\mathsf{p}}
\,\big|\,
\pi_{-i,h}(\cdot\mid x)
\big),$ where 
\[
\rho^{\mathsf{p}}_{i,h}(Z\mid \pi_{-i})
=
\sup_{p_i\in \mc P_i^{\psf}}
\{
-\langle Z,p_i\rangle
-
\penaltyfn_{\psf,i}(p_i,\pi_{-i})
\},
\]
is the risk measure given in dual form. Applying Proposition~\ref{prop:penalty-constraint} to this penalized problem yields the constrained formulation

\[
\rho^{\mathsf{p}}_{i,h}(Z\mid \pi_{-i})
=
\sup_{p_i\in \mc P_i^{\psf}}
\{
-\langle Z,p_i\rangle
\;\; \mid \;\; 
\penaltyfn_{\psf,i}(p_i,\pi_{-i}) \le \delta_{\sf opp}
\}.
\]

\paragraph{Environment robustness from risk aversion.} Similarly, we write the environment risk operator in terms of the environment risk measure
\[
\rho^{\mathsf{e}}_{i,h}(Z| \mP)
=
\inf_{\widetilde{\mP} \in \mc P(\mc X)}
\left\{
\int Z(x')\,\widetilde{\mP}(dx')
+
\penaltyfn_{\esf,i}(\widetilde{\mP}\,\|\,\mP)
\right\}
\]
which, due to Proposition~\ref{prop:penalty-constraint}, has the constrained formulation
\[
\rho^{\mathsf{e}}_{i,h}(Z| \mP)
=
\inf_{\widetilde{\mP} \in \mc P(\mc X)}
\left\{
\int Z(x')\,\widetilde{\mP}(dx')
\; \;\mid \; \; 
\penaltyfn_{\esf,i}(\widetilde{\mP}\,\|\,\mP) < \delta_{\sf env}
\right\}.
\]
Here, we replace the $\sup$ with $\inf$ and flip the sign on the penalty function.

In both these settings, the radii $\delta_{\sf opp}, \delta_{\sf env}$ is determined by by the optimality conditions of the penalized-constrained duality, with the Lagrange multiplier governed by the structure of $\penaltyfn_{\psf, i}$ and $\penaltyfn_{\esf, i}$. The first, $\delta_{\sf opp}$, can be viewed as defining a radius of opponent actions we are risk averse to. Similarly, $\delta_{\sf env}$ can be viewed as defining a radius of environment transitions. We usually set $\varphi$ to be scaled by a risk aversion parameter $1 / \tau$, and thus as $\tau \to \infty$ the radius $\delta$ will grow to cover the entire space of measures (either over the opponents actions or the environment transition). As $\tau \to 0$ we recover the risk neutral Markov game setting, where risk averse utilities become utilities and risk averse value functions become value functions.

\section{Additional Experimental Details}
\label{sec:additional-experiment-details}

In this section we provide further details on our experiments.

\begin{figure}[H]
    \centering
    \includegraphics[height=5cm]{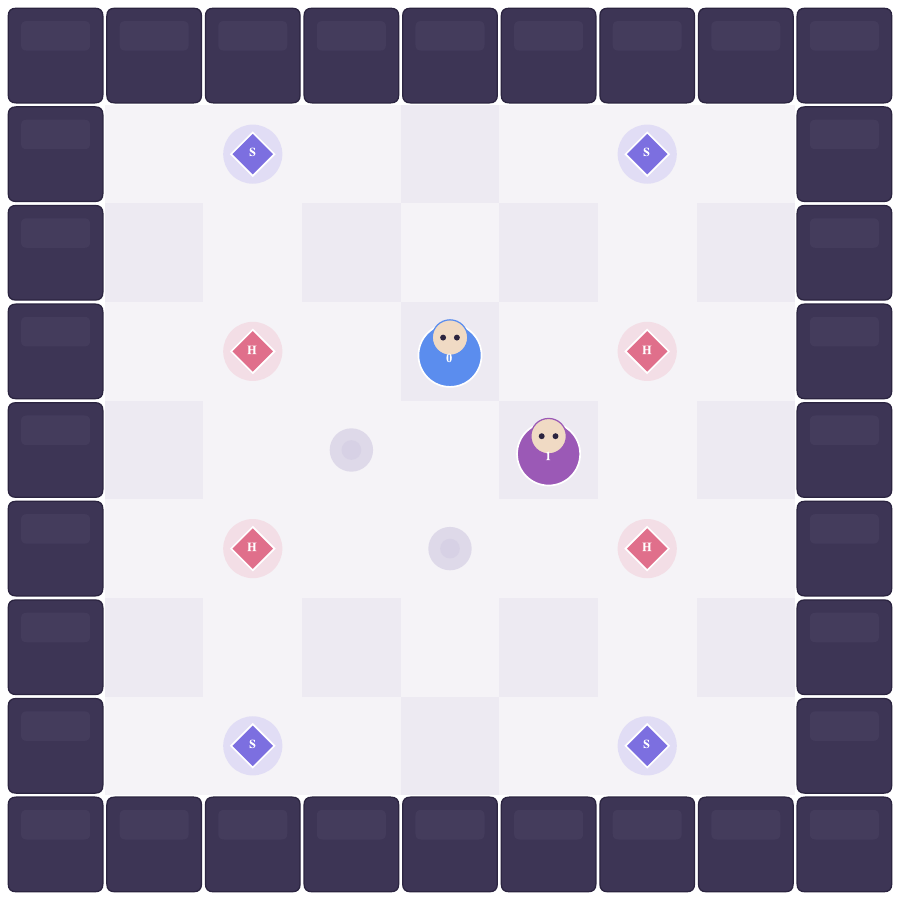}
    \hspace{1cm}
    \includegraphics[height=5cm]{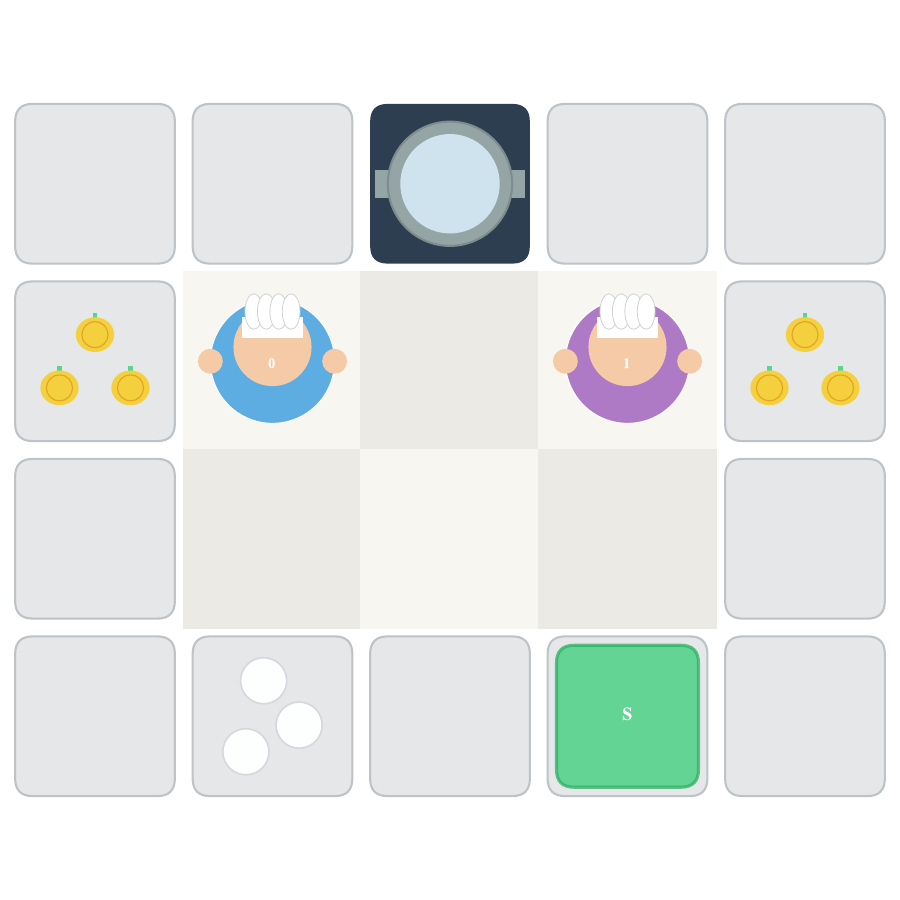}
    \caption{Dynamic Stag Hunt (left) and Overcooked (right) environments used in experiments.}
    \label{fig:environments}
\end{figure}

\subsection{Dynamic Stag Hunt}

\paragraph{Environment~1: Dynamic Stag Hunt.} We evaluate on a Dynamic Stag Hunt environment based on the Melting Pot suite \cite{agapiou2022melting} implemented as a $9\times9$ grid-world  with episodes of $75$ timesteps. At each step, agents choose from six actions: four cardinal movements, stay, and interact. The grid contains four \emph{stag} resources placed in the corners and four \emph{hare} resources adjacent to the spawn points; agents pick up a resource by walking over it and may swap by stepping onto a different type. An interaction resolves when at least one agent chooses the interact action while both agents are adjacent and carrying a resource. The resulting payoffs follow the classic Stag Hunt matrix: mutual stag yields $(4,4)$, mutual hare yields $(2,2)$, and stag-hare mismatch gives $(0,2)$ with $0$ to the stag-holder and $2$ to the hare-holder. Upon interaction, both agents are respawned randomly at one of the spawn points with empty inventories, allowing multiple interactions per episode. The spatial structure introduces a coordination challenge beyond the matrix game: hare resources are close to spawn, providing a safe default, while stag resources require navigating to the corners, demanding both spatial and strategic coordination.

\begin{figure}[htbp]
    \centering
    \includegraphics[width=\textwidth]{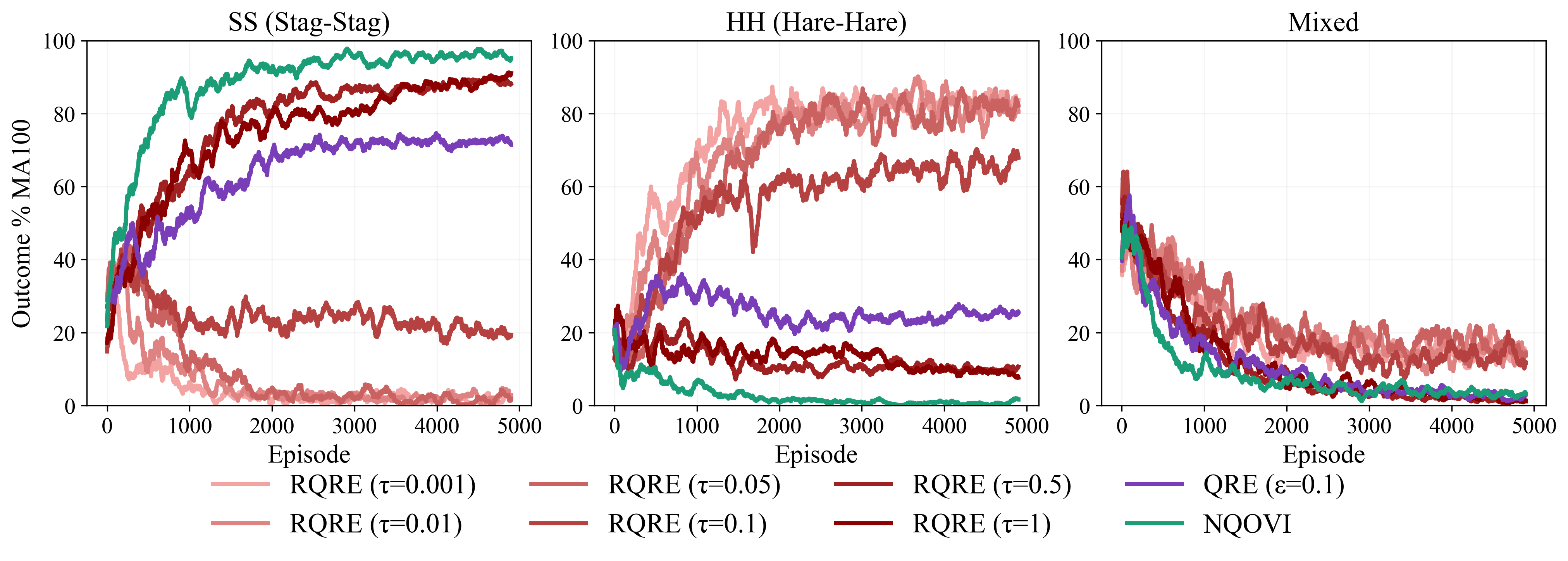}
    \caption{\textbf{Stag Hunt outcome distributions during training.} Each panel shows the fraction of stag--stag, hare--hare, and mixed interaction outcomes (rolling average) for a given algorithm and risk-aversion level. NQOVI, QRE, and low risk averse RQRE agents converge to payoff dominant stag--stag outcomes, while highly risk averse agents to risk dominant hare--hare, confirming the expected equilibrium selection.}
    \label{fig:staghunt_outcomes}
\end{figure}

Each state is represented by a 15-dimensional observation vector capturing both players' normalized grid positions, Manhattan distances to the nearest stag and hare resources, inventory indicators (one-hot encoding of whether each agent carries a stag, hare, or nothing), normalized inter-agent distance, a timing feature (fraction of episode elapsed), and bias term. The joint state-action feature vector $\phi(x,a_1,a_2) \in \mb{R}^{200}$ is constructed by concatenating the observation, one-hot action encoding for both players, observation-action cross products for each player, and an action-action cross products. 

\subsubsection{Dynamic Stag Hunt Training details}
We train all agents for $K = 5000$ episodes using  optimistic value iteration with ridge regression ($\lambda = 1.0$) and exploration bonus coefficient $\beta = 0.1$. The experience buffer stores the most recent $1000$ transitions and parameters are updated every $20$ episodes via batch least-squares regression over the buffered data. The reward scale is set to $4.0$, matching the maximum single-interaction payoff. For QRE and RQRE agents, the stage-game equilibrium solver runs for up to $T = 100$ fixed-point iterations with early stopping at tolerance $10^{-6}$, using bounded rationality $\bddrat = 0.05$ for both players. RQRE agents use symmetric risk-aversion parameters $\risk_1 = \risk_2 = \risk$, swept over $\{0.005, 0.01, 0.05, 0.1, 0.5, 1.0\}$. All experiments use a single seed with results averaged over $200$ evaluation rollouts for cross-play and perturbation experiments.

\subsection{Overcooked}

\begin{figure}[htbp]
    \centering
    \includegraphics[width=0.6\textwidth]{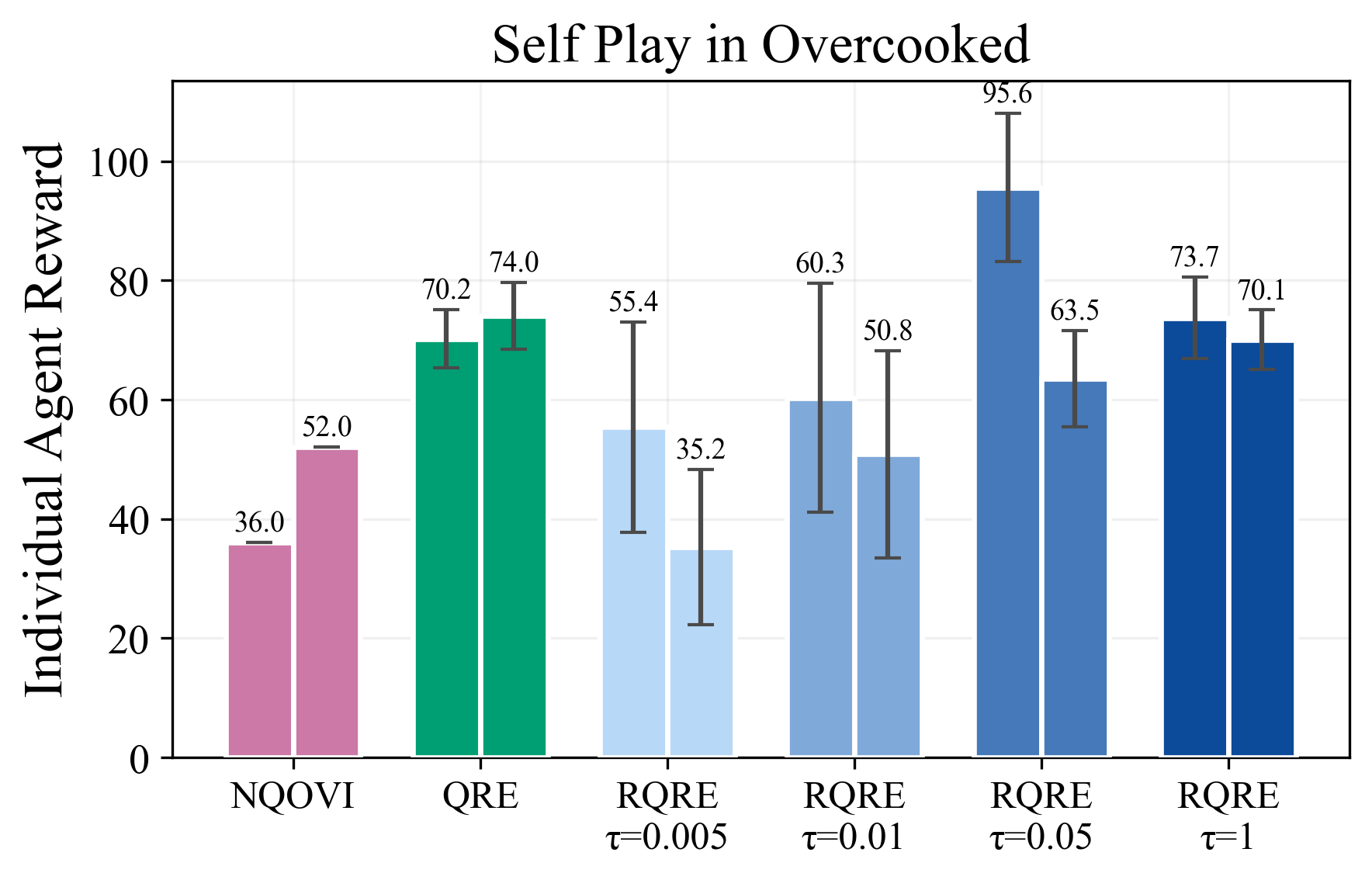}
    \caption{Self play performance for each algorithm in Overcooked over 200 evaluation rollouts.}
    \label{fig:staghunt_outcomes}
\end{figure}

\paragraph{Environment~2: Overcooked.} We also evaluate on the Overcooked environment \cite{gessler2025overcookedv}; specifically, we use the version introduced in JaxMARL \cite{flair2024jaxmarl} with custom visualizations. Two agents operate in a small kitchen and must cooperate to prepare and deliver onion soups: pick up onions from piles, place them in a pot (which requires three onions to begin cooking), retrieve the finished soup with a plate, and deliver it to a serving location for a sparse reward of 20. Each agent chooses from six actions (four cardinal movements, stay, and interact) over a horizon of 100 timesteps. The environment provides both sparse delivery rewards and shaped rewards for intermediate subtask completion (e.g., placing an onion in the pot, picking up a plate when soup is cooking). The tight layout and sequential subtask structure make coordination essential, as agents must avoid blocking each other while efficiently dividing labor.

The observation is a 49-dimensional vector encoding both agents' positions, orientations, and inventory states, pot status (number of onions, cooking progress, readiness), per-agent proximity features to key locations (pot, onion pile, plate pile, serving goal), a time fraction, and hand-crafted goal-potential features that capture task-relevant interactions. The joint state-action feature vector $\phi(x,a_1,a_2) \in \mb{R}^{685}$ concatenates the observation, one-hot action encodings, and cross products between a selected subset of a core observation features and each agents' action. 

\begin{figure}[htbp]
    \centering
    \begin{subfigure}[b]{0.49\textwidth}
        \centering
        \includegraphics[width=\textwidth]{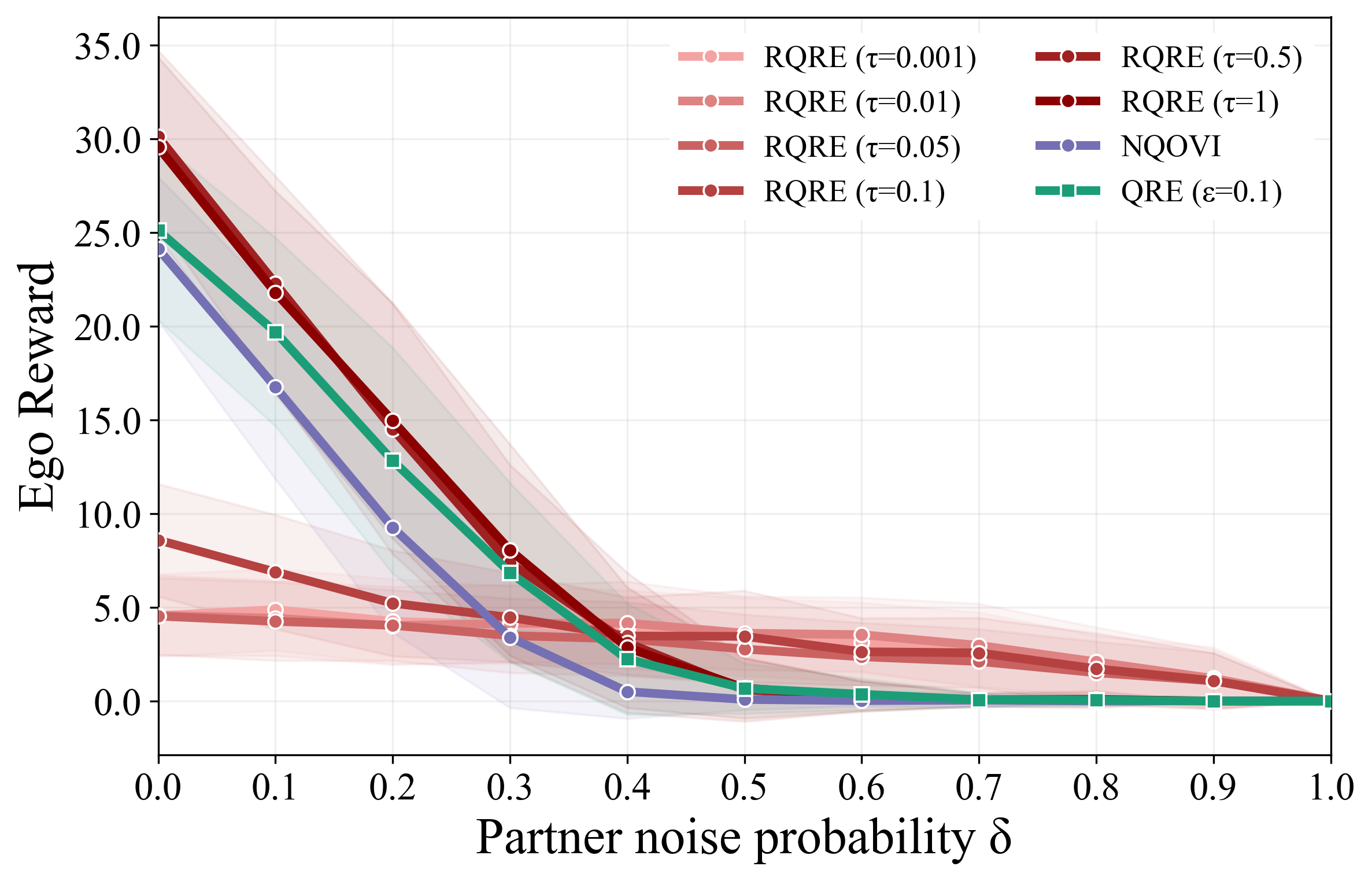}
    \end{subfigure}
    %\hfill
    \begin{subfigure}[b]{0.49\textwidth}
        \centering
        \includegraphics[width=\textwidth]{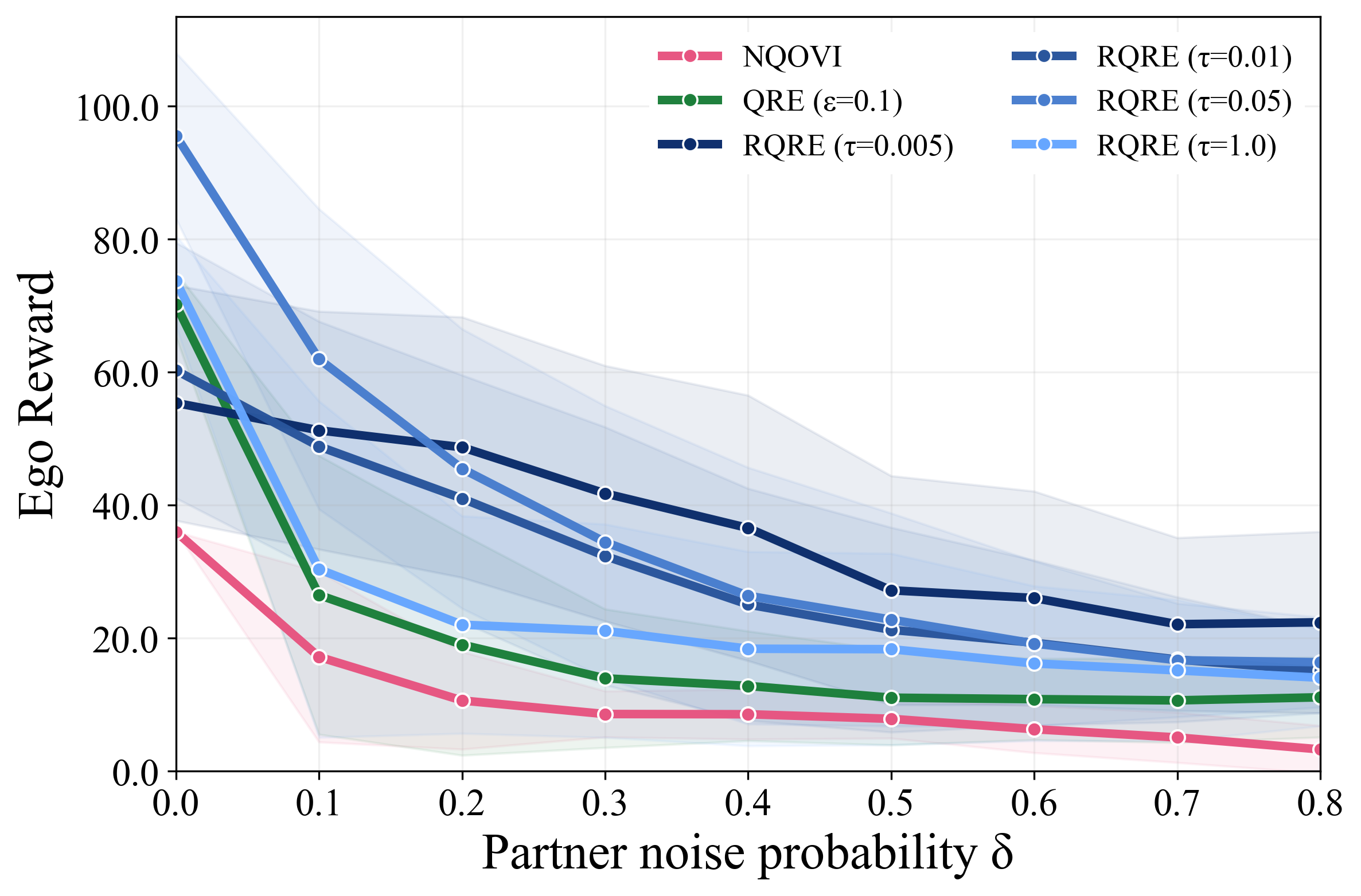}
    \end{subfigure}
    \caption{\textbf{Cross-play with a perturbed partner.} Ego agent reward as a function of the partner noise for Stag Hunt (left) and Overcooked (right). At each evaluation step, the perturbed partner's action is replaced with a fixed deterministic action (e.g., always move in a single direction) with probability $\delta$, and the trained policy is executed otherwise. This deterministic perturbation ensures that deviations are high-signal rather than averaging out under uniform noise. The results are averaged over 200 evaluation rollouts at each noise level.}
    \label{fig:perturbed_partner}
\end{figure}

\subsubsection{Overcooked Training details}
Agents are trained for $K = 5000$ episodes with ridge parameter $\lambda = 5.0$ (larger than Stag Hunt due to the higher feature dimension $d = 685$) and exploration bonus $\beta = 0.1$. The buffer stores $500$ transitions and updates occur every $10$ episodes. Training rewards combine sparse delivery rewards and shaped intermediate rewards with equal weighting (both coefficients set to $1.0$), and the reward scale is set to $20.0$ to match the delivery reward. The QRE solver runs for up to $T = 100$ iterations with early stopping at tolerance $10^{-6}$, with $\bddrat_1 = \bddrat_2 = 0.1$. RQRE agents use symmetric $\risk_1 = \risk_2 = \risk$ swept over $\{0.005, 0.01, 0.05, 1.0\}$. The environment runs in cooperative mode, where both agents receive identical sparse rewards upon delivery.

\subsection{Hyperparameters}

We provide the hyperparamters used in each experiment below. 

% Stag Hunt
\begin{table}[htbp]
    \caption{Hyperparameters for Stag Hunt experiments}
    \begin{center}
    \setlength{\tabcolsep}{4pt}
        \renewcommand{\arraystretch}{0.85}
        \begin{tabular}{lll}
            \toprule
            \textbf{Category} & \textbf{Hyperparameter} & \textbf{Value} \\
            \midrule
            \textbf{Setup} & Episodes ($K$) & 5000 \\
            & Feature dimension ($d$) & 200 \\
            & Ridge parameter ($\lambda$) & 1.0 \\
            & Exploration bonus ($\beta$) & 0.1 \\
            & Buffer size & 1000 \\
            & Update frequency & 10 \\
            & Seed & 42 \\
            \midrule
            \textbf{Environment} & Horizon ($H$) & 75 \\
            & Actions per agent & 6 \\
            & Grid size & $9 \times 9$ \\
            \midrule
            \textbf{Rewards} & Stag--Stag payoff & $(4, 4)$ \\
            & Hare--Hare payoff & $(2, 2)$ \\
            & Stag--Hare payoff & $(0, 2)$ \\
            & Step cost & $0$ \\
            \midrule
            \textbf{RQRE} & Bounded rationality ($\epsilon$) & 0.1 \\
            & Risk aversion ($\tau$) & $\{0.005, 0.01, 0.05, 1.0\}$ \\
            & Solver tolerance & $10^{-6}$ \\
            & Max solver iterations ($T$) & 100 \\
            \bottomrule
        \end{tabular}
    \end{center}
    \label{tab:hyperparams_staghunt}
\end{table}

\begin{table}[htbp]
    \caption{Hyperparameters for Overcooked experiments}
    \begin{center}
    \setlength{\tabcolsep}{4pt}
        \renewcommand{\arraystretch}{0.85}
        \begin{tabular}{lll}
            \toprule
            \textbf{Category} & \textbf{Hyperparameter} & \textbf{Value} \\
            \midrule
            \textbf{Setup} & Episodes ($K$) & 5000 \\
            & Feature dimension ($d$) & 685 \\
            & Ridge parameter ($\lambda$) & 5.0 \\
            & Exploration bonus ($\beta$) & 0.1 \\
            & Buffer size & 500 \\
            & Update frequency & 10 \\
            & Seed & 42 \\
            \midrule
            \textbf{Environment} & Horizon ($H$) & 100 \\
            & Actions per agent & 6 \\
            & Layout & Cramped Room \\
            \midrule
            \textbf{Rewards} & Placement in pot & 3 \\
            & Plate pickup & 4 \\
            & Soup pickup & 8 \\
            & Delivery completion & 6 \\
            & Soup drop penalty & $-1$ \\
            \midrule
            \textbf{RQRE} & Bounded rationality ($\epsilon$) & 0.1 \\
            & Risk aversion ($\tau$) & $\{0.005, 0.01, 0.05, 1.0\}$ \\
            & Solver tolerance & $10^{-6}$ \\
            & Max solver iterations ($T$) & 100 \\
            \bottomrule
        \end{tabular}
    \end{center}
    \label{tab:hyperparams_overcooked}
\end{table}

\end{document}